\documentclass[lettersize,journal]{IEEEtran}
\usepackage{amsmath,amsfonts}
\usepackage{amsthm}
\usepackage{amssymb}
\usepackage[most]{tcolorbox}
\usepackage{algorithm}
\usepackage[noEnd=true,indLines=true,rightComments=false,italicComments=false,
commentColor=teal,beginLComment=//~,endLComment=]{algpseudocodex}
\usepackage{placeins}

\definecolor{darkgreen}{RGB}{0,120,60}
\definecolor{theoremgray}{RGB}{225,225,225}
\definecolor{RankFirst}{RGB}{255,242,204}  
\definecolor{RankSecond}{RGB}{255,204,153} 
\definecolor{RankThird}{RGB}{221,235,247} 
\definecolor{coralPink}{HTML}{ED028C}

\newcommand{\tsb}{\textsubscript}

\tcolorboxenvironment{definition}{
  enhanced,
  breakable,
  colback=theoremgray,
  colframe=theoremgray,
  boxrule=0pt,
  arc=2pt,
  left=5pt,
  right=5pt,
  top=5pt,
  bottom=5pt,
  before skip=8pt,
  after skip=8pt
}

\tcolorboxenvironment{theorem}{
  enhanced,
  breakable,
  colback=theoremgray,
  colframe=theoremgray,
  boxrule=0pt,
  arc=2pt,
  left=5pt,
  right=5pt,
  top=5pt,
  bottom=5pt,
  before skip=8pt,
  after skip=8pt
}

\tcolorboxenvironment{proposition}{
  enhanced,
  breakable,
  colback=theoremgray,
  colframe=theoremgray,
  boxrule=0pt,
  arc=2pt,
  left=5pt,
  right=5pt,
  top=5pt,
  bottom=5pt,
  before skip=8pt,
  after skip=8pt
}

\tcolorboxenvironment{lemma}{
  enhanced,
  breakable,
  colback=theoremgray,
  colframe=theoremgray,
  boxrule=0pt,
  arc=2pt,
  left=5pt,
  right=5pt,
  top=5pt,
  bottom=5pt,
  before skip=8pt,
  after skip=8pt
}

\tcolorboxenvironment{IEEEproof}{
  enhanced,
  breakable,
  colback=theoremgray,
  colframe=theoremgray,
  boxrule=0pt,
  arc=2pt,
  left=5pt,
  right=5pt,
  top=5pt,
  bottom=5pt,
  before skip=8pt,
  after skip=8pt
}

\tcolorboxenvironment{assumption}{
  enhanced,
  breakable,
  colback=theoremgray,
  colframe=theoremgray,
  boxrule=0pt,
  arc=2pt,
  left=5pt,
  right=5pt,
  top=5pt,
  bottom=5pt,
  before skip=8pt,
  after skip=8pt
}

\newtheorem{theorem}{Theorem}
\newtheorem{proposition}{Proposition}
\newtheorem{lemma}{Lemma}

\newtheorem{assumption}{Assumption}

\tikzset{
  phaseduring/.style={
    draw=none,
    fill=ForestGreen!15,
    minimum width=0.92\columnwidth
  },
  phaseafter/.style={
    draw=none,
    fill=Orange!18,
    minimum width=0.96\columnwidth
  },
  algpxIndentLine/.style={draw=black, thin}
}

\definecolor{algcommentorange}{RGB}{190,90,0}

\newcommand{\OrangeLComment}[1]{%
    \State \textcolor{algcommentorange}{//~#1}%
}

\usepackage{array}
\usepackage[caption=false,font=normalsize,labelfont=sf,textfont=sf]{subfig}
\usepackage{textcomp}
\usepackage{stfloats}
\usepackage{url}
\usepackage{verbatim}
\usepackage{graphicx}
\usepackage{cite}
\usepackage{booktabs}
\usepackage{multirow}

\usepackage[table,dvipsnames]{xcolor}
\usepackage[pagebackref,breaklinks,colorlinks]{hyperref}

\definecolor{algblue}{RGB}{220,232,255}

\definecolor{alggreen}{RGB}{218,242,218}

\begin{document}

\title{PAPT++: Risk-Aware Adversarial Tuning and Generation for Single Domain Generalization}

\author{
Zhipeng Xu,
De Cheng\textsuperscript{*},
Xinyang Jiang,
Lingfeng He,
Huaijie Wang,
Dongsheng Li,
Nannan Wang,~\IEEEmembership{Senior Member,~IEEE},
and Xinbo Gao,~\IEEEmembership{Fellow,~IEEE}
\thanks{Zhipeng Xu, De Cheng, Lingfeng He, Huaijie Wang, Nannan Wang, and Xinbo Gao are with Xidian University, Xi'an, China. Xinyang Jiang and Dongsheng Li are with Microsoft Research Asia, Shanghai, China (e-mail: xinyangjiang@microsoft.com; dongsli@microsoft.com).} 
\thanks{\textsuperscript{*}Corresponding author: De Cheng (e-mail: dcheng@xidian.edu.cn).}
}

\markboth{IEEE Transactions on Pattern Analysis and Machine Intelligence}%
{Shell \MakeLowercase{\textit{et al.}}: A Sample Article Using IEEEtran.cls for IEEE Journals}

\maketitle

\begin{abstract}
Single domain generalization (SDG) aims to learn a model from one labeled source domain that generalizes to unseen target domains. A common strategy is to enrich the source distribution with augmented or generated samples, and recent text-to-image (T2I) diffusion models provide a strong generative prior for this purpose. However, diversity alone is insufficient for robust generalization, because useful generated samples should also capture variations that the current classifier finds difficult. Motivated by distributionally robust optimization (DRO), we define a semantic ambiguity set in the class-conditional generative space of a pretrained T2I model and search it for samples with high classification loss under the current classifier. To this end, we introduce PAPT++, a risk-aware adversarial generation-training framework for SDG. PAPT++ first learns diverse semantic reference images for each class through image-text alignment and intra-class diversity regularization. These references then serve as denoising targets during classifier-guided diffusion synthesis, reducing semantic drift while guiding generation toward challenging variations. The generated samples are combined with the source data to update the classifier, and the updated classifier guides the next synthesis round in return. In this way, PAPT++ progressively exposes the classifier to challenging yet semantically consistent variations. Extensive experiments on standard SDG benchmarks demonstrate the superiority of the proposed PAPT++ method and the effectiveness of its main components.
\end{abstract}

\begin{IEEEkeywords}
Domain Generalization, Distributional Robust Learning, Adversarial Training, Generative model.
\end{IEEEkeywords}

\section{Introduction}
\label{sec:intro}
\IEEEPARstart{D}{eep} neural networks (DNNs) are commonly trained and evaluated under the assumption that training and test data are independently and identically distributed (\emph{i.i.d.})~\cite{cha2022domain}.
In real-world applications, however, test data often deviate from the training distribution due to changes in style, background, viewpoint, image quality, or acquisition conditions.
As a result, a model trained only on the source data may suffer a noticeable performance drop when deployed in unseen environments.
To improve robustness under such distribution shifts, domain generalization (DG)~\cite{DPR,RD-MLDG} aims to learn from available labeled source data and generalize to unseen target domains without accessing target data during training.

Single domain generalization (SDG) is a more challenging setting of DG, where only one labeled source domain is available for training. To alleviate the lack of domain diversity, existing SDG methods usually expand or regularize the source distribution through data augmentation, adversarial perturbation, or normalization. Although these approaches improve robustness by introducing additional variations or regularizing the learned representation, augmentation- and perturbation-based methods often operate within predefined transformation spaces, such as perturbations of channel-wise feature statistics. Consequently, the resulting training distribution may still fail to cover complex unseen variations in style, background, viewpoint, or image quality.

Recent text-to-image (T2I) diffusion models provide a promising way to further enrich the source distribution for SDG. 
Benefiting from large-scale image-text pretraining, these models contain rich generative priors and can synthesize diverse visual and semantic variations beyond predefined image transformations. 
However, simply increasing generation diversity does not necessarily lead to robust generalization. From the perspective of robust learning, useful generated samples should not only be diverse, but also expose challenging variations that are not well covered by the source data and cannot be reliably handled by the current model.

\begin{figure}[t]
\centering
\includegraphics[width=0.48\textwidth]{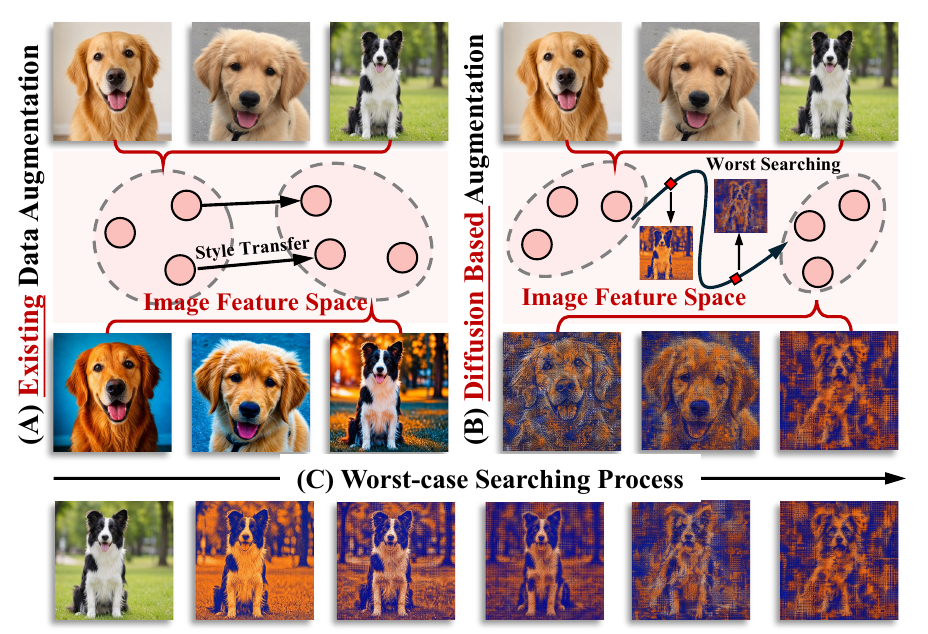}
\vspace{-3.0mm}
\caption{Comparison of data augmentation strategies for SDG.
(A) Conventional augmentation expands the source distribution within
predefined transformation spaces. (B) Diffusion-based augmentation
provides broader semantic variations but is not explicitly guided by
the current classifier. (C) PAPT++ searches the class-conditional T2I
generative space for classifier-challenging samples while preserving
class semantics.}
\label{fig:motivation}
\end{figure}

To guide the search for such variations, we build on distributionally robust optimization (DRO), which minimizes the worst-case risk over an ambiguity set around the source distribution. Rather than fitting only observed or augmented samples, DRO explicitly considers possible distributional perturbations and focuses on those that yield high model risk.
Applying this formulation to T2I diffusion models requires a suitable ambiguity set for high-risk sample search. Existing methods~\cite{PADG} have constructed such sets in feature spaces parameterized by CLIP or ResNet encoders. However, extending feature-space risk search to the generative space of T2I diffusion models is non-trivial. Without semantic constraints, directly maximizing model risk may produce degenerate high-loss samples, such as artifact-dominated or texture-biased images, rather than meaningful variations that preserve core class semantics. Therefore, a key question naturally arises: \textbf{how can we use a T2I diffusion model to construct a semantic ambiguity set and search it for high-risk samples while preserving core class semantics?}

Based on this motivation, we propose PAPT++, a risk-aware adversarial generation-training framework for SDG, as shown in Fig.~\ref{fig:motivation}. PAPT++ extends our previous PAPT~\cite{PAPT} by moving from diversity-driven generation to risk-aware diffusion synthesis guided by the current classifier. Specifically, PAPT++ uses the class-conditional generative space of a pretrained T2I model to construct a semantic ambiguity set. Semantic references for each class constrain the search to reduce semantic drift, while the current classifier guides generation toward variations with high classification loss. PAPT++ alternates between difficult sample generation and classifier training, progressively exposing the classifier to challenging yet semantically consistent variations and improving its generalization to unseen domains. It consists of two modules: Class-Level Semantic Reference Learning (CSRL) and Classifier-Guided Adversarial Diffusion Synthesis (CADS).

CSRL provides semantic constraints for subsequent adversarial synthesis. Without such constraints, classifier-guided generation may increase classification loss by drifting away from the target class. CSRL therefore first learns a set of semantic reference images for each class by adapting the diffusion model with two complementary objectives: image-text alignment and intra-class diversity regularization. The image-text alignment objective encourages generated images to follow their class prompts and preserve class-level semantics, including overall object appearance, shape, and other discriminative visual cues. However, this objective alone may cause the generator to favor a small set of high-scoring visual patterns. The intra-class diversity objective therefore penalizes feature similarity among images of the same class, encouraging more diverse visual appearances. The learned reference images are then fixed and used as denoising targets to regularize the subsequent classifier-guided synthesis.

Using the semantic references learned by CSRL, CADS alternates between classifier-guided diffusion synthesis and classifier training. In each synthesis round, the current classifier is fixed to provide classification-loss feedback, while the diffusion model is optimized to generate samples with high classification loss under this classifier. Meanwhile, the semantic references serve as denoising targets to constrain the diffusion optimization. This constraint discourages the generator from increasing classification loss through semantic drift or by introducing noise, distorted textures, and other visual artifacts. The generated challenging samples are then combined with the source data to update the classifier, and the updated classifier guides the next synthesis round.

The main contributions are summarized as follows:
\begin{itemize}
\item We propose PAPT++, a risk-aware adversarial generation-training framework for SDG that extends our previous PAPT beyond diversity-driven prompt generation. PAPT++ formulates classifier-guided sample generation as a search over a class-conditional semantic ambiguity set and alternates between challenging sample generation and classifier training.

\item We develop two complementary modules for semantically constrained risk search. CSRL learns diverse semantic reference images through image-text alignment and intra-class diversity regularization, while CADS uses these references to regularize classifier-guided diffusion synthesis, generating challenging yet semantically consistent samples. We further introduce a progressive denoising optimization strategy to extend gradient guidance over the denoising trajectory and improve generation quality.

\item We establish a theoretical connection between PAPT++ and DRO. We show that the reference denoising objective controls an upper bound on the KL divergence between the smoothed reference and generated distributions. We further derive a target-domain risk bound that accounts for ambiguity-set coverage, the finite CADS search gap, and finite-sample estimation error.

\item Extensive experiments on SDG and multi-source DG benchmarks show that PAPT++ achieves state-of-the-art performance. Ablation studies verify the effectiveness of its main components, and additional experiments analyze the generation process from multiple perspectives.
\end{itemize}

This article extends our CVPR 2025 conference paper PAPT~\cite{PAPT} in three main aspects:
(1) PAPT++ moves beyond the diversity-driven prompt generation of PAPT by using the current classifier to search for challenging variations in the T2I generative space. CSRL provides semantic references that constrain this search, while CADS incorporates classifier feedback and progressive denoising into an iterative generation-training process.
(2) We further formulate this risk-aware generation process from a DRO perspective. The analysis shows how reference denoising constrains distributional deviation and how the worst-case risk over the resulting semantic ambiguity set relates to target-domain risk, providing theoretical support for semantic preservation and high-risk search.
(3) The empirical study is substantially expanded to examine both performance and the underlying generation mechanism. PAPT++ achieves state-of-the-art results on SDG and multi-source DG benchmarks, together with strong corruption generalization.

\begin{figure*}[!htbp]
\centering
\includegraphics[width=1.00\textwidth]{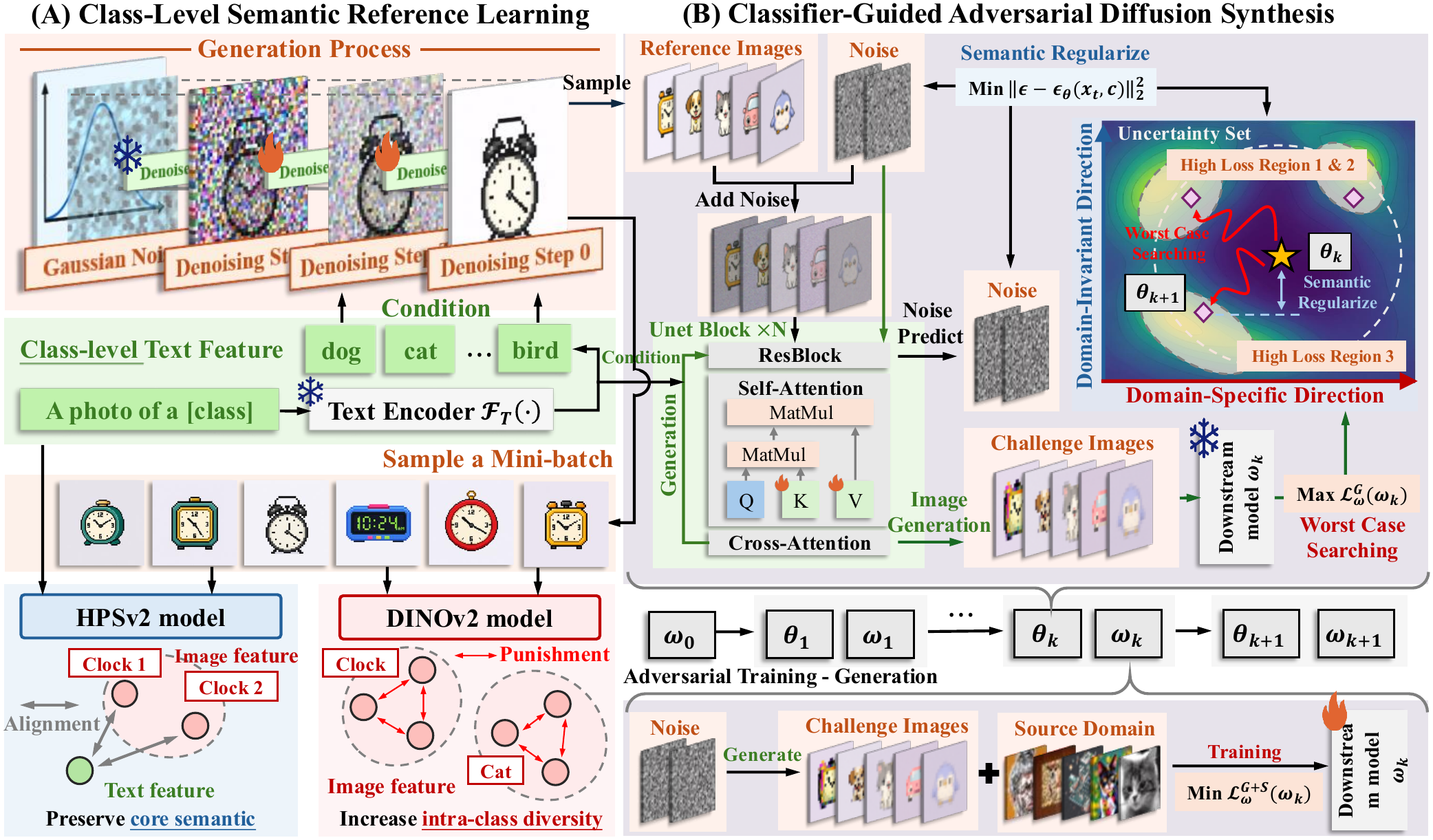}
\vspace{-5.0mm}
\caption{
The framework of our proposed PAPT++. 
PAPT++ consists of two modules: (A) Class-Level Semantic Reference Learning (CSRL) and (B) Classifier-Guided Adversarial Diffusion Synthesis (CADS). 
CSRL learns diverse class-level semantic reference images by encouraging \textbf{\textcolor{blue}{image-text alignment}} and \textbf{\textcolor{red}{intra-class diversity}}. 
CADS uses these reference images as denoising targets to regularize the diffusion parameters during classifier-guided adversarial synthesis, encouraging the generation of high-risk samples with reduced semantic drift.
The generated challenge images are combined with source-domain data to update the downstream model, and the updated model provides feedback for the next synthesis round.
}
\label{fig:framework}
\vspace{-3.0mm}
\end{figure*}

\section{Related Work}
\label{sec:related_work}

\subsection{Data Augmentation for SDG}

\noindent SDG aims to generalize a model trained on a single source domain to multiple unseen target domains. Existing methods commonly expand the training distribution through adversarial or style-based augmentation. ADA~\cite{volpi2018generalizing} and ME-ADA~\cite{zhao2020maximum} generate adversarial virtual samples, while ESDA~\cite{volpi2019addressing} searchs a predefined transformation space for model vulnerabilities. MixStyle~\cite{zhou2021domain} mixes channel-wise feature statistics, L2D~\cite{wang2021learning} learns a style-complement module, and ASR-Norm~\cite{fan2021adversarially} adapts normalization to features generated by ADA. Although effective, these methods operate within predefined transformation spaces or feature statistics and may not cover complex shifts in style, background, and viewpoint. Our method instead leverages T2I generative priors to explore broader visual and semantic variations and searches for high-risk yet semantically valid samples rather than pursuing diversity alone~\cite{PAPT}. Other works also make contributions~\cite{StPR,IKI,EKPC,SIKD,li2026few}.

\subsection{Text-to-Image Models}

\noindent Text-to-image (T2I) generation has progressed rapidly with large-scale autoregressive models such as DALL-E~\cite{ramesh2021zero} and diffusion models.
Latent diffusion models~\cite{rombach2022high} perform denoising in the latent space of a pretrained autoencoder and use text embeddings to guide image generation.
Recent studies further adapt diffusion models using human-preference supervision~\cite{xu2024imagereward} or parameter-efficient tuning.
LoRA introduces low-rank weight updates, reducing the number of trainable parameters while keeping the pretrained backbone fixed.
Our method uses the generative prior of a pretrained T2I model and optimizes only LoRA parameters to synthesize challenging yet semantically consistent samples for SDG.

\subsection{Distributional Robust Optimization}

\noindent Distributionally Robust Optimization (DRO) aims to learn models that perform well under worst-case distributional shifts within a predefined uncertainty set~\cite{sagawa2019distributionally,CBCM}. 
Instead of minimizing the empirical risk only on observed samples, DRO optimizes the maximum risk over possible distributions around the training distribution, providing a principled objective for robust learning under distributional perturbations. 
Existing DRO methods define uncertainty sets with different discrepancy measures, such as Wasserstein distance~\cite{sinha2017certifying,PADG}, $f$-divergences~\cite{namkoong2016stochastic}, and maximum mean discrepancy~\cite{staib2019distributionally}. 
Among them, Wasserstein DRO has been widely studied due to the flexibility of Wasserstein balls in modeling distributional perturbations~\cite{sinha2017certifying}. 
However, the practical effectiveness of DRO in domain generalization remains challenging~\cite{liu2021towards}. 
A key difficulty lies in how to define an appropriate uncertainty set: an overly restricted set may fail to cover useful domain shifts, while an overly flexible set may include unrealistic or label-inconsistent samples. 
For example, \cite{hu2018does} show that the over-flexibility of the uncertainty set can limit the effectiveness of Wasserstein DRO in classification tasks. 
To address this issue, some studies introduce additional constraints, such as unlabeled data~\cite{frogner2019incorporating}, data geometry~\cite{liu2022distributionally}, or topology information~\cite{qiao2023topology}, to construct more meaningful uncertainty sets.

\section{Methodology}
\label{sec:methodology}

\noindent \textbf{Problem Definition.} We formulate the single domain generalization problem as learning a classifier $\mathbf{f}_\theta:\mathcal{X}\rightarrow\mathbb{R}^{C}$ from one labeled source domain such that it generalizes to multiple unseen target domains that share the same label space. Here, $\mathbf{f}_\theta(x)$ denotes the $C$-class logits.
The source-domain training set is denoted as $\mathcal{D}^{\mathcal{S}}=\{(\mathbf{x}_{i}^{s},y_i^s)\}_{i=1}^{N_s}$,
where \(\mathbf{x}_{i}^{s}\in\mathcal{X}\) is the \(i\)-th source image, \(y_i^s\in\{1,\ldots,C\}\) is its class label, \(N_s\) is the total number of source samples, \(\mathcal{X}\) denotes the image space, and \(C\) is the total number of categories.
The target domains are represented by
$\mathcal{D}^{\mathcal{T}}=\{\mathcal{D}^{\mathcal{T}}_{m}\}_{m=1}^{M}$, where \(\mathcal{D}^{\mathcal{T}}_{m}\) denotes the \(m\)-th target domain and \(M\) is the total number of target domains. 
All target-domain data are unavailable during training and are used only for evaluation.

\vspace{2.0mm}

\noindent \textbf{Overall Framework.} The overall framework of PAPT++ is shown in Fig.~\ref{fig:framework}. PAPT++ consists of Class-Level Semantic Reference Learning (CSRL) and Classifier-Guided Adversarial Diffusion Synthesis (CADS). It searches the class-conditional generative space of a pretrained T2I model within a semantic ambiguity set defined by reference-based denoising, while the current classifier directs generation toward high-loss variations. Only lightweight LoRA~\cite{hu2022lora,DoRA,he2026harnessing,CKAA} parameters are optimized, with the pretrained diffusion backbone kept fixed.

CSRL learns diverse semantic reference images for each class through image--text alignment and intra-class diversity regularization (Fig.~\ref{fig:framework}.A). The \textbf{\textcolor{blue}{image--text alignment}} objective preserves class semantics, whereas \textbf{\textcolor{red}{diversity regularization}} encourages complementary intra-class appearances. The learned references are then fixed and used as denoising targets to reduce semantic drift during adversarial synthesis.
CADS alternates between classifier-guided diffusion synthesis and classifier training (Fig.~\ref{fig:framework}.B). In each round, the fixed classifier guides generation toward high-loss samples, while the semantic references constrain denoising. The generated samples are combined with the source data to update the classifier, which then guides the next synthesis round.

\subsection{Preliminaries}
\label{sec:preliminaries}

\noindent We employ a pretrained text-to-image latent diffusion model as the generative prior. The model consists of a variational autoencoder (VAE)~\cite{kingma2013auto}, a CLIP text encoder~\cite{radford2021learning}, and a text-conditioned U-Net~\cite{ronneberger2015u} denoiser. Given an image $\mathbf{x}$, the VAE encoder maps it into the latent space as:
\begin{equation}
    \mathbf{z}_0=\mathcal{F}_{E}(\mathbf{x}),
\end{equation}
where $\mathcal{F}_{E}(\cdot)$ denotes the VAE encoder and $\mathbf{z}_0$ is the clean latent representation. The corresponding VAE decoder
$\mathcal{F}_{D}(\cdot)$ maps a latent code back to the image space.

For class $k$, we use a text template $\mathbf{t}^c_k$, e.g., ``a photo of a [class]'', where the token [class] is then replaced by the class name. The prompt $\mathbf{t}^c_k$ is encoded by the CLIP text encoder associated with the pretrained diffusion model:
\begin{equation}
    \mathbf{T}_k^c = \boldsymbol{\tau}_{\boldsymbol{\psi}}(\mathbf{t}^c_k),
\end{equation}
where $\boldsymbol{\tau}_{\boldsymbol{\psi}}(\cdot)$ denotes the CLIP text encoder, and $\mathbf{T}_k^c \in \mathbb{R}^{L \times d}$ denotes the text embedding sequence with token length $L$ and feature dimension $d$. In this work, the text encoder is kept frozen to preserve the semantic knowledge learned from large-scale image-text pretraining.

The forward diffusion process gradually corrupts the clean latent $\mathbf{z}_0$. At timestep $t$, the noisy latent is obtained by:
\begin{equation}
    \mathbf{z}_t =
    \sqrt{\bar{\alpha}_t} \mathbf{z}_0
    +
    \sqrt{1-\bar{\alpha}_t}\boldsymbol{\epsilon},
    \quad
    \boldsymbol{\epsilon} \sim \mathcal{N}(\mathbf{0},\mathbf{I}),
\label{eq:forward_diffusion_noising}
\end{equation}
where $\bar{\alpha}_t$ is the cumulative noise schedule coefficient and $\boldsymbol{\epsilon}$ is Gaussian noise.

The denoising network predicts the added noise $\hat{\boldsymbol{\epsilon}}$ conditioned on the noisy latent $\mathbf{z}_t$, the timestep $t$, and the text representation $\mathbf{T}_k^c$:
\begin{equation}
    \hat{\boldsymbol{\epsilon}}
    =
    \boldsymbol{\epsilon}_{\boldsymbol{\omega},\boldsymbol{\phi}}(\mathbf{z}_t,t,\mathbf{T}_k^c),
\label{eq:noise_prediction}
\end{equation}
where $\boldsymbol{\epsilon}_{\boldsymbol{\omega},\boldsymbol{\phi}}(\cdot)$ denotes the U-Net denoiser.
Here, $\boldsymbol{\omega}$ denotes the frozen pretrained U-Net parameters, and $\boldsymbol{\phi}$ denotes the trainable LoRA parameters. In our implementation, $\boldsymbol{\phi}$ denotes LoRA parameters inserted into the self-attention layers of the U-Net, while the pretrained diffusion backbone remains fixed.

The textual condition interacts with the intermediate image features through the cross-attention layers of the U-net. Let $\mathbf{h}_t$ denote the intermediate image features at diffusion step $t$. The query, key, and value representations are defined as:
\begin{equation}
    \mathbf{Q} = \mathbf{h}_t \mathbf{W}_Q, \quad
    \mathbf{K} = \mathbf{T}_k^c \mathbf{W}_K, \quad
    \mathbf{V} = \mathbf{T}_k^c \mathbf{W}_V ,
\end{equation}
and the resulting cross-attention is given by:
\begin{equation}
    \operatorname{Attn}(\mathbf{h}_t,\mathbf{T}_k^c)
    =
    \operatorname{softmax}
    \left(
    \frac{\mathbf{Q}\mathbf{K}^\top}{\sqrt{d_a}}
    \right)\mathbf{V}.
\end{equation}
Here, $\mathbf{W}_{Q}$, $\mathbf{W}_K$, and $\mathbf{W}_V$ denote the projection matrices, and $d_a$ denotes the dimension of the projected queries and keys. Through this interaction, the class prompt provides semantic guidance for the denoising process.
In this work, we adapt the diffusion model by optimizing the self-attention LoRA parameters $\boldsymbol{\phi}$ while keeping the pretrained backbone fixed.

The denoising objective is formulated as:
\begin{equation}
    \mathcal{L}_{\mathrm{LDM}}
    =
    \mathbb{E}_{\mathbf{z}_0,t,\boldsymbol{\epsilon}}
    \left[
    \left\|
    \boldsymbol{\epsilon} -
    \boldsymbol{\epsilon}_{\boldsymbol{\omega},\boldsymbol{\phi}}(\mathbf{z}_t,t,\mathbf{T}_k^c)
    \right\|_2^2
    \right],
\end{equation}
where the denoising network $\boldsymbol{\epsilon}_{\boldsymbol{\omega},\boldsymbol{\phi}}$ predicts
the Gaussian noise added to the clean latent at timestep $t$.

For notational simplicity, we denote the complete text-to-image generation process as:
\begin{equation}
    \tilde{\mathbf{x}} = \mathbf{G}_{\boldsymbol{\phi}}(\mathbf{t}_k^c,\boldsymbol{\eta}),
\end{equation}
where $\boldsymbol{\eta}$ denotes the initial Gaussian noise latent. Specifically, starting from the initial noise latent $\boldsymbol{\eta}$, the text-conditioned U-Net iteratively performs reverse denoising under the prompt representation $\mathbf{T}_k^c=\boldsymbol{\tau}_{\boldsymbol{\psi}}(\mathbf{t}_k^c)$, and the final latent is decoded by the VAE decoder $\mathcal{F}_{D}(\cdot)$ to obtain the generated image $\tilde{\mathbf{x}}$. 

These preliminaries define how the generator can be adapted through lightweight self-attention LoRA parameters. In the following sections, we use this adaptation interface for two purposes: first, to obtain class-level semantic reference images, and second, to synthesize classifier-guided high-risk variations under semantic reference constraints.

\subsection{Class-Level Semantic Reference Learning}

\noindent \underline{\textbf{Review of PAPT:}} In our conference version~\cite{PAPT}, the first stage is Category Prompt Tuning, which learns class prompts to capture domain-invariant category information from the source domain.
For the $k$-th class, PAPT defines a textual template:
\begin{equation}
\mathbf{t}_k^c = ``\text{a\ photo\ of\ a\ }*",    
\label{eq:PAPT_class_prompt}
\end{equation}
where the placeholder $*$ is then replaced by a learnable category prompt $\mathbf{E}_k^c$ in the embedding space $\mathbf{w}_k^c$:
\begin{equation}
\mathbf{w}_k^c =
[\varepsilon(\hat{\mathbf{t}}_{k,1}^c), \cdots,
\varepsilon(\hat{\mathbf{t}}_{k,M_c}^c), \mathbf{E}_k^c].
\label{eq:category_embedding}
\end{equation}
The resulting condition $\Phi(\mathbf{w}_k^c)$ is then fed into the cross-attention layers of the diffusion model and optimized with the denoising objective. In this way, PAPT obtains category-level conditions that preserve class-discriminative semantics when generating images with different domain styles.

\vspace{3.0mm}

\noindent\textbf{Motivation.}
Preserving class semantics is especially important in PAPT++, because the diffusion model is optimized using the classification risk of the current classifier. Although the class prompt provides semantic conditioning, it does not directly restrict changes in the diffusion parameters during risk-guided optimization. The generator may therefore increase the classification loss by drifting away from the target class instead of producing meaningful within-class variations.

To provide a more direct semantic constraint, we extend the category-preserving mechanism of PAPT from prompt-level conditioning to image-level semantic references. CSRL adapts the generator to produce reference images that follow the class prompt while covering diverse visual appearances within each class. These images are then fixed and used as denoising targets during adversarial synthesis, regularizing the diffusion parameters and reducing semantic drift.

\vspace{3.0mm}

Following the textual template in Eq.~\ref{eq:PAPT_class_prompt}, for each class $k$, we define the class-level text prompt as:
\begin{equation}
\mathbf{t}_{k,\mathrm{ref}}^c=\text{``a photo of a [class]''},
\label{eq:prompt}
\end{equation}
where the token \([\mathrm{class}]\) is replaced by the class name of category \(k\). 
The prompt \(\mathbf{t}_{k,\mathrm{ref}}^c\) provides the textual condition for generating semantic reference images of class \(k\).

Given \(\mathbf{t}_{k,\mathrm{ref}}^c\), we first sample a batch of \(B\) candidate 
images from the diffusion generator:
\begin{equation}
\mathbf{x}_{k,b}=\mathbf{G}_{\boldsymbol{\phi}_{\mathrm{ref}}}
(\mathbf{t}_{k,\mathrm{ref}}^c,\boldsymbol{\eta}_{k,b}),
\qquad b=1,\ldots,B,
\label{eq:reference_candidate_generation}
\end{equation}
where \(\mathbf{G}_{\boldsymbol{\phi}_{\mathrm{ref}}}\) denotes the T2I diffusion 
generator with reference-learning LoRA parameters 
\(\boldsymbol{\phi}_{\mathrm{ref}}\), \(\boldsymbol{\eta}_{k,b}\) is the random 
generation noise for the \(b\)-th candidate, and \(\mathbf{x}_{k,b}\) is the 
generated candidate image. We optimize \(\boldsymbol{\phi}_{\mathrm{ref}}\) so 
that these candidates are semantically aligned with the class-level prompt 
while remaining diverse within the same class.

Specifically, we use a HPSv2~\cite{wu2023hpsv2} model to measure the alignment between each candidate and its corresponding class prompt. 
Since HPSv2 is trained with human-preference annotations, its score reflects prompt-image correspondence and perceptual quality. Maximizing this score encourages the candidates to follow the class prompt while discouraging low-quality generations.
Let \(\mathbf{F}_{\mathrm{HPS}}(\mathbf{x}_{k,b},\mathbf{t}_{k,\mathrm{ref}}^c)\) denote the HPSv2 score between \(\mathbf{x}_{k,b}\) and \(\mathbf{t}_{k,\mathrm{ref}}^c\). 
The alignment loss is:
\begin{equation}
\mathcal{L}_{\mathrm{align}}
=-\frac{1}{B}\sum_{b=1}^{B}\mathbf{F}_{\mathrm{HPS}}(\mathbf{x}_{k,b},\mathbf{t}_{k,\mathrm{ref}}^c).
\label{eq:alignment_loss}
\end{equation}
Minimizing \(\mathcal{L}_{\mathrm{align}}\) encourages the generated candidates to better match the class-level textual condition.

Optimizing only image--text alignment may cause the generator to favor a small set of high-scoring visual patterns.
We therefore introduce a diversity loss in the DINOv2~\cite{oquab2024dinov2} feature space.
Because DINOv2 features encode high-level visual information, reducing their similarity promotes variations in shape, structure, and appearance beyond low-level pixel differences.
Fig.~\ref{fig:method2} compares the generated samples with and without diversity regularization.
Specifically, we extract normalized image features with a frozen DINOv2 encoder \(\mathbf{F}_{\mathrm{DINO}}(\cdot)\):
\begin{equation}
\mathbf{h}_{k,b}=\frac{\mathbf{F}_{\mathrm{DINO}}(\mathbf{x}_{k,b})}
{\left\|\mathbf{F}_{\mathrm{DINO}}(\mathbf{x}_{k,b})\right\|_2},
\label{eq:dino_feature}
\end{equation}
where \(\mathbf{h}_{k,b}\) is the normalized feature of \(\mathbf{x}_{k,b}\). 
We then define the diversity loss as the average pairwise feature similarity among same-class candidates:
\begin{equation}
\mathcal{L}_{\mathrm{div}}
=\frac{1}{B(B-1)}
\sum_{b=1}^{B}\sum_{b'\neq b}
\mathbf{h}_{k,b}^{\top}\mathbf{h}_{k,b'}.
\label{eq:diversity_loss}
\end{equation}

Minimizing \(\mathcal{L}_{\mathrm{div}}\) reduces the feature similarity among same-class candidates. This encourages the reference set to cover diverse object appearances within the same class.
The overall objective for semantic reference learning is:
\begin{equation}
\mathcal{L}_{\mathrm{ref}}
=\mathcal{L}_{\mathrm{align}}
+\lambda_{\mathrm{div}}\mathcal{L}_{\mathrm{div}},
\label{eq:reference_objective}
\end{equation}
where \(\lambda_{\mathrm{div}}\) is a hyper-parameter that balances class-level semantic alignment and intra-class diversity. 
By optimizing \(\boldsymbol{\phi}_{\mathrm{ref}}\) with Eq.~\ref{eq:reference_objective}, the generator is adapted to produce class-consistent reference images with diverse visual appearances.

After optimizing \(\boldsymbol{\phi}_{\mathrm{ref}}\), we merge the LoRA parameters into the pretrained diffusion model $\boldsymbol{\omega}$ and denote the resulting reference-adapted diffusion parameters as \(\bar{\boldsymbol{\omega}}_{\mathrm{ref}}\). 
The resulting diffusion model is then used to sample semantic reference images for each class:
\begin{equation}
\mathcal{A}_k=\{\mathbf{a}_{k,m}\}_{m=1}^{M_a},
\label{eq:reference_set}
\end{equation}
where \(\mathbf{a}_{k,m}\) denotes the \(m\)-th reference image of class \(k\), and \(M_a\) is the number of images per class. 
After construction, \(\mathcal{A}_k\) is fixed,  and its images are used as denoising targets in CADS to reduce semantic drift during high-risk sample search.

\begin{figure}[t]
\centering
\includegraphics[width=0.48\textwidth]{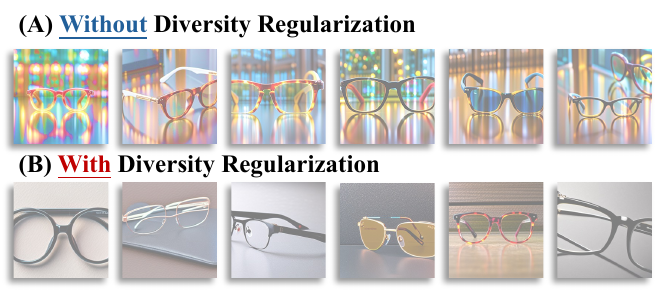}
\vspace{-3.0mm}
\caption{
Visualization of images generated with and without diversity regularization.
(A) Without diversity regularization, the generator tends to produce eyeglass images with similar front-view layouts, similar frame shapes, and colorful blurred backgrounds.
(B) With diversity regularization, the generated references show richer intra-class variations, including different viewpoints, frame shapes, lens colors, backgrounds, and eyeglass types.
}
\label{fig:method2}
\end{figure}

\subsection{Classifier-Guided Adversarial Diffusion Synthesis}
\label{sec:classifier_guided_generation}

\noindent \underline{\textbf{Review of PAPT.}} In our conference version PAPT~\cite{PAPT}, after learning category prompts, 
Adversarial Domain Prompt Tuning module further learns domain prompts to expand the 
source distribution with diverse domain styles. Specifically, PAPT maintains 
a domain-specific prompt memory bank:
\begin{equation}
\mathcal{P}=\{\mathbf{E}_1^d,\mathbf{E}_2^d,\cdots,\mathbf{E}_N^d\},
\label{eq:papt_domain_bank}
\end{equation}
where each domain prompt represents an abstract domain style. Given the category prompt \(\mathbf{E}_k^c\) and the domain prompt \(\mathbf{E}_i^d\), PAPT composes a category-domain textual template:
\begin{equation}
\mathbf{t}_{k,i}^{c-d}
=
\text{``a photo of a } * \text{ in the style of } \dagger \text{''},
\label{eq:papt_domain_template}
\end{equation}
where the placeholders \(*\) and \(\dagger\) are replaced by 
\(\mathbf{E}_k^c\) and \(\mathbf{E}_i^d\), respectively. 
The corresponding token embedding is then fed into the diffusion model to 
generate images of the \(k\)-th category under the \(i\)-th domain style.

To learn a new domain prompt $\mathbf{E}_o^d$, PAPT optimizes:
\begin{equation}
\mathcal{L}_{\mathrm{PAPT}}
=
\mathcal{L}_{\mathrm{class}}
+
\lambda \mathcal{L}_{\mathrm{domain}},
\label{eq:papt_objective}
\end{equation}
where \(\mathcal{L}_{\mathrm{class}}\) preserves category consistency and \(\mathcal{L}_{\mathrm{domain}}\) encourages the newly generated domain to be different from existing domains in the memory bank. After optimization, the newly learned domain prompt is added to the memory bank:
\begin{equation}
\mathcal{P} \leftarrow \mathcal{P} \cup \{\mathbf{E}_o^d\}.
\label{eq:papt_memory_update}
\end{equation}

By iteratively adding newly learned domain prompts to the memory bank, PAPT enlarges the coverage of generated domain styles and alleviates the lack of diversity in SDG.

\vspace{3.0mm}

\noindent\textbf{Motivation.}
Although this progressive prompt-based strategy expands the range of generated domain styles, its objective is mainly diversity-driven and independent of the current classifier. Specifically, PAPT encourages each newly generated domain to differ from those already stored in the memory bank, but it does not explicitly identify samples that the current classifier finds difficult. The generated samples may therefore improve visual diversity without providing the most effective supervision for the current classifier.

To address this limitation, we extend PAPT from diversity-driven prompt generation to classifier-guided adversarial diffusion synthesis. Instead of searching only for novel domain styles, CADS uses classification loss as feedback to search the T2I generative space for high-risk samples. Meanwhile, the semantic references learned by CSRL serve as denoising targets to preserve class semantics and reduce noise, distorted textures, and other visual artifacts during risk-guided synthesis.

\vspace{3.0mm}

\begin{figure}[t]
\centering
\includegraphics[width=0.48\textwidth]{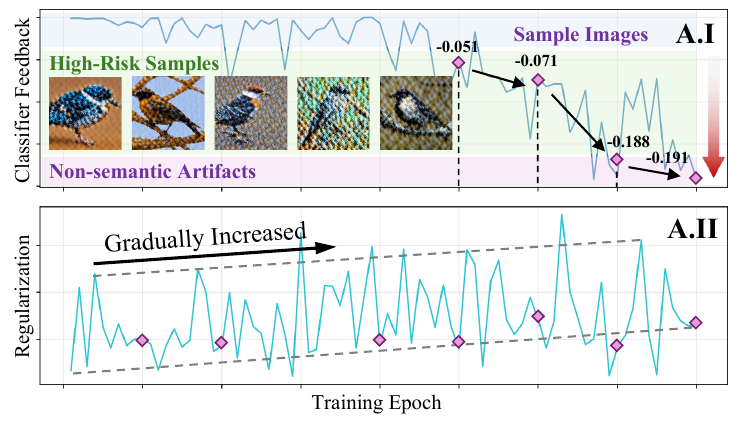}
\vspace{-3.0mm}
\caption{Optimization dynamics of the risk term
$\mathcal{L}_{\mathrm{risk}}^{(r)}$ in Eq.~\ref{eq:classifier_feedback_loss} and the semantic-reference denoising loss $\mathcal{L}_{\mathrm{den}}^{(r)}$ in Eq.~\ref{eq:reference_denoising_loss}.
The risk term guides the generator toward samples with high classification loss, while the denoising loss discourages semantic drift and non-semantic artifacts. Their balance produces high-risk yet semantically consistent samples for classifier training.}
\label{fig:method3}
\end{figure}

Following the above motivation, we formulate CADS as an iterative adversarial synthesis process guided by the classification risk of the downstream model. We first warm up the downstream model on the source domain $\mathcal{D}^{S}$:
\begin{equation}
\boldsymbol{\theta}^{(0)}
=\arg\min_{\boldsymbol{\theta}}\frac{1}{N_s}
\sum_{i=1}^{N_s}
\ell_{\mathrm{CE}}(\mathbf{f}_{\boldsymbol{\theta}}(\mathbf{x}_{i}^{s}),y_i^s),
\label{eq:source_warmup}
\end{equation}
where \(\boldsymbol{\theta}^{(0)}\) is the warmed-up downstream model parameter, \(\mathbf{x}_{i}^{s}\) is the \(i\)-th source image, \(y_i^s\) is its corresponding label, and \(N_s\) is the number of source samples. This warm-up stage provides the initial classifier for the subsequent risk-guided synthesis.

At synthesis round \(r\), the classifier \(\mathbf{f}_{\boldsymbol{\theta}^{(r-1)}}\) is fixed. 
The LoRA parameters \(\boldsymbol{\phi}^{(r)}\) are newly initialized when $r=1$ and initialized from the optimized LoRA parameters of the previous synthesis round when \(r>1\). For the \(k\)-th class, the \(j\)-th generated sample is obtained by:
\begin{equation}
\tilde{\mathbf{x}}_{k,j}^{(r)}
=
\mathbf{G}_{\boldsymbol{\phi}^{(r)}}
(\mathbf{t}_{k,\mathrm{ref}}^c,\boldsymbol{\eta}_{k,j}^{(r)}),
\label{eq:challenging_sample_generation}
\end{equation}
where \(\mathbf{t}_{k,\mathrm{ref}}^c\) is the class prompt defined in Eq.~\ref{eq:prompt}, \(\boldsymbol{\eta}_{k,j}^{(r)}\) is the random noise, and \(\tilde{\mathbf{x}}_{k,j}^{(r)}\) is the generated sample assigned to the \(k\)-th class. These samples are then evaluated by the fixed classifier to provide a classification-risk signal.

The classification risk loss is defined as:
\begin{equation}
\mathcal{L}_{\mathrm{risk}}^{(r)}
=\frac{1}{CM_g}
\sum_{k=1}^{C}
\sum_{j=1}^{M_g}
\ell_{\mathrm{CE}}
\left(\mathbf{f}_{\boldsymbol{\theta}^{(r-1)}}(\tilde{\mathbf{x}}_{k,j}^{(r)}),k\right),
\label{eq:classifier_feedback_loss}
\end{equation}
where \(M_g\) is the number of generated samples per class, \(C\) is the total number of classes, and \(\mathbf{f}_{\boldsymbol{\theta}^{(r-1)}}\) is the frozen classifier from the previous round. 
Maximizing $\mathcal{L}^{(r)}_{\rm risk}$ encourages the generator to produce samples that incur high classification loss for the current classifier $\mathbf{f}_{\theta^{(r-1)}}$. When combined with the semantic-reference constraint introduced below, these samples represent challenging within-class variations.

To preserve class semantics during risk-guided synthesis, we propose a semantic-reference denoising constraint. For the \(k\)-th class, a semantic reference image \(\mathbf{a}_{k,m}\) is sampled from the reference set \(\mathcal{A}_k\) and encoded by the VAE encoder:
\begin{equation}
\mathbf{z}_{k,m}^{a}
=
\mathcal{F}_E(\mathbf{a}_{k,m}),
\label{eq:reference_latent}
\end{equation}
where \(\mathbf{z}_{k,m}^{a}\) is the clean latent representation of \(\mathbf{a}_{k,m}\). The reference latent is then corrupted at timestep \(t\) according to:
\begin{equation}
\mathbf{z}_{k,m,t}^{a}
=
\sqrt{\bar{\alpha}_{t}}\mathbf{z}_{k,m}^{a}
+
\sqrt{1-\bar{\alpha}_{t}}\boldsymbol{\epsilon},
\label{eq:reference_forward_diffusion}
\end{equation}
where \(\mathbf{z}_{k,m,t}^{a}\) is the noisy reference latent, and 
\(\boldsymbol{\epsilon}\sim\mathcal{N}(\mathbf{0},\mathbf{I})\) is sampled Gaussian noise.
The denoising network predicts the added noise by:
\begin{equation}
\widehat{\boldsymbol{\epsilon}}_{k,m,t}
=
\boldsymbol{\epsilon}_{\boldsymbol{\phi}^{(r)}}
\left(
\mathbf{z}_{k,m,t}^{a},
t,
\mathbf{T}_{k,\mathrm{ref}}
\right),
\label{eq:reference_noise_prediction}
\end{equation}
where \(\widehat{\boldsymbol{\epsilon}}_{k,m,t}\) is the predicted noise, \(\boldsymbol{\epsilon}_{\boldsymbol{\phi}^{(r)}}(\cdot)\) denotes the denoising network at synthesis round \(r\), and \(\mathbf{T}_{k,\mathrm{ref}}=\tau_{\psi}(\mathbf{t}_{k,\mathrm{ref}}^c)\) is the text embedding of the \(k\)-th class prompt. The semantic-reference denoising loss is defined as:
\begin{equation}
\mathcal{L}_{\mathrm{den}}^{(r)}
=
\mathbb{E}_{k,m,t,\boldsymbol{\epsilon}}
\left[
\left\|
\boldsymbol{\epsilon}
-
\widehat{\boldsymbol{\epsilon}}_{k,m,t}
\right\|_2^2
\right],
\label{eq:reference_denoising_loss}
\end{equation}
where the expectation is taken over class index \(k\), reference index \(m\), timestep \(t\), and noise \(\boldsymbol{\epsilon}\). 
Notably, Eq.~\ref{eq:reference_denoising_loss} does not directly supervise each generated high-risk sample. Instead, it regularizes the LoRA parameters used for generation, preventing risk maximization from moving the generator away from semantic references.
The generation objective is:
\begin{equation}
\min_{\boldsymbol{\phi}^{(r)}}
\mathcal{L}_{\mathrm{gen}}^{(r)}
=\mathcal{L}_{\mathrm{den}}^{(r)}
-\lambda_{\mathrm{risk}}\mathcal{L}_{\mathrm{risk}}^{(r)},
\label{eq:generation_objective}
\end{equation}
where \(\lambda_{\mathrm{risk}}\) controls the strength of classification-risk signal. 
The denoising term acts as a reference-based regularizer, while the risk term guides the generator toward samples with high classification loss under the current classifier. 
Together, these terms guide the diffusion model toward samples with high classification loss while keeping the class-conditional generated distribution close to the semantic references.

After optimizing $\boldsymbol{\phi}^{(r)}$ by Eq.~\ref{eq:generation_objective}, we construct a class-balanced generated dataset for round $r$:
\begin{equation}
\mathcal{D}_{g}^{(r)}
=
\{(\tilde{\mathbf{x}}_{k,j}^{(r)},k)
\mid k=1,\ldots,C,\; j=1,\ldots,M_g\},
\label{eq:generated_dataset}
\end{equation}
where $\tilde{\mathbf{x}}_{k,j}^{(r)}$ is the $j$-th generated image of class $k$, and $M_g$ is the number of generated samples per class.
The classifier is then updated using both source-domain data $\mathcal{D}^{S}$ and current generated data $\mathcal{D}_{g}^{(r)}$:
\begin{equation}
\mathcal{D}_{\mathrm{train}}^{(r)}
=
\mathcal{D}^{S}\cup\mathcal{D}_{g}^{(r)}.
\label{eq:training_dataset}
\end{equation}
Starting from $\boldsymbol{\theta}^{(r-1)}$, the classifier is optimized by:
\begin{equation}
\begin{aligned}
\mathcal{L}_{\mathrm{clf}}^{(r)}
=&\frac{1}{N_s}\sum_{i=1}^{N_s}
\ell_{\mathrm{CE}}
\left(\mathbf{f}_{\boldsymbol{\theta}^{(r)}}(\mathbf{x}_{i}^{s}),y_i^s\right) \\
&+\mu\frac{1}{CM_g}\sum_{k=1}^{C}\sum_{j=1}^{M_g}
\ell_{\mathrm{CE}}
\left(\mathbf{f}_{\boldsymbol{\theta}^{(r)}}(\tilde{\mathbf{x}}_{k,j}^{(r)}),k\right),
\end{aligned}
\label{eq:classifier_objective}
\end{equation}
where $\mu$ controls the weight of generated-data supervision.
The first term preserves supervision from real source samples, while the second term trains the classifier on generated high-risk samples.
After this update, the resulting classifier $\mathbf{f}_{\boldsymbol{\theta}^{(r)}}$ provides feedback for the next synthesis round.

\begin{algorithm}[t]
\caption{Training pipeline of our proposed PAPT++}
\label{alg:cadg}
\begin{algorithmic}[1]
\Statex \textbf{Input:} Source dataset 
\(\mathcal{D}^{S}=\{(\mathbf{x}_i^s,y_i^s)\}_{i=1}^{N_s}\); 
class prompts \(\{\mathbf{t}_{k,\mathrm{ref}}^c\}_{k=1}^{C}\); 
pretrained diffusion generator \(\mathbf{G}_{\boldsymbol{\phi}}\); 
classifier \(\mathbf{f}_{\boldsymbol{\theta}}\); 
number of reference images per class \(M_a\); 
number of generated samples per class \(M_g\); 
generation-training rounds \(R\); 
weights \(\lambda_{\mathrm{div}}, \lambda_{\mathrm{risk}}, \mu\).
\Statex \textbf{Output:} Robust classifier \(\mathbf{f}_{\boldsymbol{\theta}^{(R)}}\).

\BeginBox[phaseduring]
\LComment{Class-Level Semantic Reference Learning}

\Repeat
    \State Optimize the diffusion with two objectives by Eq.~\ref{eq:reference_objective}:
    \[
    \boldsymbol{\phi}_{\mathrm{ref}}
    \leftarrow
    \arg\min_{\boldsymbol{\phi}_{\mathrm{ref}}}
    \{
    \mathcal{L}_{\mathrm{align}}
    +\lambda_{\mathrm{div}}\mathcal{L}_{\mathrm{div}}\}.
    \]
\Until{\(\mathcal{L}_{\mathrm{ref}}\) converges}

\State Merge $\boldsymbol{\phi}_{\mathrm{ref}}$ into the pretrained diffusion backbone $\boldsymbol{\omega}$ and obtain the reference-adapted backbone $\boldsymbol{\bar{\omega}}_{\mathrm{ref}}$.

\For{each class \(k=1,\ldots,C\)}
    \State Sample reference images $\mathcal{A}_k$ using $G_{\bar{\omega}_{\mathrm{ref}}}$.
\EndFor
\EndBox

\BeginBox[phaseafter]
\OrangeLComment{Classifier-Guided Adversarial Diffusion Synthesis}
\State Warm up the classifier on \(\mathcal{D}^{S}\) and obtain \(\boldsymbol{\theta}^{(0)}\) by minimizing Eq.~\ref{eq:source_warmup}.
\For{round \(r=1,\ldots,R\)}    
    
    \If{\(r=1\)} 
    \State Initialize a new LoRA adapter $\boldsymbol{\phi}^{(1)}$;;
    \ElsIf{\(r>1\)}
    \State Initialize the diffusion \(\boldsymbol{\phi}^{(r)} \leftarrow \boldsymbol{\phi}^{(r-1)}\);
    \EndIf

    \OrangeLComment{Adversarial Training between \(\boldsymbol{\phi}^{(r)}\) and \(\mathbf{f}_{\boldsymbol{\theta}^{(r-1)}}\)}
    \State \textcolor{red}{Fix the downstream classifier 
    \(\mathbf{f}_{\boldsymbol{\theta}^{(r-1)}}\);}
    \Repeat
        \State Optimize the diffusion model by Eq.~\ref{eq:generation_objective}:
        \[
        \boldsymbol{\phi}^{(r)}
        \leftarrow
        \arg\min_{\boldsymbol{\phi}^{(r)}}
        \{
        \mathcal{L}_{\mathrm{den}}^{(r)}
        -\lambda_{\mathrm{risk}}\mathcal{L}_{\mathrm{risk}}^{(r)}\};
        \]
    \Until{\(\mathcal{L}_{\mathrm{gen}}^{(r)}\) converges}
    
    \State Generate the class-balanced dataset 
    \(\mathcal{D}_{g}^{(r)}\) by Eq.~\ref{eq:generated_dataset}.
    
    \State Form the training set 
    \(\mathcal{D}_{\mathrm{train}}^{(r)}
    =\mathcal{D}^{S}\cup\mathcal{D}_{g}^{(r)}\).
    \State \textcolor{red}{Fix the diffusion model 
    \(\boldsymbol{\phi}^{(r)}\);}
    \Repeat
        \State Update the downstream classifier with Eq.~\ref{eq:classifier_objective}: 
        \[
        \boldsymbol{\theta}^{(r)}
        \leftarrow
        \arg\min_{\boldsymbol{\theta}}
        \mathcal{L}_{\mathrm{clf}}^{(r)}.
        \]
    \Until{\(\mathcal{L}_{\mathrm{clf}}^{(r)}\) converges}
\EndFor
\EndBox

\end{algorithmic}
\end{algorithm}

\begin{figure}[t]
\centering
\includegraphics[width=0.48\textwidth]{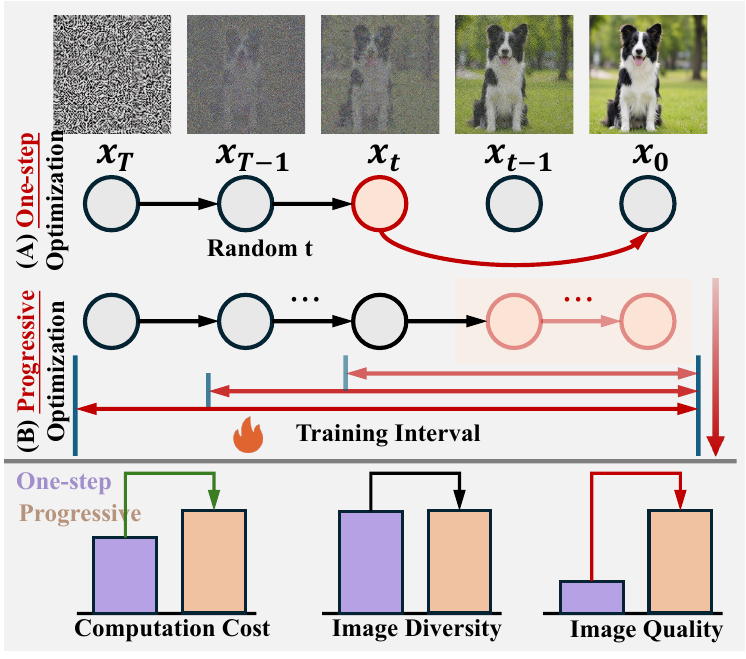}
\vspace{-3.0mm}
\caption{
Comparison between the one-step denoising optimization used in PAPT and our progressive denoising optimization strategy. The proposed strategy gradually increases the number of denoising steps for back-propagation, providing longer trajectory-level guidance and improving image quality.
}
\label{fig:method1}
\end{figure}

\subsection{LoRA Optimization and Sampling Strategy}

\noindent The objectives in Eq.~\ref{eq:reference_objective} and Eq.~\ref{eq:generation_objective} rely on feedback from external scoring and feature models as well as the current classifier. Backpropagating these signals to the trainable LoRA parameters through the full reverse denoising process is computationally expensive. PAPT therefore uses an efficient one-step approximation: it samples an intermediate low-noise timestep, predicts the clean image $\tilde{x}_0$, and backpropagates the loss through only this one-step denoising path. Although efficient, this strategy provides limited guidance along the full denoising trajectory and may produce noisy or overly abstract images when the optimization signal becomes more complex.

To improve generation quality and optimization stability, we adopt a progressive denoising optimization strategy. At the beginning of training, the loss is backpropagated through only the final denoising step, as shown in Fig.~\ref{fig:method1}. As training proceeds, earlier denoising steps are gradually included in the backward path. Specifically, after every $N_{\mathrm{opt}}$ optimization steps, the number of denoising steps used for gradient propagation is increased by one. This schedule stabilizes early optimization while allowing the objective to guide a longer part of the denoising trajectory later in training, resulting in images with fewer artifacts.

During sampling, we further adopt CFG++~\cite{chung2025cfg++} for both semantic reference generation and high-risk sample generation. 
CFG++ improves the visual quality of sampled images without changing the training objectives in Eq.~\ref{eq:reference_objective} and Eq.~\ref{eq:generation_objective}. 
Together, the progressive denoising optimization and CFG++ sampling help produce more semantically faithful reference images and higher-quality high-risk samples for downstream classifier training.

\subsection{Theoretical Analysis}

\noindent We provide four theoretical results to connect CADS with Distributionally Robust Optimization (DRO). Theorem~\ref{thm:semantic_consistency} shows that the smoothed reference-denoising objective controls an upper bound on the KL divergence between the generated and smoothed reference latent distributions. Theorem~\ref{thm:target_risk_bound} relates the worst-case risk over the resulting semantic ambiguity sets to the target-domain risk. Propositions~\ref{prop:semantic_worst_to_generated} accounts for the finite class-wise search gaps of CADS, while Propositions~\ref{prop:generated_risk_gap} analyzes the use of a finite generated sample set.

\begin{theorem}[\textbf{Distributional proximity induced by reference-based denoising regularization}]
\label{thm:semantic_consistency}
For class $k$, let $P_A^k$ denote the smoothed clean-latent distribution induced by the encoded reference images in $\mathcal A_k$ from Eq.~\ref{eq:reference_set}.
Let $P_{\boldsymbol{\phi}}^k$ denote the clean-latent distribution generated by the reverse diffusion process conditioned on $T_{k,\mathrm{ref}}$.

\vspace{2.0mm}

For class $k$, the reference denoising loss in Eq.~\ref{eq:reference_denoising_loss} can be written as:
\begin{equation}
\mathcal L_{\mathrm{den}}^k(\boldsymbol{\phi};\mathcal A_k)
=
\mathbb E_{\mathbf z_0^a\sim P_A^k,\;t,\;\boldsymbol{\epsilon}}
\left[
\left\|
\boldsymbol{\epsilon}
-
\epsilon_{\phi}
(z_t^a,t,T_{k,\mathrm{ref}})
\right\|_2^2
\right],
\end{equation}
where $z_t^a$ is obtained from $z_0^a$ by the forward diffusion process in Eq.~\ref{eq:reference_forward_diffusion}.

\vspace{2.0mm}

Let $q_{A,T}^k$ denote the distribution obtained by adding Gaussian diffusion noise to $\mathbf z_0^a\sim P_A^k$ at step $T$.
Assume that $D_{\mathrm{KL}}(q_{A,T}^k\Vert p_{\mathrm{prior}})<\infty$, where $p_{\mathrm{prior}}=\mathcal N(\mathbf 0,\mathbf I)$, and that the reverse process uses the same fixed variance schedule as the forward process.
Then there exist constants $\gamma_k>0$ and $C_k$, independent of $\phi$, such that:
\begin{equation}
D_{\mathrm{KL}}(P_A^k\Vert P_{\boldsymbol{\phi}}^k)
\le
\gamma_k
\mathcal L_{\mathrm{den}}^k(\phi;\mathcal A_k)
+
C_k .
\label{eq:first_statement}
\end{equation}

For a semantic tolerance $\delta_{\mathrm{sem}}$, define:
\begin{equation}
\mathcal B_{\mathrm{sem}}^k(\delta_{\mathrm{sem}})
=
\left\{
P_{\phi}^k
\;\middle|\;
\mathcal L_{\mathrm{den}}^k(\phi;\mathcal A_k)
\le
\delta_{\mathrm{sem}}
\right\}.
\label{eq:semantic_ambiguity_set}
\end{equation}
Then, for any $P_{\phi}^k\in\mathcal B_{\mathrm{sem}}^k(\delta_{\mathrm{sem}})$,
\begin{equation}
D_{\mathrm{KL}}(P_A^k\Vert P_{\phi}^k)
\le
\gamma_k\delta_{\mathrm{sem}}+C_k .
\label{eq:second_statement}
\end{equation}

\textbf{Implication.}
The reference denoising loss controls an upper bound on the KL divergence between the smoothed reference distribution and the generated distribution. Thus, Eq.~\ref{eq:generation_objective} can be interpreted as a penalized relaxation of searching for high-risk samples under a class-level semantic constraint.
\end{theorem}

\begin{theorem}[\textbf{Target-domain risk bound in the decoded latent space}]
\label{thm:target_risk_bound}
Consider an unseen target domain $\mathcal D^{\mathcal T}$.
For class $k$, let $P_T^k$ denote its latent distribution and $\pi_T^k$ denote its class prior.
For any latent distribution $P$ of class-$k$ samples, we measure the risk of classifier $f_{\theta}$ by the expected classification loss:
\begin{equation}
\mathcal R_k(f_{\theta};P)
=
\mathbb E_{z\sim P}
\left[
\ell_{\mathrm{CE}}
\left(
f_{\theta}(\mathcal F_D(z)), k
\right)
\right],
\end{equation}
where $\mathcal F_D$ denotes the VAE decoder.
The target-domain risk can be defined as:
\begin{equation}
\mathcal R_T(f_{\theta})
=
\sum_{k=1}^{C}
\pi_T^k
\mathcal R_k(f_{\theta};P_T^k).
\end{equation}

For all classes $k$, assume that the classification loss is bounded as
$0\le \ell_{\mathrm{CE}}\le L_{\max}$
for samples from $P_T^k$ and from any
$P\in\mathcal B_{\mathrm{sem}}^k(\delta_{\mathrm{sem}})$.
Suppose that each $P_T^k$ has distance at most $\rho_k$ to the semantic diffusion ambiguity set in Theorem~\ref{thm:semantic_consistency}:
\begin{equation}
\inf_{P\in \mathcal B_{\mathrm{sem}}^k(\delta_{\mathrm{sem}})}
d_{\mathrm{TV}}(P_T^k,P)
\le
\rho_k,
\end{equation}
where $d_{\mathrm{TV}}(P_T^k,P)=\sup_A|P_T^k(A)-P(A)|$ denotes the total variation distance.
Then:
\begin{equation}
\begin{aligned}
\mathcal R_T(f_{\theta})
&\le
\sum_{k=1}^{C}
\pi_T^k
\sup_{P\in \mathcal B_{\mathrm{sem}}^k(\delta_{\mathrm{sem}})}
\mathcal R_k(f_{\theta};P) \\ 
&+
L_{\max}
\sum_{k=1}^{C}
\pi_T^k\rho_k .
\end{aligned}
\end{equation}

\textbf{Implication.}
The target risk is bounded by the worst-case risk over the semantic ambiguity set and the remaining coverage gap.
Therefore, reducing the worst-case risk over the semantic ambiguity sets $\mathcal B_{\mathrm{sem}}^k(\delta_{\mathrm{sem}})$ tightens the target-risk bound when the coverage gaps $\rho_k$ are small. This motivates training the classifier on high-risk samples searched within these sets.
\end{theorem}

Theorem~\ref{thm:target_risk_bound} bounds the target-domain risk
in terms of the class-wise worst-case risks over the semantic
ambiguity sets. In CADS, the class-conditional distributions are
generated jointly by optimizing the shared LoRA parameters in
Eq.~\ref{eq:generation_objective}, while keeping the classifier fixed.
Since all classes share the same LoRA parameters and the optimization
is performed for a finite number of steps, the resulting distribution
for a given class may not attain the corresponding worst-case risk.
We describe this difference using a class-wise search gap.

\begin{assumption}[\textbf{Semantic constraint and class-wise search gap}]
\label{assump:approx_adv_search}
At synthesis round $r$, the classifier
$f_{\theta^{(r-1)}}$ is fixed.
For each class $k$, we assume that the distribution generated at
synthesis round $r$ satisfies the semantic constraint defined in
Eq.~\ref{eq:semantic_ambiguity_set}, i.e.,
\begin{equation}
P_{\phi^{(r)}}^k
\in
\mathcal B_{\rm sem}^k(\delta_{\rm sem}).
\label{eq:generated_distribution_constraint}
\end{equation}
Under Eq.~\ref{eq:generated_distribution_constraint}, we define the class-wise search gap as:
\begin{equation}
\begin{aligned}
\varepsilon^{(r)}_{{\rm adv},k}
&={}
\sup_{P\in\mathcal B^k_{\rm sem}(\delta_{\rm sem})}
\mathcal R_k
\left(
f_{\theta^{(r-1)}};P
\right) \\
&-
\mathcal R_k
\left(
f_{\theta^{(r-1)}};
P^k_{\phi^{(r)}}
\right)
\ge 0 .
\label{eq:approx_adv_search}
\end{aligned}
\end{equation}
This gap is zero if the CADS-generated distribution attains the
class-wise worst-case risk, and is positive otherwise.
\end{assumption}

Assumption~\ref{assump:approx_adv_search} links the class-conditional
distributions generated by CADS to the semantic ambiguity sets and
introduces $\varepsilon^{(r)}_{{\rm adv},k}$ as the class-wise search
gap.

\begin{proposition}[\textbf{Target-risk bound with class-wise search gaps}]
\label{prop:semantic_worst_to_generated}
Under Assumption~\ref{assump:approx_adv_search}, for any nonnegative class weights $\{\pi^k\}_{k=1}^{C}$ satisfying: $\sum_{k=1}^{C}\pi^k=1$, we have:
\begin{equation}
\begin{aligned}
&\sum_{k=1}^{C}\pi^k
\sup_{P\in \mathcal B^k_{\rm sem}(\delta_{\rm sem})}
\mathcal{R}_k(f_{\theta^{(r-1)}};P) \\
&=
\sum_{k=1}^{C}\pi^k
\mathcal{R}_k(f_{\theta^{(r-1)}};P^k_{\phi^{(r)}})
+ 
\sum_{k=1}^{C}\pi^k
\varepsilon^{(r)}_{{\rm adv},k}.
\label{eq:weighted_worst_to_generated}
\end{aligned}
\end{equation}
In particular, taking $\pi^k=\pi_T^k$ and combining
Eq.~\ref{eq:weighted_worst_to_generated} with
Theorem~\ref{thm:target_risk_bound} gives:
\begin{equation}
\begin{aligned}
\mathcal{R}_T(f_{\theta^{(r-1)}})
&\le 
\sum_{k=1}^{C}\pi_T^k
\mathcal{R}_k(f_{\theta^{(r-1)}};P^k_{\phi^{(r)}}) \\ 
&+ 
\sum_{k=1}^{C}\pi_T^k
\varepsilon^{(r)}_{{\rm adv},k}
+
L_{\max}
\sum_{k=1}^{C}\pi_T^k \rho_k .
\label{eq:target_bound_generated_distribution}
\end{aligned}
\end{equation}
For class-balanced risk $\pi^k=1/C$, let:
\begin{equation}
\mathcal R_g^{(r)}
\left(
f_{\theta^{(r-1)}}
\right)
:=
\frac{1}{C}
\sum_{k=1}^{C}
\mathcal R_k
\left(
f_{\theta^{(r-1)}};
P^k_{\phi^{(r)}}
\right),
\end{equation}
and define:
\begin{equation}
\varepsilon_{\rm adv}^{(r)}
=
\frac{1}{C}
\sum_{k=1}^{C}
\varepsilon^{(r)}_{{\rm adv},k}.
\label{eq:avg_adv_error}
\end{equation}
Then Eq.~\ref{eq:weighted_worst_to_generated} reduces to:
\begin{equation}
\frac{1}{C}
\sum_{k=1}^{C}
\sup_{P\in\mathcal B^k_{\rm sem}(\delta_{\rm sem})}
\mathcal R_k
\left(
f_{\theta^{(r-1)}};P
\right)
=
\mathcal R_g^{(r)}
\left(
f_{\theta^{(r-1)}}
\right)
+
\varepsilon_{\rm adv}^{(r)}.
\label{eq:balanced_worst_to_generated}
\end{equation}
\end{proposition}

\begin{proposition}[\textbf{Finite-sample bound for generated risk}]
\label{prop:generated_risk_gap}
Under Assumption~\ref{assump:approx_adv_search}, the class-balanced
generated-distribution risk satisfies:
\begin{equation}
\mathcal R_g^{(r)}(f_{\theta})
\le
\frac{1}{C}
\sum_{k=1}^{C}
\sup_{P\in\mathcal B_{\mathrm{sem}}^k(\delta_{\mathrm{sem}})}
\mathcal R_k(f_{\theta};P).
\label{eq:generated_risk_semantic_bound}
\end{equation}

For the generated dataset $\mathcal D_g^{(r)}$, define the empirical
generated risk as:
\begin{equation}
\widehat{\mathcal R}_g^{(r)}
(f_{\theta})
=
\frac{1}{CM_g}
\sum_{k=1}^{C}
\sum_{j=1}^{M_g}
\ell_{\mathrm{CE}}
\left(
f_{\theta}
(\tilde{\mathbf x}_{k,j}^{(r)}),k
\right).
\label{eq:empirical_generated_risk}
\end{equation}
This is the generated-data loss in
Eq.~\ref{eq:classifier_objective} before multiplication by $\mu$.

\vspace{2.0mm}

Conditioned on the distributions
$\{P_{\phi^{(r)}}^k\}_{k=1}^{C}$,
assume that these distributions are fixed before constructing
$\mathcal D_g^{(r)}$. The $CM_g$ latent samples are then drawn
independently, with $M_g$ samples drawn from each
$P_{\phi^{(r)}}^k$, and decoded to obtain
$\{\tilde{x}_{k,j}^{(r)}\}_{k,j}$. 
We further assume that:
\begin{equation}
0
\le
\ell_{\mathrm{CE}}
\left(
f_{\theta}
(\tilde{x}_{k,j}^{(r)}),k
\right)
\le
L_{\max}
\end{equation}
for all $f_{\theta}\in\mathcal H$.

\vspace{2.0mm}

Let $\mathfrak R_{CM_g}^{(r)}
(\ell_{\mathrm{CE}}\circ\mathcal H)$
denote the expected Rademacher complexity of the loss class under
this class-balanced sampling process. Then, with probability at least
$1-\delta$, for all
$f_{\theta}\in\mathcal H$,
\begin{equation}
\begin{aligned}
\left|
\mathcal R_g^{(r)}(f_{\theta})
-
\widehat{\mathcal R}_g^{(r)}(f_{\theta})
\right|
&\le{}
2\mathfrak R_{CM_g}^{(r)}
(\ell_{\mathrm{CE}}\circ\mathcal H) \\
&+
L_{\max}
\sqrt{
\frac{\log(2/\delta)}{2CM_g}
}.
\label{eq:generated_risk_gap_bound}
\end{aligned}
\end{equation}

For convenience, define:
\begin{equation}
\Delta_g^{(r)}(\delta)=
2\mathfrak R_{CM_g}^{(r)}
(\ell_{\mathrm{CE}}\circ\mathcal H)
+
L_{\max}
\sqrt{
\frac{\log(2/\delta)}{2CM_g}
}.
\label{eq:finite_sample_gap}
\end{equation}

On the same event, in the class-balanced setting, where $\pi_T^k=1/C$, combining
Proposition~\ref{prop:semantic_worst_to_generated} with
Eq.~\ref{eq:generated_risk_gap_bound} gives:
\begin{equation}
\begin{aligned}
\underbrace{
\mathcal R_T
\left(
f_{\theta^{(r-1)}}
\right)
}_{\text{target risk}}
\le{}&
\underbrace{
\widehat{\mathcal R}_g^{(r)}
\left(
f_{\theta^{(r-1)}}
\right)
}_{\text{empirical generated risk}}
+
\underbrace{
\varepsilon_{\rm adv}^{(r)}
}_{\text{search gap}}
\\
&+
\underbrace{
\frac{L_{\max}}{C}
\sum_{k=1}^{C}\rho_k
}_{\text{coverage gap}}
+
\underbrace{
\Delta_g^{(r)}(\delta)
}_{\text{finite-sample gap}} .
\end{aligned}
\end{equation}

\textbf{Implication.}
Eq.~\ref{eq:generated_risk_gap_bound} shows that the generated-data
loss in Eq.~\ref{eq:classifier_objective} is a finite-sample estimate
of the classifier risk under the class-conditional distributions
generated by CADS. For a fixed hypothesis class whose Rademacher
complexity decreases with the sample size, this estimation gap
decreases as $M_g$ increases.
\end{proposition}

\begin{table*}[!t]
\begin{center}
\caption{
SDG results (\%) on CIFAR-10-C with the backbone of ResNet-18. 
Each level (name in column) is viewed as a target domain, and a higher level denotes more severe corruption.
The best results are in \textcolor{red}{red} and the second highest results are in \textcolor{blue}{blue}.
}
\vspace{-2.0mm}
\label{tab:single-source-cifar10c}
\renewcommand\arraystretch{1.0}
\scalebox{1.0}{
\begin{tabular}{l@{\hspace{16pt}}|@{\hspace{14pt}}c@{\hspace{14pt}}|@{\hspace{16pt}}c@{\hspace{16pt}}c@{\hspace{16pt}}c@{\hspace{16pt}}c@{\hspace{16pt}}c@{\hspace{16pt}}|>{\hspace{13pt}}>{\columncolor{gray!30}}c<{\hspace{13pt}}}
\hline
Method & Venue 
& level1 & level2 & level3 & level4 & level5 & Avg. \\ 
\hline
\multicolumn{8}{c}{\textit{ResNet-18}} \\
\hline
DANN~\cite{ganin2016domain}       
& IJCAI'16    
& 75.40\tsb{\(\pm\)0.40} 
& 72.60\tsb{\(\pm\)0.30}
& 69.70\tsb{\(\pm\)0.20}  
& 65.60\tsb{\(\pm\)0.00}
& 59.60\tsb{\(\pm\)0.20}
& 68.60 \\

CORAL~\cite{sun2016deep}       
& ICCV'16     
& 76.00\tsb{\(\pm\)0.40}
& 72.90\tsb{\(\pm\)0.20}
& 69.90\tsb{\(\pm\)0.00}  
& 65.80\tsb{\(\pm\)0.10}
& 59.60\tsb{\(\pm\)0.10}
& 68.80 \\

MMD~\cite{li2018domain}
& ICCV'18     
& 76.40\tsb{\(\pm\)0.40}
& 72.90\tsb{\(\pm\)0.20}
& 69.90\tsb{\(\pm\)0.00}  
& 65.80\tsb{\(\pm\)0.10}
& 59.60\tsb{\(\pm\)0.40}
& 68.80 \\


Mixup~\cite{yan2020improve}
& -           
& 76.30\tsb{\(\pm\)0.30}
& 73.20\tsb{\(\pm\)0.20}
& 70.20\tsb{\(\pm\)0.20}  
& 66.10\tsb{\(\pm\)0.10}
& 60.10\tsb{\(\pm\)0.10}
& 69.20 \\



GroupDRO~\cite{sagawa2019distributionally}    
& ICLR'20     
& 76.00\tsb{\(\pm\)0.10}
& 72.90\tsb{\(\pm\)0.10}
& 69.80\tsb{\(\pm\)0.20}  
& 65.50\tsb{\(\pm\)0.30}
& 59.50\tsb{\(\pm\)0.50}
& 68.70 \\

RSC~\cite{huang2020self}         
& ECCV'20     
& 76.10\tsb{\(\pm\)0.40}
& 73.20\tsb{\(\pm\)0.50}
& 70.10\tsb{\(\pm\)0.50}
& 66.20\tsb{\(\pm\)0.50}
& 60.10\tsb{\(\pm\)0.50}
& 69.10 \\


ARM~\cite{zhang2021adaptive}
& NeurIPS'21  
& 75.70\tsb{\(\pm\)0.10}
& 72.90\tsb{\(\pm\)0.10}
& 69.90\tsb{\(\pm\)0.20}  
& 65.90\tsb{\(\pm\)0.20}
& 59.80\tsb{\(\pm\)0.30}
& 68.80 \\

VREx~\cite{krueger2021out}        
& ICML'21     
& 76.00\tsb{\(\pm\)0.20}
& 73.00\tsb{\(\pm\)0.20}
& 70.00\tsb{\(\pm\)0.20}
& 66.00\tsb{\(\pm\)0.10}
& 60.00\tsb{\(\pm\)0.20}
& 69.00 \\

ERM~\cite{vapnik2013nature}
& ICLR'21     
& 75.90\tsb{\(\pm\)0.50}
& 72.90\tsb{\(\pm\)0.40}
& 70.00\tsb{\(\pm\)0.40}  
& 65.90\tsb{\(\pm\)0.40}
& 59.90\tsb{\(\pm\)0.50}
& 68.90 \\

SAM~\cite{foret2021sharpness}
& ICLR'21     
& 79.00\tsb{\(\pm\)0.30}
& 76.00\tsb{\(\pm\)0.30}
& 72.90\tsb{\(\pm\)0.30}  
& 68.70\tsb{\(\pm\)0.20}
& 62.50\tsb{\(\pm\)0.30}
& 71.80 \\

SagNet~\cite{nam2021reducing}
& CVPR'21     
& 76.60\tsb{\(\pm\)0.20}
& 73.60\tsb{\(\pm\)0.30}
& 70.50\tsb{\(\pm\)0.40}  
& 66.40\tsb{\(\pm\)0.40}
& 60.10\tsb{\(\pm\)0.40}
& 69.50 \\

Fishr~\cite{rame2022fishr}
& ICML'22     
& 76.30\tsb{\(\pm\)0.30}
& 73.40\tsb{\(\pm\)0.30}
& 70.40\tsb{\(\pm\)0.50}  
& 66.30\tsb{\(\pm\)0.80}
& 60.10\tsb{\(\pm\)1.10}
& 69.30 \\

SAGM~\cite{wang2023sharpness}
& CVPR'23     
& 79.00\tsb{\(\pm\)0.10}
& 76.20\tsb{\(\pm\)0.00}
& 73.20\tsb{\(\pm\)0.20}  
& 69.00\tsb{\(\pm\)0.30}
& 62.70\tsb{\(\pm\)0.40}
& 72.00 \\


UDIM w/ SAM~\cite{shinunknown}
& ICLR'24
& 80.30\tsb{\(\pm\)0.00}
& 77.70\tsb{\(\pm\)0.10}
& 75.10\tsb{\(\pm\)0.00}
& 71.50\tsb{\(\pm\)0.10}
& 66.20\tsb{\(\pm\)0.10}
& 74.20 \\

UDIM w/ SAGM~\cite{shinunknown}
& ICLR'24
& 80.10\tsb{\(\pm\)0.10}
& 77.50\tsb{\(\pm\)0.10}
& 74.80\tsb{\(\pm\)0.10}
& 71.20\tsb{\(\pm\)0.20}
& 65.90\tsb{\(\pm\)0.20}
& 73.90 \\

UDIM w/ GAM~\cite{shinunknown}
& ICLR'24
& 81.40\tsb{\(\pm\)0.10}
& \textcolor{blue}{78.90\tsb{\(\pm\)0.00}}
& \textcolor{blue}{76.30\tsb{\(\pm\)0.00}}
& \textcolor{blue}{72.80\tsb{\(\pm\)0.10}}
& \textcolor{blue}{67.40\tsb{\(\pm\)0.10}}
& \textcolor{blue}{75.30} \\
\hline              

PAPT~\cite{PAPT}
& CVPR'25
& \textcolor{blue}{83.02\tsb{\(\pm\)0.09}}
& 78.82\tsb{\(\pm\)0.31}
& 75.51\tsb{\(\pm\)0.13}
& 70.13\tsb{\(\pm\)0.20}
& 63.05\tsb{\(\pm\)0.35}
& 74.11 \\

PAPT++ (Ours)  
& -        
& \textcolor{red}{90.44\tsb{\(\pm\)2.53}}
& \textcolor{red}{87.72\tsb{\(\pm\)1.86}}
& \textcolor{red}{84.65\tsb{\(\pm\)2.01}}
& \textcolor{red}{79.46\tsb{\(\pm\)3.14}}
& \textcolor{red}{69.45\tsb{\(\pm\)2.88}}
& \textcolor{red}{82.34} \\
\hline
\end{tabular}}
\end{center}
\vspace{-2.0mm}
\end{table*}

\section{Experiments}
\label{sec:exp}

\subsection{Datasets and Evaluation Protocols}
\label{sec:exp_data}

\noindent \textbf{Benchmark Datasets.}
We conduct experiments on three widely used SDG
benchmarks included in DomainBed~\cite{gulrajani2020search}, namely
PACS~\cite{li2017deeper}, VLCS~\cite{li2017deeper}, and
OfficeHome~\cite{venkateswara2017deep}. We further evaluate corruption
generalization by training on clean CIFAR-10~\cite{krizhevsky2009learning}
and testing on CIFAR-10-C~\cite{hendrycks2019benchmarking}.

PACS contains 9,991 images from seven categories distributed across four
domains: Art Painting, Cartoon, Photo, and Sketch.
VLCS contains 10,729 images from five categories collected from four
datasets: Caltech101, LabelMe, SUN09, and VOC2007.
OfficeHome contains 15,588 images from 65 categories distributed across
four domains: Art, Clipart, Product, and Real World.
CIFAR-10 contains 50,000 training images and 10,000 test images from ten
categories. CIFAR-10-C applies multiple common corruption types at five
severity levels to the CIFAR-10 test images, providing a controlled
benchmark for evaluating robustness to unseen corruptions.

\vspace{2.0mm}

\noindent \textbf{Single-Domain Evaluation Protocol.}
For PACS, VLCS, and OfficeHome, each domain is used once as the sole
labeled source domain, while all remaining domains are treated as unseen
target domains. The result for each source-domain setting is averaged
over its corresponding unseen target domains, and the overall performance
is averaged over all source-domain settings. Target-domain data are not
used during training or model selection. Each experiment is repeated
three times with different random seeds, and the mean accuracy is reported.

\vspace{2.0mm}

\noindent \textbf{Corruption Evaluation Protocol.}
For CIFAR-10-C, the model is trained using only the clean CIFAR-10
training set and is directly evaluated on CIFAR-10-C without adaptation.
We report the average classification accuracy across corruption types at
each severity level. CIFAR-10-C images are not used during training or
model selection.

\begin{table}[!t]
\caption{
SDG results (\%) on PACS with the backbone of ResNet-18.
One domain is used as the source domain and the others are used as the target domain.
The best results are in \textcolor{red}{red} and the second highest results are in \textcolor{blue}{blue}.
}
\vspace{-2.0mm}
\renewcommand\arraystretch{1.0}
\scalebox{0.88}{
\centering
\begin{tabular}{l|c|cccc|>{\columncolor{gray!30}}c}
\hline
Method & Venue & A & C & P & S & Avg. \\
\hline
\multicolumn{7}{c}{\textit{ResNet-18}} \\
\hline
Augmix~\cite{hendrycks2019augmix}
& ICLR'20
& 66.54
& 70.16
& 38.30
& 52.48
& 56.87 \\

RSC~\cite{huang2020self}
& ECCV'20
& 73.40
& 75.90
& 41.60
& 56.20
& 61.80 \\

L2D~\cite{wang2021learning}
& ICCV'21
& 76.91
& 77.88
& 52.29
& 53.66
& 65.18 \\

RSC+ASR~\cite{huang2020self}
& CVPR'21
& 76.70
& 79.30
& 54.60
& 61.60
& 68.10 \\

pAdaIn~\cite{nuriel2021permuted}
& CVPR'21
& 64.96
& 65.24
& 33.66
& 32.04
& 49.98 \\

ERM~\cite{vapnik2013nature}
& ICLR'21
& 65.38
& 64.20
& 33.65
& 34.15
& 49.34 \\

Mixstyle~\cite{zhou2021domain}
& ICLR'21
& 67.60
& 70.38
& 37.44
& 34.57
& 52.50 \\

EFDMix~\cite{zhang2022exact}
& CVPR'22
& 63.20
& 73.90
& 42.50
& 38.10
& 54.40 \\

DSU~\cite{li2022uncertainty}
& ICLR'22
& 71.54
& 74.51
& 42.10
& 47.75
& 58.97 \\

ACVC~\cite{cugu2022attention}
& CVPRW'22
& 73.68
& 77.39
& 48.05
& 55.30
& 63.61 \\

MAD~\cite{qu2023modality}
& CVPR'23
& 75.51
& 77.25
& 52.95
& 57.75
& 65.87 \\

P-RC~\cite{choi2023progressive}
& CVPR'23
& 76.98
& 78.54
& 57.11
& 62.89
& 68.88 \\

Meta-Casual~\cite{chen2023meta}
& CVPR'23
& 77.13
& 80.14
& 59.60
& 62.55
& 69.86 \\

ITTA~\cite{chen2023improved}
& CVPR'23
& 78.40
& 79.80
& 56.50
& 60.70
& 68.80 \\

Prompt-Driven~\cite{li2024prompt}
& CVPR'24
& 78.77
& \textcolor{red}{82.69}
& 60.09
& 62.94
& 71.12 \\

UDIM~\cite{shinunknown}
& ICLR'24
& \textcolor{blue}{80.49}
& 80.28
& 60.94
& 65.64
& 71.83 \\

StyDeSty~\cite{liu2024stydesty}
& ICML'24
& 78.90
& 79.20
& \textcolor{red}{62.50}
& 59.40
& 70.00 \\

PSDG~\cite{yang2024practical}
& KDD'24
& 78.20
& 78.30
& 62.10
& 63.00
& 70.40 \\

ProMEA~\cite{wang2025promea}
& IJCAI'25
& \textcolor{red}{82.58}
& \textcolor{blue}{81.43}
& 61.09
& 62.52
& 71.91 \\

MISA~\cite{yang2026mutual}
& TIST'26
& 79.20
& 78.50
& \textcolor{blue}{62.30}
& 64.20
& 71.10 \\
\hline

PAPT~\cite{PAPT}
& CVPR'25
& 69.84
& 78.57
& 62.12
& \textcolor{blue}{80.34}
& \textcolor{blue}{72.72} \\

PAPT++ (Ours)
& -
& 70.01
& 80.45
& 61.08
& \textcolor{red}{82.18}
& \textcolor{red}{73.43} \\
\hline
\end{tabular}}
\label{tab:single_pacs}
\vspace{-2.0mm}
\end{table}

\subsection{Implementation Details}

\noindent \textbf{Diffusion Optimization.}
We adopt Stable Diffusion v1.5~\cite{rombach2022high} as the pretrained text-to-image diffusion model. The pretrained U-Net backbone and text encoder are kept frozen, and only the LoRA parameters inserted into the self-attention layers of the U-Net are optimized. The LoRA rank is set to 4. We optimize the LoRA parameters using AdamW with a learning rate of $6\times10^{-4}$ and a batch size of 64.

For CSRL module, the LoRA parameters are optimized for 80 gradient steps, after which we generate $M_a=32$ semantic reference images for each class. For CADS module, we perform $R=8$ generation-training rounds and optimize the LoRA parameters for 10 gradient steps in each round, resulting in 80 gradient steps in total. The progressive denoising optimization interval $N_{\rm opt}$ is set to 16. The loss weights $\lambda_{\rm div}$ and $\lambda_{\rm risk}$ are set to 0.02 and 0.002, respectively.
After optimizing the LoRA parameters in the $r$-th round, we fix $\phi^{(r)}$ and independently sample a fresh class-balanced set containing $M_g=16$ high-risk images per class to construct $\mathcal D_g^{(r)}$. During both semantic reference generation and high-risk sample generation, we employ the DDIM~\cite{song2021ddim} sampler with 50 denoising steps and CFG++ with a guidance scale of 0.6.

\vspace{2.0mm}

\noindent \textbf{Backbone Training.}
We use ResNet-18 and ResNet-50~\cite{he2016deep} as the downstream classifiers. Following common SDG settings, ResNet-18 is used for PACS, VLCS, and CIFAR-10 whereas ResNet-50 is used for OfficeHome. The classifier is optimized using SGD for 20 epochs in each classifier-update stage, with a batch size of 64 and an initial learning rate of $1\times10^{-3}$. The weight $\mu$ of the generated-data loss in Eq.~\ref{eq:classifier_objective} is set to 1.0.

\begin{table}[!t]
\begin{center}
\caption{
SDG results (\%) on PACS with the backbone of ResNet-50.
One domain (name in column) is used as the source domain and the others are used as the target domains.
The best results are in \textcolor{red}{red} and the second highest results are in \textcolor{blue}{blue}.
}
\vspace{-2.0mm}
\label{tab:single-source-pacs-res50}
\renewcommand\arraystretch{1.0}
\scalebox{0.88}{
\begin{tabular}{l|c|cccc|>{\columncolor{gray!30}}c}
\hline
Method & Venue & A & C & P & S & Avg. \\
\hline
\multicolumn{7}{c}{\textit{ResNet-50}} \\
\hline
DANN~\cite{ganin2016domain}
& IJCAI'16
& 79.00
& 76.50
& 48.70
& 57.90
& 65.50 \\

CORAL~\cite{sun2016deep}
& ICCV'16
& 76.30
& 79.20
& 45.90
& 57.00
& 64.60 \\


MMD~\cite{li2018domain}
& ICCV'18
& 75.40
& 80.10
& 45.20
& 58.20
& 64.70 \\


Mixup~\cite{yan2020improve}
& -
& 77.40
& 80.00
& 47.30
& 58.20
& 65.70 \\



GroupDRO~\cite{sagawa2019distributionally}
& ICLR'20
& 79.00
& 79.00
& 42.00
& 60.80
& 65.20 \\


ARM~\cite{zhang2021adaptive}
& NeurIPS'21
& 76.20
& 75.50
& 45.20
& 61.90
& 64.70 \\

VREx~\cite{krueger2021out}
& ICML'21
& 75.30
& 80.20
& 44.90
& 56.80
& 64.30 \\

Mixstyle~\cite{zhou2021domain}
& ICLR'21
& 78.10
& 78.80
& 56.10
& 54.70
& 66.90 \\

ERM~\cite{vapnik2013nature}
& ICLR'21
& 79.90
& 79.90
& 48.10
& 59.60
& 66.90 \\

SAM~\cite{foret2021sharpness}
& ICLR'21
& 77.70
& 80.50
& 46.70
& 54.20
& 64.80 \\

SagNet~\cite{nam2021reducing}
& CVPR'21
& 77.40
& 78.90
& 47.60
& 56.40
& 65.10 \\

Fishr~\cite{rame2022fishr}
& ICML'22
& 75.90
& 81.10
& 46.90
& 57.20
& 65.30 \\

RIDG~\cite{chen2023domain}
& ICCV'23
& 76.20
& 80.00
& 48.50
& 54.80
& 64.90 \\

SAGM~\cite{wang2023sharpness}
& CVPR'23
& 78.90
& 79.80
& 44.70
& 55.60
& 64.80 \\

ITTA~\cite{chen2023improved}
& CVPR'23
& 78.40
& 79.80
& 56.50
& 60.70
& 68.80 \\

UDIM~\cite{shinunknown}
& ICLR'24
& \textcolor{red}{82.40}
& \textcolor{red}{84.20}
& \textcolor{blue}{68.80}
& 64.00
& 74.90 \\

FSAM~\cite{li2024friendly}
& CVPR'24
& 79.23
& 82.55
& 48.83
& 60.40
& 67.75 \\

Crafting-Shifts~\cite{efthymiadis2025crafting}
& WACV'24
& 81.14
& 78.34
& 60.59
& 68.13
& 72.05 \\

StyDeSty~\cite{liu2024stydesty}
& ICML'24
& 80.60
& \textcolor{blue}{83.90}
& 63.90
& 66.10
& 73.60 \\

PSDG~\cite{yang2024practical}
& KDD'24
& 81.10
& 83.40
& 63.40
& 63.60
& 72.90 \\

GCSAM~\cite{liu2024generalizable}
& TMM'25
& 78.00
& 82.84
& 50.35
& 57.50
& 67.17 \\

SSESAM~\cite{lyu2025sse}
& AAAI'25
& 77.14
& 82.08
& 46.72
& 65.78
& 67.93 \\

PhysAug~\cite{xu2025physaug}
& AAAI'25
& 79.04
& 81.28
& 57.45
& 62.16
& 69.98 \\

SAML~\cite{zhou2025sharpness}
& ICLR'25
& 68.18
& 81.45
& 49.53
& 63.89
& 65.77 \\

MISA~\cite{yang2026mutual}
& TIST'26
& \textcolor{blue}{82.30}
& 82.80
& 64.00
& 67.50
& 74.20 \\
\hline

PAPT~\cite{PAPT}
& CVPR'25
& 73.56
& 80.10
& 67.83
& \textcolor{red}{85.85}
& \textcolor{blue}{76.84} \\

PAPT++ (Ours)
& -
& 79.21
& 82.07
& \textcolor{red}{74.49}
& \textcolor{blue}{81.72}
& \textcolor{red}{79.37} \\
\hline
\end{tabular}}
\end{center}
\vspace{-2.0mm}
\end{table}

\begin{table}[!t]
\caption{
SDG results (\%) on VLCS with the backbone of ResNet-18.
One domain is used as the source domain and the others are used as the target domains.
The best results are in \textcolor{red}{red} and the second highest results are in \textcolor{blue}{blue}.
}
\vspace{-2.0mm}
\centering
\renewcommand\arraystretch{1.0}
\scalebox{0.91}{
\begin{tabular}{l|c|cccc|>{\columncolor{gray!30}}c}
\hline
Method & Venue & V & L & C & S & Avg. \\
\hline
\multicolumn{7}{c}{\textit{ResNet-18}} \\
\hline
Augmix~\cite{hendrycks2019augmix}
& ICLR'20
& 75.25
& 59.52
& 45.90
& 57.43
& 59.53 \\

ERM~\cite{vapnik2013nature}
& ICLR'21
& 76.72
& 58.86
& 44.95
& 57.71
& 59.56 \\

pAdaIn~\cite{nuriel2021permuted}
& CVPR'21
& 76.03
& 65.21
& 43.17
& 57.94
& 60.59 \\

Mixstyle~\cite{zhou2021domain}
& ICLR'21
& 75.73
& 61.29
& 44.66
& 56.57
& 59.56 \\

EFDMix~\cite{zhang2022exact}
& CVPR'22
& 72.35
& 61.41
& 52.34
& 63.28
& 62.33 \\

DSU~\cite{li2022uncertainty}
& ICLR'22
& \textcolor{blue}{76.93}
& 69.20
& 46.54
& 58.36
& 62.76 \\

ACVC~\cite{cugu2022attention}
& CVPRW'22
& 76.15
& 61.23
& 47.43
& 60.18
& 61.25 \\

MAD~\cite{qu2023modality}
& CVPR'23
& 76.15
& 69.36
& 48.04
& 61.74
& 63.82 \\

StyDeSty~\cite{liu2024stydesty}
& ICML'24
& 76.87
& 62.87
& 53.73
& 65.41
& 64.72 \\

ProMEA~\cite{wang2025promea}
& IJCAI'25
& \textcolor{red}{79.07}
& 71.97
& 59.02
& 59.94
& 67.50 \\
\hline

PAPT~\cite{PAPT}
& CVPR'25
& 73.16
& \textcolor{blue}{74.69}
& \textcolor{blue}{69.66}
& \textcolor{red}{75.96}
& \textcolor{blue}{73.37} \\

PAPT++ (Ours)
& -
& 73.07
& \textcolor{red}{77.22}
& \textcolor{red}{71.42}
& \textcolor{blue}{75.14}
& \textcolor{red}{74.21} \\
\hline
\end{tabular}}
\label{tab:single_vlcs}
\vspace{-2.0mm}
\end{table}

\begin{table}[!t]
\begin{center}
\caption{
SDG results (\%) on OfficeHome with the backbone of ResNet-50.
One domain is used as the source domain and the others are used as the target domains.
The best results are in \textcolor{red}{red} and the second highest results are in \textcolor{blue}{blue}.
}
\vspace{-2.0mm}
\label{tab:single_officehome}
\renewcommand\arraystretch{1.0}
\scalebox{0.86}{
\begin{tabular}{l|c|cccc|>{\columncolor{gray!30}}c}
\hline
Method & Venue & A & C & P & R & Avg. \\
\hline
\multicolumn{7}{c}{\textit{ResNet-50}} \\
\hline
DANN~\cite{ganin2016domain}
& IJCAI'16
& 55.20
& 49.30
& 48.40
& 58.40
& 52.80 \\

CORAL~\cite{sun2016deep}
& ICCV'16
& 55.60
& 52.80
& 50.30
& 59.40
& 54.50 \\


MMD~\cite{li2018domain}
& ICCV'18
& 55.10
& 52.00
& 50.30
& 59.30
& 54.20 \\


Mixup~\cite{yan2020improve}
& -
& 55.50
& 54.10
& 49.40
& 59.40
& 54.60 \\



GroupDRO~\cite{sagawa2019distributionally}
& ICLR'20
& 55.10
& 52.00
& 50.30
& 59.30
& 54.20 \\


ARM~\cite{zhang2021adaptive}
& NeurIPS'21
& 55.00
& 51.60
& 47.30
& 59.30
& 53.30 \\

VREx~\cite{krueger2021out}
& ICML'21
& 55.50
& 52.60
& 49.10
& 59.30
& 54.10 \\

Mixstyle~\cite{zhou2021domain}
& ICLR'21
& 44.30
& 29.80
& 33.60
& 48.50
& 39.00 \\

ERM~\cite{vapnik2013nature}
& ICLR'21
& 55.60
& 52.80
& 50.30
& 59.40
& 54.50 \\

SAM~\cite{foret2021sharpness}
& ICLR'21
& 56.90
& 53.80
& 50.90
& 61.50
& 55.80 \\

SagNet~\cite{nam2021reducing}
& CVPR'21
& 56.90
& 53.40
& 50.80
& 61.20
& 55.60 \\

Fishr~\cite{rame2022fishr}
& ICML'22
& 55.10
& 51.20
& 49.20
& 59.90
& 53.90 \\

RIDG~\cite{chen2023domain}
& ICCV'23
& 56.80
& 55.40
& 50.50
& 60.90
& 55.90 \\

SAGM~\cite{wang2023sharpness}
& CVPR'23
& 57.70
& 54.80
& 51.50
& 61.40
& 56.30 \\

ITTA~\cite{chen2023improved}
& CVPR'23
& 56.00
& 51.50
& 50.50
& 61.60
& 54.90 \\

UDIM~\cite{shinunknown}
& ICLR'24
& 58.50
& 55.70
& 54.50
& 64.50
& 58.30 \\

FSAM~\cite{li2024friendly}
& CVPR'24
& 58.15
& 56.41
& 52.59
& 62.20
& 57.34 \\

Crafting-Shifts~\cite{efthymiadis2025crafting}
& WACV'24
& 59.77
& 55.31
& 51.46
& 63.10
& 57.31 \\

StyDeSty~\cite{liu2024stydesty}
& ICML'24
& 59.09
& 57.78
& 54.73
& \textcolor{blue}{65.02}
& 59.16 \\

ProMEA~\cite{wang2025promea}
& IJCAI'25
& \textcolor{blue}{62.52}
& 56.98
& 56.02
& \textcolor{red}{65.64}
& 60.29 \\

GCSAM~\cite{liu2024generalizable}
& TMM'25
& 56.86
& 56.45
& 53.06
& 62.24
& 57.15 \\

SSESAM~\cite{lyu2025sse}
& AAAI'25
& 55.16
& 54.82
& 50.65
& 60.08
& 55.18 \\

PhysAug~\cite{xu2025physaug}
& AAAI'25
& 54.30
& 52.43
& 49.57
& 60.10
& 54.10 \\

SAML~\cite{zhou2025sharpness}
& ICLR'25
& 44.56
& 56.48
& 53.04
& 62.18
& 54.07 \\
\hline

PAPT~\cite{PAPT}
& CVPR'25
& 59.81
& \textcolor{blue}{68.30}
& \textcolor{blue}{59.12}
& 60.87
& \textcolor{blue}{62.03} \\

PAPT++ (Ours)
& -
& \textcolor{red}{64.72}
& \textcolor{red}{70.91}
& \textcolor{red}{61.94}
& 63.87
& \textcolor{red}{65.36} \\
\hline
\end{tabular}}
\end{center}
\vspace{-2.0mm}
\end{table}

\subsection{Comparison with State-of-the-Art Methods}

\noindent We compare PAPT++ with ERM~\cite{vapnik2013nature} and a broad
range of representative single-source DG methods, including data
augmentation-based approaches~\cite{hendrycks2019augmix,
nuriel2021permuted, zhou2021domain, zhang2022exact,
li2022uncertainty, cugu2022attention, qu2023modality},
domain-invariant representation learning approaches~\cite{ganin2016domain,
sun2016deep,
krueger2021out, li2018domain}, feature disentanglement and
distributionally robust optimization methods~\cite{nam2021reducing,
sagawa2019distributionally}, gradient- and flatness-aware optimization
methods~\cite{huang2020self, rame2022fishr, wang2023sharpness,
shinunknown, lyu2025sse, liu2024generalizable, li2024friendly},
and recent SDG methods~\cite{PAPT, zhou2025sharpness, xu2025physaug}.

Following recent DG studies~\cite{qu2023modality,guo2023domaindrop},
we adopt ResNet-18 on PACS and VLCS and ResNet-50 on OfficeHome.
As reported in Tabs.~\ref{tab:single_pacs}--\ref{tab:single_officehome},
PAPT++ achieves the highest average accuracy on all three benchmarks,
reaching 73.43\%, 74.21\%, and 65.36\%, respectively. These results
improve upon PAPT by 0.71, 0.84, and 3.33 pp, respectively. The gains
are particularly pronounced when Cartoon and Sketch serve as the
source domains on PACS and when LabelMe and Caltech serve as the
source domains on VLCS. Consistent improvements are also observed
across all four source-domain settings on OfficeHome. Overall, these
results show that risk-aware synthesis improves generalization across
diverse semantic domain shifts.

Following the experimental setting of UDIM~\cite{shinunknown}, we
additionally evaluate PAPT++ with ResNet-50 on PACS. As shown in
Tab.~\ref{tab:single-source-pacs-res50}, PAPT++ achieves the highest
average accuracy of 79.37\%, surpassing PAPT and UDIM by 2.53 and
4.47 pp, respectively. Together with the ResNet-18 results, this
comparison indicates that the gains of PAPT++ persist when scaling
the backbone from ResNet-18 to ResNet-50.

We further assess corruption generalization by training ResNet-18 on clean CIFAR-10 and directly evaluating it on CIFAR-10-C, following UDIM~\cite{shinunknown}. As reported in Tab.~\ref{tab:single-source-cifar10c}, PAPT++ achieves the highest accuracy at all five corruption severity levels, with an average accuracy of 82.34\%. It outperforms PAPT and UDIM with GAM by 8.23 and 7.04 pp, respectively. The consistent gains across severity levels demonstrate the robustness of PAPT++ to unseen corruptions.

\begin{table}[!t]
\begin{center}
\caption{
Multi-source DG results (\%) on OfficeHome with the backbone of ResNet-50.
One domain is used as the target domain and the others are used as the source domains.
The best results are in \textcolor{red}{red} and the second highest results are in \textcolor{blue}{blue}.
}
\vspace{-2.0mm}
\label{tab:multi-source-officehome}
\renewcommand\arraystretch{1.05}
\scalebox{0.92}{
\begin{tabular}{l|c|cccc|>{\columncolor{gray!30}}c}
\hline
Method & Venue & A & C & P & R & Avg. \\
\hline
\multicolumn{7}{c}{\textit{ResNet-50}} \\
\hline
DANN~\cite{ganin2016domain}
& IJCAI'16
& 59.90
& 53.00
& 73.60
& 76.90
& 65.90 \\

CORAL~\cite{sun2016deep}
& ICCV'16
& 64.10
& 54.50
& 76.20
& 77.80
& 68.20 \\


GroupDRO~\cite{sagawa2019distributionally}
& ICLR'20
& 61.30
& 53.30
& 75.40
& 76.00
& 66.50 \\

RSC~\cite{huang2020self}
& ECCV'20
& 60.70
& 51.40
& 74.80
& 75.10
& 65.50 \\

VREx~\cite{krueger2021out}
& ICML'21
& 60.70
& 53.00
& 75.30
& 76.60
& 66.40 \\

Mixstyle~\cite{zhou2021domain}
& ICLR'21
& 51.10
& 53.20
& 68.20
& 69.20
& 60.40 \\

ERM~\cite{vapnik2013nature}
& ICLR'21
& 61.40
& 53.50
& 75.90
& 77.10
& 67.00 \\

SAM~\cite{foret2021sharpness}
& ICLR'21
& 62.20
& 55.90
& 77.00
& 78.80
& 68.50 \\

SagNet~\cite{nam2021reducing}
& CVPR'21
& 62.30
& 51.70
& 75.40
& 78.10
& 66.90 \\

Miro~\cite{cha2022domain}
& ECCV'22
& 67.50
& 54.60
& 78.00
& 81.60
& 70.50 \\

GSAM~\cite{zhuang2022surrogate}
& ICLR'22
& 64.90
& 55.20
& 77.80
& 79.20
& 69.30 \\

SAGM~\cite{wang2023sharpness}
& CVPR'23
& 65.40
& 57.00
& 78.00
& 80.00
& 70.10 \\

DomainDrop~\cite{guo2023domaindrop}
& ICCV'23
& -
& -
& -
& -
& 68.70 \\


XDomainMix~\cite{liu2024cross}
& IJCAI'24
& -
& -
& -
& -
& 68.10 \\

GMDG~\cite{tan2024rethinking}
& CVPR'24
& 68.90
& 56.20
& 79.90
& \textcolor{blue}{82.00}
& 70.70 \\


RES~\cite{huangrepresentation}
& ECCV'24
& -
& -
& -
& -
& 71.80 \\

SFT~\cite{li2025seeking}
& CVPR'25
& 65.80
& \textcolor{blue}{58.80}
& 78.30
& 80.60
& 70.90 \\

GGA~\cite{ballas2025gradient}
& CVPR'25
& -
& -
& -
& -
& 67.00 \\

SDK+SCA~\cite{wei2025indirect}
& IJCAI'25
& 70.20
& \textcolor{red}{59.50}
& 78.50
& 81.20
& \textcolor{blue}{72.30} \\

CBD-Gen~\cite{wang2025rethinking}
& NeurIPS'25
& 66.50
& 58.70
& 78.80
& 81.60
& 71.40 \\

GUIDE~\cite{thomas2025latent}
& ICCV'25
& -
& -
& -
& -
& 68.60 \\
\hline

PAPT~\cite{PAPT}
& CVPR'25
& \textcolor{blue}{71.12}
& 52.51
& \textcolor{blue}{80.72}
& 80.47
& 71.21 \\

PAPT++ (Ours)
& -
& \textcolor{red}{72.77}
& 53.26
& \textcolor{red}{80.84}
& \textcolor{red}{82.55}
& \textcolor{red}{72.36} \\
\hline
\end{tabular}}
\end{center}
\vspace{-2.0mm}
\end{table}

\begin{table}[!t]
\begin{center}
\caption{
Multi-source DG results (\%) on VLCS with the backbone of ResNet-50.
One domain is used as the target domain and the others are used as the source domains.
The best results are in \textcolor{red}{red} and the second highest results are in \textcolor{blue}{blue}.
}
\vspace{-2.0mm}
\label{tab:multi-source-vlcs}
\renewcommand\arraystretch{1.0}
\scalebox{0.90}{
\begin{tabular}{l|c|cccc|>{\columncolor{gray!30}}c}
\hline
Method & Venue & V & L & C & S & Avg. \\
\hline
\multicolumn{7}{c}{\textit{ResNet-50}} \\
\hline
DANN~\cite{ganin2016domain}
& IJCAI'16
& 51.10
& 65.10
& 99.00
& 73.10
& 78.60 \\

CORAL~\cite{sun2016deep}
& ICCV'16
& 77.50
& 66.10
& 98.30
& 73.40
& 78.80 \\

MLDG~\cite{li2018learning}
& AAAI'18
& 75.30
& 65.20
& 97.40
& 71.00
& 77.20 \\

GroupDRO~\cite{sagawa2019distributionally}
& ICLR'20
& 76.70
& 63.40
& 97.30
& 69.50
& 76.70 \\

RSC~\cite{huang2020self}
& ECCV'20
& 75.60
& 62.50
& 97.90
& 72.30
& 77.10 \\

VREx~\cite{krueger2021out}
& ICML'21
& 76.20
& 64.40
& 98.40
& 74.10
& 78.30 \\

Mixstyle~\cite{zhou2021domain}
& ICLR'21
& 75.70
& 64.50
& 98.60
& 72.60
& 77.90 \\

ERM~\cite{vapnik2013nature}
& ICLR'21
& 75.20
& 64.70
& 98.00
& 71.40
& 77.30 \\

SAM~\cite{foret2021sharpness}
& ICLR'21
& 79.80
& 65.00
& 99.10
& 73.70
& 79.40 \\

SagNet~\cite{nam2021reducing}
& CVPR'21
& 77.50
& 64.50
& 97.90
& 71.40
& 77.80 \\

Miro~\cite{cha2022domain}
& ECCV'22
& 77.80
& 64.70
& 98.30
& 75.30
& 79.00 \\

GSAM~\cite{zhuang2022surrogate}
& ICLR'22
& 78.50
& 64.90
& 98.70
& 74.30
& 79.10 \\

CoOp~\cite{zhou2022learning}
& IJCV'22
& \textcolor{blue}{82.53}
& 60.85
& \textcolor{red}{100.0}
& 76.14
& 79.88 \\

SAGM~\cite{wang2023sharpness}
& CVPR'23
& 80.70
& 65.20
& 99.00
& 75.10
& 80.00 \\

DomainDrop~\cite{guo2023domaindrop}
& ICCV'23
& -
& -
& -
& -
& 79.80 \\

XDomainMix~\cite{liu2024cross}
& IJCAI'24
& -
& -
& -
& -
& 76.30 \\

GMDG~\cite{tan2024rethinking}
& CVPR'24
& 79.30
& 65.90
& 98.30
& 73.40
& 79.20 \\


RES~\cite{huangrepresentation}
& ECCV'24
& -
& -
& -
& -
& 79.80 \\

SFT~\cite{li2025seeking}
& CVPR'25
& 78.70
& 66.20
& \textcolor{blue}{99.50}
& 74.80
& 79.80 \\

GGA~\cite{ballas2025gradient}
& CVPR'25
& 77.40
& 65.40
& 98.40
& 73.80
& 78.70 \\

SDK+SCA~\cite{wei2025indirect}
& IJCAI'25
& 79.40
& \textcolor{red}{69.60}
& 98.00
& 74.50
& 80.40 \\

CBD-Gen~\cite{wang2025rethinking}
& NeurIPS'25
& 80.50
& \textcolor{blue}{67.30}
& 99.10
& 75.10
& 80.50 \\

GUIDE~\cite{thomas2025latent}
& ICCV'25
& -
& -
& -
& -
& 78.50 \\
\hline

PAPT~\cite{PAPT}
& CVPR'25
& \textcolor{red}{84.50}
& 61.73
& \textcolor{red}{100.0}
& \textcolor{blue}{77.16}
& \textcolor{blue}{80.84} \\

PAPT++ (Ours)
& -
& 81.64
& 67.01
& \textcolor{red}{100.0}
& \textcolor{red}{77.42}
& \textcolor{red}{81.52} \\
\hline
\end{tabular}}
\end{center}
\vspace{-2.0mm}
\end{table}

\subsection{Results on Multi-Source Domain Generalization}

\noindent We further evaluate PAPT++ in the multi-source domain generalization setting, where multiple labeled source domains are available during training. We follow the standard leave-one-domain-out protocol: for each benchmark, one domain is held out as the unseen target domain, while the remaining domains are jointly used for training. 

We compare PAPT++ with representative multi-source DG methods spanning domain alignment and invariant representation learning~\cite{ganin2016domain,sun2016deep,krueger2021out,nam2021reducing}, data and feature augmentation~\cite{zhou2021domain}, robust learning and representation regularization~\cite{sagawa2019distributionally,huang2020self}, flatness-aware optimization~\cite{wang2023sharpness}, and recent multi-source DG approaches~\cite{guo2023domaindrop,tan2024rethinking,huangrepresentation}. Following the evaluation protocol of GMDG~\cite{tan2024rethinking}, we conduct experiments on TerraInc, OfficeHome, and VLCS using ResNet-50. The results are reported in Tabs.~\ref{tab:multi-source-officehome}--\ref{tab:multi-source-terrainc}.

PAPT++ achieves the highest average accuracy on all three benchmarks, reaching 55.18\%, 72.36\%, and 81.52\% on TerraInc, OfficeHome, and VLCS, respectively. Compared with PAPT, the corresponding improvements are 0.91, 1.15, and 0.68 pp. PAPT++ also outperforms SDK+SCA by 1.28 pp on TerraInc and CBD-Gen by 1.02 pp on VLCS, while slightly surpassing SDK+SCA on OfficeHome. At the domain level, particularly pronounced gains over PAPT are observed when L100 on TerraInc and LabelMe on VLCS are held out, with improvements of 5.47 and 5.28 pp, respectively. Moreover, PAPT++ consistently improves upon PAPT across all four target-domain settings on OfficeHome.

These results show that the benefits of PAPT++ extend beyond the single-source setting. Even when multiple source domains already provide diverse visual variations, synthesizing challenging yet class-consistent samples can expose complementary failure modes of the current classifier and further improve generalization to unseen domains.

\begin{table}[!t]
\begin{center}
\caption{
Multi-source DG results (\%) on TerraInc with the backbone of ResNet-50.
One domain is used as the target domain and the others are used as the source domains.
The best results are in \textcolor{red}{red} and the second highest results are in \textcolor{blue}{blue}.
}
\vspace{-2.0mm}
\label{tab:multi-source-terrainc}
\renewcommand\arraystretch{1.0}
\scalebox{0.9}{
\begin{tabular}{l|c|cccc|>{\columncolor{gray!30}}c}
\hline
Method & Venue & L100 & L38 & L43 & L46 & Avg. \\
\hline
\multicolumn{7}{c}{\textit{ResNet-50}} \\
\hline
DANN~\cite{ganin2016domain}
& IJCAI'16
& 51.10
& 40.60
& 57.40
& 37.70
& 46.70 \\

CORAL~\cite{sun2016deep}
& ICCV'16
& 51.60
& 42.20
& 57.00
& 39.80
& 47.70 \\


GroupDRO~\cite{sagawa2019distributionally}
& ICLR'20
& 41.20
& 38.60
& 56.70
& 36.40
& 43.20 \\

RSC~\cite{huang2020self}
& ECCV'20
& 50.20
& 39.20
& 56.30
& 40.80
& 46.60 \\

VREx~\cite{krueger2021out}
& ICML'21
& 48.20
& 41.70
& 56.80
& 38.70
& 46.40 \\

Mixstyle~\cite{zhou2021domain}
& ICLR'21
& 54.30
& 34.10
& 55.90
& 31.70
& 44.00 \\

ERM~\cite{vapnik2013nature}
& ICLR'21
& 49.80
& 42.10
& 56.90
& 35.70
& 46.10 \\

SAM~\cite{foret2021sharpness}
& ICLR'21
& 46.30
& 38.40
& 54.00
& 34.50
& 43.30 \\

SagNet~\cite{nam2021reducing}
& CVPR'21
& 53.00
& 43.00
& \textcolor{blue}{57.90}
& 40.40
& 48.60 \\

Miro~\cite{cha2022domain}
& ECCV'22
& 61.10
& 43.90
& 56.90
& 39.60
& 50.40 \\

GSAM~\cite{zhuang2022surrogate}
& ICLR'22
& 50.80
& 39.30
& \textcolor{red}{59.60}
& 38.20
& 47.00 \\

SAGM~\cite{wang2023sharpness}
& CVPR'23
& 54.80
& 41.40
& 57.70
& 41.30
& 48.80 \\

DomainDrop~\cite{guo2023domaindrop}
& ICCV'23
& -
& -
& -
& -
& 51.50 \\

XDomainMix~\cite{liu2024cross}
& IJCAI'24
& -
& -
& -
& -
& 48.20 \\

GMDG~\cite{tan2024rethinking}
& CVPR'24
& 59.80
& 45.30
& 57.10
& 38.20
& 50.10 \\

RES~\cite{huangrepresentation}
& ECCV'24
& -
& -
& -
& -
& 51.40 \\

SFT~\cite{li2025seeking}
& CVPR'25
& 57.50
& 44.60
& \textcolor{red}{59.60}
& 41.00
& 50.70 \\

GGA~\cite{ballas2025gradient}
& CVPR'25
& -
& -
& -
& -
& 48.50 \\

SDK+SCA~\cite{wei2025indirect}
& IJCAI'25
& 61.90
& \textcolor{red}{53.00}
& 57.30
& 43.30
& 53.90 \\

CBD-Gen~\cite{wang2025rethinking}
& NeurIPS'25
& 56.40
& 45.00
& \textcolor{red}{59.60}
& 41.00
& 50.50 \\

GUIDE~\cite{thomas2025latent}
& ICCV'25
& -
& -
& -
& -
& 51.30 \\
\hline

PAPT~\cite{PAPT}
& CVPR'25
& \textcolor{blue}{64.15}
& \textcolor{blue}{48.93}
& 55.27
& \textcolor{blue}{48.74}
& \textcolor{blue}{54.27} \\

PAPT++ (Ours)
& -
& \textcolor{red}{69.62}
& 47.87
& 52.66
& \textcolor{red}{50.57}
& \textcolor{red}{55.18} \\
\hline
\end{tabular}}
\end{center}
\vspace{-2.0mm}
\end{table}
\begin{table*}[t]
\centering
\caption{Ablation results (\%) on PACS and VLCS with the ResNet-18 backbone.}
\vspace{-2.0mm}
\label{tab:ablation}
\setlength{\tabcolsep}{5.8pt}
\renewcommand{\arraystretch}{0.9}
\begin{tabular}{c|l|ccccc|c|ccccc|c}
\toprule
\multirow{2}{*}{Idx}
& \multirow{2}{*}{Method}
& \multicolumn{6}{c|}{PACS}
& \multicolumn{6}{c}{VLCS} \\
\cmidrule(lr){3-8}
\cmidrule(lr){9-14}
& & A & C & P & S & Avg. & $\Delta$
  & V & L & C & S & Avg. & $\Delta$ \\
\midrule
1 & Source Only
& 58.60 & 66.40 & 34.00 & 27.50 & 46.60
& \textcolor{darkgreen}{\(\uparrow 26.83\)}
& 71.81 & 61.06 & 52.60 & 62.32 & 61.95
& \textcolor{darkgreen}{\(\uparrow 12.26\)} \\

2 & PAPT
& 69.84 & 78.57 & 62.12 & 80.34 & 72.72
& \textcolor{darkgreen}{\(\uparrow 0.71\)}
& 73.16 & 74.69 & 69.66 & 75.96 & 73.37
& \textcolor{darkgreen}{\(\uparrow 0.84\)} \\
\midrule

3 & w/o CSRL
& 72.42 & 79.39 & 60.40 & 70.63 & 70.71
& \textcolor{darkgreen}{\(\uparrow 2.72\)}
& 73.15 & 75.02 & 62.87 & 75.41 & 71.61
& \textcolor{darkgreen}{\(\uparrow 2.60\)} \\

4 & w/o CADS
& 67.44 & 74.17 & 56.30 & 72.04 & 67.49
& \textcolor{darkgreen}{\(\uparrow 5.94\)}
& 73.57 & 73.93 & 67.95 & 70.21 & 71.42
& \textcolor{darkgreen}{\(\uparrow 2.79\)} \\

5 & w/o Diversity ($\lambda_{\mathrm{div}}=0$)
& 70.23 & 76.83 & 49.60 & 76.01 & 68.17
& \textcolor{darkgreen}{\(\uparrow 5.26\)}
& 73.23 & 76.99 & 62.35 & 75.53 & 72.03
& \textcolor{darkgreen}{\(\uparrow 2.18\)} \\

6 & w/o Risk Guidance ($\lambda_{\mathrm{risk}}=0$)
& 71.73 & 78.92 & 61.53 & 78.50 & 72.67
& \textcolor{darkgreen}{\(\uparrow 0.76\)}
& 74.20 & 76.84 & 62.68 & 75.26 & 72.25
& \textcolor{darkgreen}{\(\uparrow 1.96\)} \\

7 & w/o Reference Denoising ($\mathcal{L}_{\mathrm{den}}$)
& 73.02 & 74.37 & 58.98 & 75.82 & 70.65
& \textcolor{darkgreen}{\(\uparrow 2.78\)}
& 73.32 & 75.69 & 60.64 & 75.21 & 71.71
& \textcolor{darkgreen}{\(\uparrow 2.50\)} \\
\midrule

8 & \textbf{PAPT++}
& \textbf{70.01} & \textbf{80.45} & \textbf{61.08} 
& \textbf{82.18} & \textbf{73.43} & -
& \textbf{73.07} & \textbf{77.22} & \textbf{71.42} 
& \textbf{75.14} & \textbf{74.21} & - \\
\bottomrule
\end{tabular}
\vspace{-3.0mm}
\end{table*}

\subsection{Ablation Study}

\noindent We conduct ablation studies on PACS and VLCS using ResNet-18 to study the contribution of each component in PAPT++. As reported in Tab.~\ref{tab:ablation}, rows 3 and 4 show the results after removing CSRL and CADS, respectively, while rows 5--7 study the main objectives within these two modules. Compared with Source Only, PAPT++ improves the average accuracy from 46.60\% to 73.43\% on PACS and from 61.95\% to 74.21\% on VLCS, corresponding to gains of 26.83 and 12.26 pp, respectively (rows 1 and 8).

Although PAPT and PAPT++ both use a pretrained diffusion model to expand the training distribution, their generation mechanisms are different. PAPT learns category and domain prompts to generate diverse domain styles, whereas PAPT++ learns image-level semantic references and uses feedback from the current classifier to search for high-risk samples. As shown in rows 2 and 8, PAPT++ improves upon PAPT by 0.71 pp on PACS and 0.84 pp on VLCS. These results show that the risk-aware generation strategy of PAPT++ provides more useful training variations than diversity-oriented prompt generation on both benchmarks.

\begin{figure}[!t]
\centering
\includegraphics[width=0.47\textwidth]{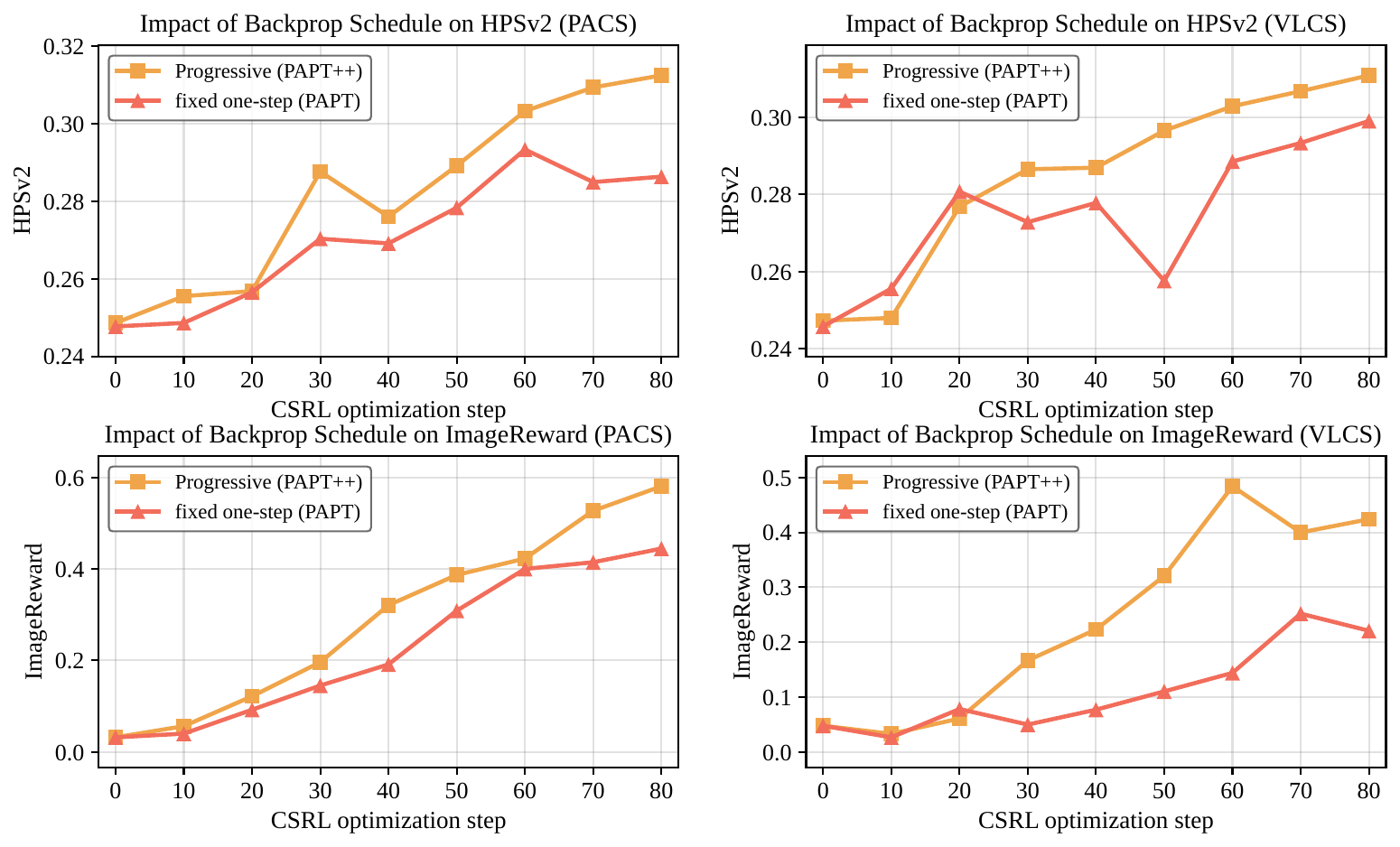
}
\vspace{-2.0mm}
\caption{Effect of the denoising backpropagation schedule in CSRL. Progressive (PAPT++) and fixed one-step optimization (PAPT) are compared on PACS and VLCS using HPSv2 and ImageReward.}
\label{fig:csrl_image_quality}
\vspace{-3.0mm}
\end{figure}

\vspace{2.0mm}

\noindent\textbf{Effectiveness of CSRL.}
We first evaluate the overall contribution of the CSRL module. Compared with the variant without CSRL, incorporating CSRL improves the average accuracy by 2.72 pp on PACS and 2.60 pp on VLCS (rows 3 and 8). CSRL learns image-level semantic references that guide subsequent high-risk synthesis. These references provide reliable class-level information, allowing CADS to explore new domain variations while maintaining class consistency. The improvements show that semantic references are important for generating challenging yet class-consistent samples.

We further study the diversity objective in CSRL by setting $\lambda_{\mathrm{div}}=0$. Compared with this setting, adding the diversity objective improves the  accuracy by 5.26 pp on PACS and 2.18 pp on VLCS (rows 5 and 8). The alignment objective encourages the learned references to match their class prompts, but it does not explicitly encourage different references from the same class to capture different visual appearances. The diversity objective encourages these references to represent complementary intra-class variations, providing a broader range of semantic references for subsequent high-risk synthesis. These results confirm the overall contribution of CSRL and the importance of the diversity objective within this module.

We also study the progressive denoising backpropagation schedule in CSRL. To isolate the effect of this schedule, we conduct the comparison using CSRL alone without CADS and keep the other settings unchanged. As shown in Fig.~\ref{fig:csrl_image_quality}, the progressive schedule achieves higher HPSv2 and ImageReward scores than the fixed one-step schedule during the later optimization steps on both PACS and VLCS. These results show that gradually increasing the number of denoising steps involved in backpropagation helps CSRL learn semantic references with better semantic alignment and visual quality.

\vspace{2.0mm}

\begin{figure}[!t]
\centering
\includegraphics[width=0.47\textwidth]{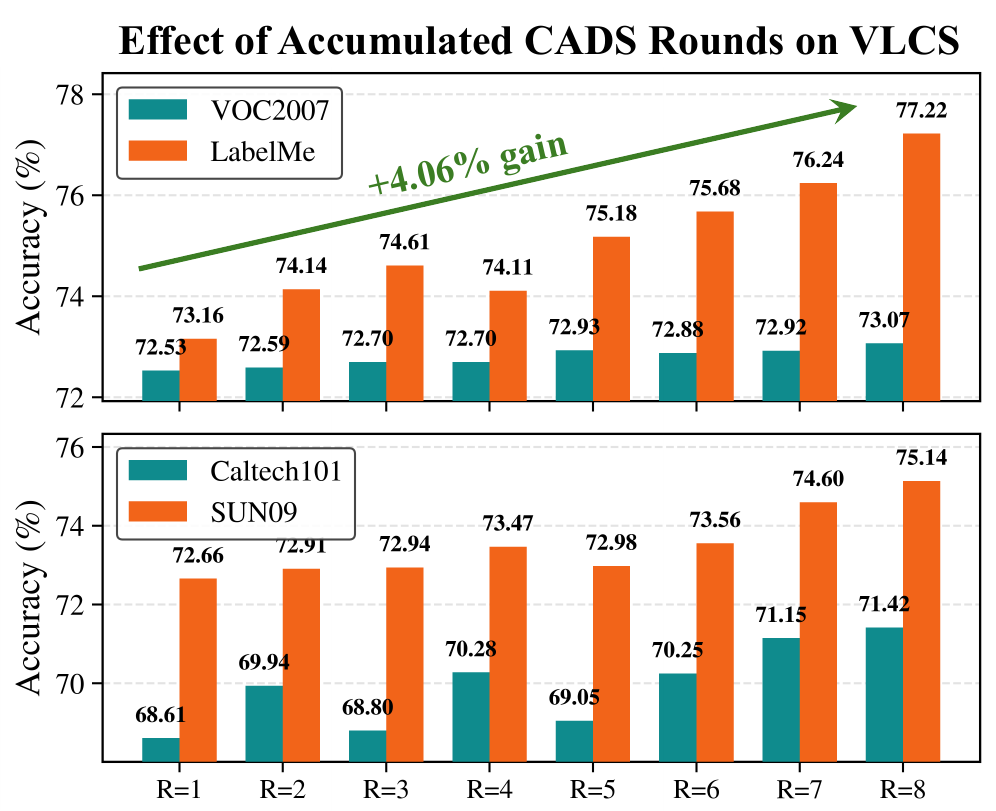}
\vspace{-3.0mm}
\caption{Effect of the number of accumulated CADS rounds $R$ on VLCS.}
\label{fig:cads_round_r}
\vspace{-3.0mm}
\end{figure}

\begin{table*}[!t]
\centering
\caption{Analysis of generated samples and downstream performance for different CADS variants on PACS.}
\vspace{-2.0mm}
\label{tab:mech-pacs-predicted}
\setlength{\tabcolsep}{5.8pt}
\renewcommand{\arraystretch}{0.9}
\begin{tabular}{c|l|cccccc}
\toprule
Idx
& Method
& Guide CE $\uparrow$
& Transfer CE $\uparrow$
& CLIP Cons. (\%) $\uparrow$
& ImageReward $\uparrow$
& LPIPS $\uparrow$
& PACS Avg. (\%) $\uparrow$ \\
\midrule

1 & w/o CADS
& 0.76
& 3.76
& 90.3
& 0.42
& 0.67
& 67.49 \\

2 & w/o Risk Guidance ($\lambda_{\mathrm{risk}}=0$)
& 0.78
& 6.05
& 93.1
& 0.78
& 0.59
& 72.67 \\

3 & w/o Reference Denoising ($\mathcal{L}_{\mathrm{den}}$)
& 0.93
& 7.30
& 61.8
& -1.47
& 0.60
& 70.65 \\

\midrule

4 & \textbf{PAPT++}
& \textbf{0.79}
& \textbf{6.70}
& \textbf{87.5}
& \textbf{-0.58}
& \textbf{0.62}
& \textbf{73.43} \\

\bottomrule
\end{tabular}
\vspace{-3.0mm}
\end{table*}
\begin{table*}[t]
\centering
\caption{Comparison of computational and storage overheads and multi-source DG performance (\%) among diffusion-based methods.
The best results are in \textcolor{red}{red} and the second highest results are in \textcolor{blue}{blue}.
}
\vspace{-2.0mm}
\label{tab:diffusion_dg_efficiency}
\setlength{\tabcolsep}{3.0pt}
\renewcommand{\arraystretch}{1.0}
\resizebox{\textwidth}{!}{
\begin{tabular}{l|l|l|c|cc|cc|ccc}
\toprule
\multirow{2}{*}{Method}
& \multirow{2}{*}{Venue}
& \multirow{2}{*}{Diffusion}
& \multirow{2}{*}{\shortstack{Diffusion\\Params}}
& \multicolumn{2}{c|}{GPU Usage (GB,BS=1)}
& \multirow{2}{*}{\shortstack{Trainable\\Params}}
& \multirow{2}{*}{\shortstack{Extra Image\\Storage}}
& \multirow{2}{*}{VLCS}
& \multirow{2}{*}{TerraInc}
& \multirow{2}{*}{OfficeHome} \\
\cline{5-6}
& & & & Avg. & Max. & & & & & \\
\midrule

CDGA~\cite{hemati2023cross}
& --
& SD-v1.4
& 1.07B
& 5.0
& 6.0
& 23.50M
& 2.25 MiB
& 78.90
& --
& 68.20 \\

DomainFusion~\cite{huang2024domainfusion}
& ECCV'24
& SD-v1.4
& 1.07B
& 4.7
& 5.5
& 23.50M
& 0.29 MiB
& 79.20
& 51.10
& \textcolor{red}{72.40} \\

Terra~\cite{zhuang2024time}
& NeurIPS'24
& SDXL
& 3.50B
& 20.0
& 24.0
& 70.50M
& 0.75 MiB
& 78.25
& --
& 69.63 \\

FDS-ERM~\cite{noori2025fds}
& WACV'25
& SD-v1.5
& 1.07B
& 13.5
& 16.0
& 883.5M
& 3.73 MiB
& 79.80
& --
& 71.10 \\

TRIDENT~\cite{choi2025trident}
& --
& SD-v2.1-unCLIP
& 1.93B
& 8.0
& 10.0
& 34.00M
& 3.38 MiB
& 77.80
& --
& 70.00 \\

PAPT~\cite{PAPT}
& CVPR'25
& SD-v1.4
& 1.07B
& 8.0
& 10.5
& 23.55M
& 0.75 MiB
& \textcolor{blue}{80.84}
& \textcolor{blue}{54.27}
& 71.21 \\

GUIDE~\cite{thomas2025latent}
& ICCV'25
& SD-v2.1-base
& 1.29B
& 4.8
& 5.8
& 23.65M
& 0
& 77.00
& 51.30
& 68.60 \\

\midrule

PAPT++ (Ours)
& --
& SD-v1.5
& 1.07B
& 11.7
& 17.7
& 23.86M
& 0.75 MiB
& \textcolor{red}{81.52}
& \textcolor{red}{55.18}
& \textcolor{blue}{72.36} \\

\bottomrule
\end{tabular}}
\end{table*}

\noindent\textbf{Effectiveness of CADS.}
We next evaluate the contribution of CADS module. Compared with the variant without CADS, incorporating CADS improves the accuracy by 5.94 pp on PACS and 2.79 pp on VLCS (rows 4 and 8). While CSRL provides class-level semantic references, it does not consider the weaknesses of the current classifier. CADS uses classifier feedback to search for high-risk variations around these references, allowing the synthesized samples to focus on cases that the classifier finds difficult. The improvements show that classifier-guided synthesis makes the learned semantic references more useful for improving generalization.

We further study the two objectives in CADS. To evaluate the classifier-risk objective, we set $\lambda_{\mathrm{risk}}=0$. Compared with this setting, adding the classifier-risk objective improves the average accuracy by 0.76 pp on PACS and 1.96 pp on VLCS (rows 6 and 8). We also evaluate the reference denoising loss by removing $\mathcal{L}_{\mathrm{den}}$. Adding this loss improves the average accuracy by 2.78 pp on PACS and 2.50 pp on VLCS (rows 7 and 8). The classifier-risk objective guides the synthesis process toward challenging variations that expose the weaknesses of the current classifier, while the reference denoising loss keeps the synthesized samples consistent with the class semantics of the learned references. These results show that the two objectives play complementary roles in balancing sample difficulty and class consistency, enabling CADS to generate challenging yet class-consistent samples for improving generalization to unseen domains.

We finally study the effect of accumulating synthesized samples across CADS rounds. The samples generated in each round are retained and used together with those from previous rounds for subsequent classifier training. As shown in Fig.~\ref{fig:cads_round_r}, the final-round accuracy is higher than the first-round accuracy on all four held-out domains of VLCS, with the largest improvement of 4.06 pp observed on LabelMe. As the classifier is updated, CADS uses its latest feedback to find new high-risk variations, while accumulating samples from different rounds provides a broader set of challenging training examples. These results show that iterative classifier-guided synthesis and cross-round sample accumulation jointly improve generalization to unseen domains.

\subsection{Analysis of the CADS Generation Mechanism}

\noindent The four settings in Tab.~\ref{tab:mech-pacs-predicted} correspond to \emph{w/o CADS}, \emph{w/o Risk Guidance}, \emph{w/o Reference Denoising}, and the full PAPT++ in Tab.~\ref{tab:ablation}. For each generated image, \emph{Guide CE} is the cross-entropy loss with respect to its conditioning label, computed using the source-trained ResNet-18 that guides CADS. \emph{Transfer CE} is computed in the same manner using an independently ResNet-50 that is not involved in generation, thereby measuring whether sample difficulty transfers to another classifier. \emph{CLIP Cons.} is the percentage of generated images whose zero-shot CLIP prediction matches the conditioning class. \emph{ImageReward} is the average image--text alignment score between each generated image and its conditioning prompt. \emph{LPIPS}~\cite{zhang2018perceptual} is computed between images conditioned on the same class and then averaged across classes to measure intra-class perceptual diversity. All metrics are averaged over the generated samples. For \emph{w/o CADS}, they are computed directly on the semantic references learned by CSRL.

The results show that Risk Guidance and Reference Denoising play complementary roles. Adding Risk Guidance increases Transfer CE from 6.05 to 6.70 and improves accuracy from 72.67\% to 73.43\%, showing that classifier feedback helps generate more challenging and useful samples. Without Reference Denoising, Guide CE and Transfer CE reach 0.93 and 7.30, but the loss of semantic consistency limits the accuracy to 70.65\%. Reference Denoising improves class consistency and ImageReward to 87.5\% and -0.58 while maintaining a high Transfer CE of 6.70, enabling PAPT++ to achieve the best accuracy of 73.43\%. Moreover, although the CSRL references in \emph{w/o CADS} have the highest LPIPS, their lower Transfer CE and accuracy show that perceptual diversity alone does not necessarily produce useful domain variations. Overall, CADS balances sample difficulty and class consistency to generate informative samples for domain generalization.

\begin{figure*}[!t]
\centering
\includegraphics[width=0.98\textwidth]{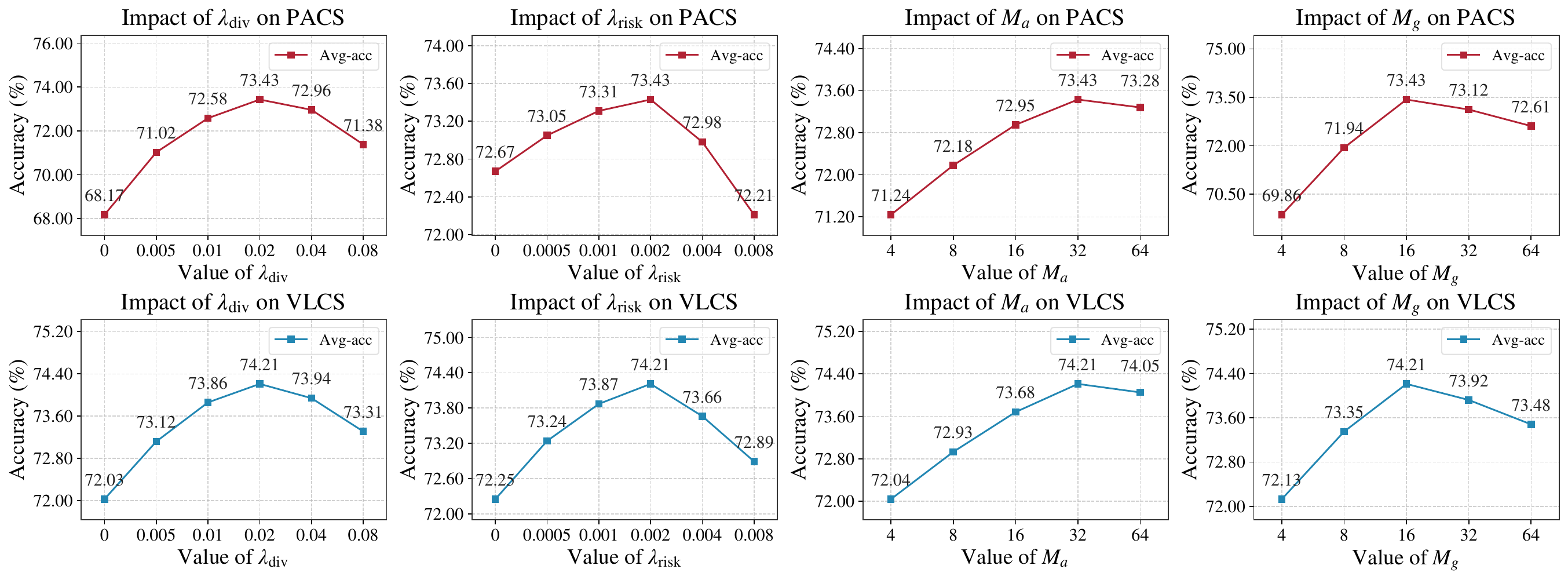}
\vspace{-3.0mm}
\caption{Hyper-parameter analysis of $\lambda_{\mathrm{div}}$,
$\lambda_{\mathrm{risk}}$, $M_a$, and $M_g$ on PACS and VLCS.}
\label{fig:hyperparameter}
\vspace{-3.0mm}
\end{figure*}

\begin{table}[!t]
\centering
\caption{Impact of CFG/CFG++ on VLCS (\%) with ResNet-18 under a high risk-guidance weight ($\lambda_{\mathrm{risk}}=0.5$). The best results are in \textcolor{red}{red}, and the second-highest results are in \textcolor{blue}{blue}.}
\label{tab:vlcs_cfg_ablation}
\setlength{\tabcolsep}{5.8pt}
\renewcommand{\arraystretch}{0.95}
\begin{tabular}{c|l|cccc|c}
\toprule
Idx & Method & V & L & C & S & Avg. \\
\midrule
1 & CFG ($g=1$)
& 71.34
& 69.78
& 62.98
& \textcolor{red}{73.79}
& 69.47 \\
2 & CFG ($g=3$)
& \textcolor{red}{72.50}
& 70.77
& \textcolor{blue}{65.66}
& 72.85
& \textcolor{blue}{70.44} \\
3 & CFG ($g=5$)
& 71.07
& 69.49
& 63.90
& 71.45
& 68.98 \\
\midrule
4 & CFG++ ($\lambda=0.4$)
& 70.77
& 70.93
& 64.36
& 72.43
& 69.62 \\
5 & CFG++ ($\lambda=0.6$)
& \textcolor{blue}{72.24}
& \textcolor{red}{73.43}
& \textcolor{red}{66.50}
& 72.87
& \textcolor{red}{71.26} \\
6 & CFG++ ($\lambda=0.8$)
& 71.91
& \textcolor{blue}{71.32}
& 62.78
& \textcolor{blue}{73.01}
& 69.76 \\
\bottomrule
\end{tabular}
\end{table}
\subsection{Further Analysis}

\noindent\textbf{Computational cost.}
Tab.~\ref{tab:diffusion_dg_efficiency} compares the efficiency and performance of diffusion-based DG methods. PAPT++ achieves the highest accuracies on VLCS (81.52\%) and TerraInc (55.18\%), and the second-highest accuracy on OfficeHome (72.36\%). Its average and maximum GPU usage are 11.7 GB and 17.7 GB, respectively, which are higher than those of PAPT but lower than those of Terra. Despite using a 1.07B-parameter diffusion pipeline, PAPT++ updates only 23.86M parameters (2.23\%). Compared with PAPT, it adds only 0.31M trainable parameters and retains the same storage overhead of 0.75 MiB per source sample, while improving accuracy by 0.68, 0.91, and 1.15 pp on the three benchmarks, respectively. Moreover, the diffusion pipeline is used only during training and introduces no additional inference cost.

\vspace{2.0mm}

\noindent\textbf{Effect of CFG/CFG++ Guidance.}
We compare different CFG scales and CFG++ coefficients on VLCS in Tab.~\ref{tab:vlcs_cfg_ablation}. For this controlled analysis, we set $\lambda_{\mathrm{risk}}=0.5$ to construct a more challenging synthesis setting. Under this stronger risk signal, the comparison focuses on how different guidance configurations balance semantic consistency and sample difficulty. Among the evaluated configurations, CFG++ with $\lambda=0.6$ achieves the highest average accuracy of 71.26\%, outperforming the best CFG configuration ($g=3$) by 0.82 pp. For both guidance methods, moderate guidance performs better than smaller or larger values, indicating a more effective balance between class consistency and challenging within-class variations.

\vspace{2.0mm}

\noindent \textbf{Hyperparameter Analysis.}
We analyze the influence of $\lambda_{\mathrm{div}}$, $\lambda_{\mathrm{risk}}$, $M_a$, and $M_g$ on PACS and VLCS.
As shown in Fig.~\ref{fig:hyperparameter}, increasing
$\lambda_{\mathrm{div}}$ and $\lambda_{\mathrm{risk}}$ initially improves accuracy, whereas overly large weights lead to performance degradation. For $M_a$ and $M_g$, increasing the number of samples is beneficial within a moderate range.

\section{Conclusion}
\noindent In this work, we presented PAPT++, a diffusion-based framework for domain
generalization that improves the semantic reliability and generalization value
of synthesized samples. PAPT++ first employs CSRL to construct semantically
meaningful and diverse references from source-domain data. Based on these
references, CADS explores challenging domain variations through classifier-guided
risk maximization, while reference denoising preserves category-level semantics
during synthesis. 
Extensive experiments demonstrate that PAPT++ consistently improves upon
PAPT and achieves strong performance on standard DG benchmarks. These
results highlight the importance of synthesizing challenging yet semantically
reliable samples for generalization to unseen domains.
Future work will explore more efficient high-risk synthesis strategies and extend PAPT++ to broader generalization settings with complex and evolving domain shifts.

\bibliographystyle{IEEEtran}
\bibliography{egbib}

@String(CVPR= {IEEE Conf. Comput. Vis. Pattern Recog.})

@String(ICCV= {Int. Conf. Comput. Vis.})

@String(ECCV= {Eur. Conf. Comput. Vis.})

@String(AAAI = {AAAI})

@String(CVPR  = {CVPR})

@String(ICCV  = {ICCV})

@String(ECCV  = {ECCV})

@inproceedings{cha2022domain,
      title        = {Domain Generalization by Mutual-Information Regularization with Pre-trained Models}, 
      author       = {Junbum Cha and
                  Kyungjae Lee and
                  Sungrae Park and
                  Sanghyuk Chun},
      year={2022},
      booktitle    = ECCV,
      pages        = {440--457}
}

@inproceedings{volpi2019addressing,
  title={Addressing model vulnerability to distributional shifts over image transformation sets},
  author={Volpi, Riccardo and Murino, Vittorio},
  booktitle={Proceedings of the IEEE/CVF International Conference on Computer Vision},
  pages={7980--7989},
  year={2019}
}

@article{volpi2018generalizing,
  title={Generalizing to unseen domains via adversarial data augmentation},
  author={Volpi, Riccardo and Namkoong, Hongseok and Sener, Ozan and Duchi, John C and Murino, Vittorio and Savarese, Silvio},
  journal={Advances in neural information processing systems},
  volume={31},
  year={2018}
}

@article{zhao2020maximum,
  title={Maximum-entropy adversarial data augmentation for improved generalization and robustness},
  author={Zhao, Long and Liu, Ting and Peng, Xi and Metaxas, Dimitris},
  journal={Advances in Neural Information Processing Systems},
  volume={33},
  pages={14435--14447},
  year={2020}
}

@inproceedings{wang2021learning,
  title={Learning to diversify for single domain generalization},
  author={Wang, Zijian and Luo, Yadan and Qiu, Ruihong and Huang, Zi and Baktashmotlagh, Mahsa},
  booktitle={Proceedings of the IEEE/CVF International Conference on Computer Vision},
  pages={834--843},
  year={2021}
}

@article{zhou2021domain,
  title={Domain generalization with mixstyle},
  author={Zhou, Kaiyang and Yang, Yongxin and Qiao, Yu and Xiang, Tao},
  journal={arXiv preprint arXiv:2104.02008},
  year={2021}
}

@inproceedings{fan2021adversarially,
  title={Adversarially adaptive normalization for single domain generalization},
  author={Fan, Xinjie and Wang, Qifei and Ke, Junjie and Yang, Feng and Gong, Boqing and Zhou, Mingyuan},
  booktitle={Proceedings of the IEEE/CVF conference on Computer Vision and Pattern Recognition},
  pages={8208--8217},
  year={2021}
}

@inproceedings{ramesh2021zero,
  title={Zero-shot text-to-image generation},
  author={Ramesh, Aditya and Pavlov, Mikhail and Goh, Gabriel and Gray, Scott and Voss, Chelsea and Radford, Alec and Chen, Mark and Sutskever, Ilya},
  booktitle={International conference on machine learning},
  pages={8821--8831},
  year={2021},
  organization={Pmlr}
}

@inproceedings{rombach2022high,
  title={High-resolution image synthesis with latent diffusion models},
  author={Rombach, Robin and Blattmann, Andreas and Lorenz, Dominik and Esser, Patrick and Ommer, Bj{\"o}rn},
  booktitle={Proceedings of the IEEE/CVF conference on computer vision and pattern recognition},
  pages={10684--10695},
  year={2022}
}

@article{xu2024imagereward,
  title={Imagereward: Learning and evaluating human preferences for text-to-image generation},
  author={Xu, Jiazheng and Liu, Xiao and Wu, Yuchen and Tong, Yuxuan and Li, Qinkai and Ding, Ming and Tang, Jie and Dong, Yuxiao},
  journal={Advances in Neural Information Processing Systems},
  volume={36},
  year={2024}
}

@article{sagawa2019distributionally,
  title={Distributionally robust neural networks for group shifts: On the importance of regularization for worst-case generalization},
  author={Sagawa, Shiori and Koh, Pang Wei and Hashimoto, Tatsunori B and Liang, Percy},
  journal={arXiv preprint arXiv:1911.08731},
  year={2019}
}

@article{sinha2017certifying,
  title={Certifying some distributional robustness with principled adversarial training},
  author={Sinha, Aman and Namkoong, Hongseok and Volpi, Riccardo and Duchi, John},
  journal={arXiv preprint arXiv:1710.10571},
  year={2017}
}

@article{namkoong2016stochastic,
  title={Stochastic gradient methods for distributionally robust optimization with f-divergences},
  author={Namkoong, Hongseok and Duchi, John C},
  journal={Advances in neural information processing systems},
  volume={29},
  year={2016}
}

@article{staib2019distributionally,
  title={Distributionally robust optimization and generalization in kernel methods},
  author={Staib, Matthew and Jegelka, Stefanie},
  journal={Advances in Neural Information Processing Systems},
  volume={32},
  year={2019}
}

@article{liu2021towards,
  title={Towards out-of-distribution generalization: A survey},
  author={Liu, Jiashuo and Shen, Zheyan and He, Yue and Zhang, Xingxuan and Xu, Renzhe and Yu, Han and Cui, Peng},
  journal={arXiv preprint arXiv:2108.13624},
  year={2021}
}

@inproceedings{hu2018does,
  title={Does distributionally robust supervised learning give robust classifiers?},
  author={Hu, Weihua and Niu, Gang and Sato, Issei and Sugiyama, Masashi},
  booktitle={International Conference on Machine Learning},
  pages={2029--2037},
  year={2018},
  organization={PMLR}
}

@article{frogner2019incorporating,
  title={Incorporating unlabeled data into distributionally robust learning},
  author={Frogner, Charlie and Claici, Sebastian and Chien, Edward and Solomon, Justin},
  journal={arXiv preprint arXiv:1912.07729},
  year={2019}
}

@article{liu2022distributionally,
  title={Distributionally robust learning with stable adversarial training},
  author={Liu, Jiashuo and Shen, Zheyan and Cui, Peng and Zhou, Linjun and Kuang, Kun and Li, Bo},
  journal={IEEE Transactions on Knowledge and Data Engineering},
  volume={35},
  number={11},
  pages={11288--11300},
  year={2022},
  publisher={IEEE}
}

@article{qiao2023topology,
  title={Topology-aware robust optimization for out-of-distribution generalization},
  author={Qiao, Fengchun and Peng, Xi},
  journal={arXiv preprint arXiv:2307.13943},
  year={2023}
}

@inproceedings{chung2025cfg++,
  title={Cfg++: Manifold-constrained classifier free guidance for diffusion models},
  author={Chung, Hyungjin and Kim, Jeongsol and Park, Geon Yeong and Nam, Hyelin and Ye, Jong Chul},
  booktitle={International Conference on Learning Representations},
  volume={2025},
  pages={30824--30850},
  year={2025}
}

@article{gulrajani2020search,
      title={In search of lost domain generalization},
      author={Gulrajani, Ishaan and Lopez-Paz, David},
      journal={arXiv preprint arXiv:2007.01434},
      year={2020}
}

@inproceedings{li2017deeper,
      title={Deeper, broader and artier domain generalization},
      author={Li, Da and Yang, Yongxin and Song, Yi-Zhe and Hospedales, Timothy M},
      booktitle=ICCV,
      pages={5542--5550},
      year={2017}
}

@inproceedings{venkateswara2017deep,
      title={Deep hashing network for unsupervised domain adaptation},
      author={Venkateswara, Hemanth and Eusebio, Jose and Chakraborty, Shayok and Panchanathan, Sethuraman},
      booktitle=CVPR,
      pages={5018--5027},
      year={2017}
}

@inproceedings{shinunknown,
  title     = {Unknown Domain Inconsistency Minimization for Domain Generalization},
  author    = {Shin, Seungjae and Bae, HeeSun and Na, Byeonghu and Kim, Yoon-Yeong and Moon, Il-Chul},
  booktitle = {International Conference on Learning Representations},
  year      = {2024}
}

@inproceedings{chen2023improved,
  title={Improved test-time adaptation for domain generalization},
  author={Chen, Liang and Zhang, Yong and Song, Yibing and Shan, Ying and Liu, Lingqiao},
  booktitle={Proceedings of the IEEE/CVF Conference on Computer Vision and Pattern Recognition},
  pages={24172--24182},
  year={2023}
}

@article{yan2020improve,
  title={Improve unsupervised domain adaptation with mixup training},
  author={Yan, Shen and Song, Huan and Li, Nanxiang and Zou, Lincan and Ren, Liu},
  journal={arXiv preprint arXiv:2001.00677},
  year={2020}
}

@inproceedings{rame2022fishr,
  title={Fishr: Invariant gradient variances for out-of-distribution generalization},
  author={Rame, Alexandre and Dancette, Corentin and Cord, Matthieu},
  booktitle={International Conference on Machine Learning},
  pages={18347--18377},
  year={2022},
  organization={PMLR}
}

@inproceedings{chen2023domain,
  title={Domain generalization via rationale invariance},
  author={Chen, Liang and Zhang, Yong and Song, Yibing and Van Den Hengel, Anton and Liu, Lingqiao},
  booktitle={Proceedings of the IEEE/CVF International Conference on Computer Vision},
  pages={1751--1760},
  year={2023}
}

@article{hendrycks2019benchmarking,
  title={Benchmarking neural network robustness to common corruptions and perturbations},
  author={Hendrycks, Dan and Dietterich, Thomas},
  journal={arXiv preprint arXiv:1903.12261},
  year={2019}
}

@techreport{krizhevsky2009learning,
  title       = {Learning Multiple Layers of Features from Tiny Images},
  author      = {Krizhevsky, Alex},
  institution = {University of Toronto},
  year        = {2009}
}

@article{hemati2023cross,
  title={Cross Domain Generative Augmentation: Domain Generalization with Latent Diffusion Models},
  author={Hemati, Sobhan and Beitollahi, Mahdi and Estiri, Amir Hossein and Omari, Bassel Al and Chen, Xi and Zhang, Guojun},
  journal={arXiv preprint arXiv:2312.05387},
  year={2023}
}

@inproceedings{guo2023domaindrop,
  title={Domaindrop: Suppressing domain-sensitive channels for domain generalization},
  author={Guo, Jintao and Qi, Lei and Shi, Yinghuan},
  booktitle={Proceedings of the IEEE/CVF international conference on computer vision},
  pages={19114--19124},
  year={2023}
}

@inproceedings{tan2024rethinking,
  title={Rethinking Multi-domain Generalization with A General Learning Objective},
  author={Tan, Zhaorui and Yang, Xi and Huang, Kaizhu},
  booktitle={Proceedings of the IEEE/CVF Conference on Computer Vision and Pattern Recognition},
  pages={23512--23522},
  year={2024}
}

@inproceedings{choi2023progressive,
  title={Progressive random convolutions for single domain generalization},
  author={Choi, Seokeon and Das, Debasmit and Choi, Sungha and Yang, Seunghan and Park, Hyunsin and Yun, Sungrack},
  booktitle={Proceedings of the IEEE/CVF Conference on Computer Vision and Pattern Recognition},
  pages={10312--10322},
  year={2023}
}

@article{hendrycks2019augmix,
  title={Augmix: A simple data processing method to improve robustness and uncertainty},
  author={Hendrycks, Dan and Mu, Norman and Cubuk, Ekin D and Zoph, Barret and Gilmer, Justin and Lakshminarayanan, Balaji},
  journal={arXiv preprint arXiv:1912.02781},
  year={2019}
}

@book{vapnik2013nature,
  title={The nature of statistical learning theory},
  author={Vapnik, Vladimir},
  year={2013},
  publisher={Springer science \& business media}
}

@inproceedings{nuriel2021permuted,
  title={Permuted adain: Reducing the bias towards global statistics in image classification},
  author={Nuriel, Oren and Benaim, Sagie and Wolf, Lior},
  booktitle={Proceedings of the IEEE/CVF conference on computer vision and pattern recognition},
  pages={9482--9491},
  year={2021}
}

@inproceedings{zhang2022exact,
  title={Exact feature distribution matching for arbitrary style transfer and domain generalization},
  author={Zhang, Yabin and Li, Minghan and Li, Ruihuang and Jia, Kui and Zhang, Lei},
  booktitle={Proceedings of the IEEE/CVF conference on computer vision and pattern recognition},
  pages={8035--8045},
  year={2022}
}

@article{li2022uncertainty,
  title={Uncertainty modeling for out-of-distribution generalization},
  author={Li, Xiaotong and Dai, Yongxing and Ge, Yixiao and Liu, Jun and Shan, Ying and Duan, Ling-Yu},
  journal={arXiv preprint arXiv:2202.03958},
  year={2022}
}

@inproceedings{cugu2022attention,
  title={Attention consistency on visual corruptions for single-source domain generalization},
  author={Cugu, Ilke and Mancini, Massimiliano and Chen, Yanbei and Akata, Zeynep},
  booktitle={Proceedings of the IEEE/CVF Conference on Computer Vision and Pattern Recognition},
  pages={4165--4174},
  year={2022}
}

@inproceedings{qu2023modality,
  title={Modality-agnostic debiasing for single domain generalization},
  author={Qu, Sanqing and Pan, Yingwei and Chen, Guang and Yao, Ting and Jiang, Changjun and Mei, Tao},
  booktitle={Proceedings of the IEEE/CVF Conference on Computer Vision and Pattern Recognition},
  pages={24142--24151},
  year={2023}
}

@article{ganin2016domain,
  title={Domain-adversarial training of neural networks},
  author={Ganin, Yaroslav and Ustinova, Evgeniya and Ajakan, Hana and Germain, Pascal and Larochelle, Hugo and Laviolette, Fran{\c{c}}ois and March, Mario and Lempitsky, Victor},
  journal={Journal of machine learning research},
  volume={17},
  number={59},
  pages={1--35},
  year={2016}
}

@inproceedings{sun2016deep,
  title={Deep coral: Correlation alignment for deep domain adaptation},
  author={Sun, Baochen and Saenko, Kate},
  booktitle={Computer Vision--ECCV 2016 Workshops: Amsterdam, The Netherlands, October 8-10 and 15-16, 2016, Proceedings, Part III 14},
  pages={443--450},
  year={2016},
  organization={Springer}
}

@inproceedings{li2018domain,
  title={Domain generalization via conditional invariant representations},
  author={Li, Ya and Gong, Mingming and Tian, Xinmei and Liu, Tongliang and Tao, Dacheng},
  booktitle={Proceedings of the AAAI conference on artificial intelligence},
  volume={32},
  number={1},
  year={2018}
}

@article{zhang2021adaptive,
  title={Adaptive risk minimization: Learning to adapt to domain shift},
  author={Zhang, Marvin and Marklund, Henrik and Dhawan, Nikita and Gupta, Abhishek and Levine, Sergey and Finn, Chelsea},
  journal={Advances in Neural Information Processing Systems},
  volume={34},
  pages={23664--23678},
  year={2021}
}

@inproceedings{krueger2021out,
  title={Out-of-distribution generalization via risk extrapolation (rex)},
  author={Krueger, David and Caballero, Ethan and Jacobsen, Joern-Henrik and Zhang, Amy and Binas, Jonathan and Zhang, Dinghuai and Le Priol, Remi and Courville, Aaron},
  booktitle={International conference on machine learning},
  pages={5815--5826},
  year={2021},
  organization={PMLR}
}

@inproceedings{nam2021reducing,
  title={Reducing domain gap by reducing style bias},
  author={Nam, Hyeonseob and Lee, HyunJae and Park, Jongchan and Yoon, Wonjun and Yoo, Donggeun},
  booktitle={Proceedings of the IEEE/CVF Conference on Computer Vision and Pattern Recognition},
  pages={8690--8699},
  year={2021}
}

@inproceedings{wang2023sharpness,
  title={Sharpness-aware gradient matching for domain generalization},
  author={Wang, Pengfei and Zhang, Zhaoxiang and Lei, Zhen and Zhang, Lei},
  booktitle={Proceedings of the IEEE/CVF Conference on Computer Vision and Pattern Recognition},
  pages={3769--3778},
  year={2023}
}

@inproceedings{PAPT,
  title={Adversarial domain prompt tuning and generation for single domain generalization},
  author={Xu, Zhipeng and Cheng, De and Jiang, Xinyang and Wang, Nannan and Li, Dongsheng and Gao, Xinbo},
  booktitle={Proceedings of the IEEE/CVF Conference on Computer Vision and Pattern Recognition},
  pages={18584--18595},
  year={2025}
}

@inproceedings{huang2020self,
  title={Self-challenging improves cross-domain generalization},
  author={Huang, Zeyi and Wang, Haohan and Xing, Eric P and Huang, Dong},
  booktitle={Computer vision--ECCV 2020: 16th European conference, Glasgow, UK, August 23--28, 2020, proceedings, part II 16},
  pages={124--140},
  year={2020},
  organization={Springer}
}

@inproceedings{chen2023meta,
  title={Meta-causal learning for single domain generalization},
  author={Chen, Jin and Gao, Zhi and Wu, Xinxiao and Luo, Jiebo},
  booktitle={Proceedings of the IEEE/CVF Conference on Computer Vision and Pattern Recognition},
  pages={7683--7692},
  year={2023}
}

@inproceedings{li2024prompt,
  title={Prompt-Driven Dynamic Object-Centric Learning for Single Domain Generalization},
  author={Li, Deng and Wu, Aming and Wang, Yaowei and Han, Yahong},
  booktitle={Proceedings of the IEEE/CVF Conference on Computer Vision and Pattern Recognition},
  pages={17606--17615},
  year={2024}
}

@inproceedings{li2024friendly,
  title={Friendly sharpness-aware minimization},
  author={Li, Tao and Zhou, Pan and He, Zhengbao and Cheng, Xinwen and Huang, Xiaolin},
  booktitle={Proceedings of the IEEE/CVF conference on computer vision and pattern recognition},
  pages={5631--5640},
  year={2024}
}

@article{liu2024generalizable,
  title={Generalizable prompt learning via gradient constrained sharpness-aware minimization},
  author={Liu, Liangchen and Wang, Nannan and Zhou, Dawei and Liu, Decheng and Yang, Xi and Gao, Xinbo and Liu, Tongliang},
  journal={IEEE Transactions on Multimedia},
  volume={27},
  pages={1100--1113},
  year={2024},
  publisher={IEEE}
}

@inproceedings{efthymiadis2025crafting,
  title={Crafting distribution shifts for validation and training in single source domain generalization},
  author={Efthymiadis, Nikos and Tolias, Giorgos and Chum, Ond{\v{r}}ej},
  booktitle={2025 IEEE/CVF Winter Conference on Applications of Computer Vision (WACV)},
  pages={1883--1892},
  year={2025},
  organization={IEEE}
}

@inproceedings{lyu2025sse,
  title={SSE-SAM: balancing head and tail classes gradually through stage-wise SAM},
  author={Lyu, Xingyu and Xu, Qianqian and Yang, Zhiyong and Lyu, Shaojie and Huang, Qingming},
  booktitle={Proceedings of the AAAI Conference on Artificial Intelligence},
  volume={39},
  number={18},
  pages={19278--19286},
  year={2025}
}

@inproceedings{zhou2025sharpness,
  title={Sharpness-aware minimization efficiently selects flatter minima late in training},
  author={Zhou, Zhanpeng and Wang, Mingze and Mao, Yuchen and Li, Bingrui and Yan, Junchi},
  booktitle={International Conference on Learning Representations},
  volume={2025},
  pages={20949--20980},
  year={2025}
}

@inproceedings{xu2025physaug,
  title={Physaug: A physical-guided and frequency-based data augmentation for single-domain generalized object detection},
  author={Xu, Xiaoran and Yang, Jiangang and Shi, Wenhui and Ding, Siyuan and Luo, Luqing and Liu, Jian},
  booktitle={Proceedings of the AAAI Conference on Artificial Intelligence},
  volume={39},
  number={20},
  pages={21815--21823},
  year={2025}
}

@article{PADG,
  title={Prompt Disentanglement via Language Guidance and Representation Alignment for Domain Generalization},
  author={Cheng, De and Xu, Zhipeng and Jiang, Xinyang and Li, Dongsheng and Wang, Nannan and Gao, Xinbo},
  journal={IEEE Transactions on Pattern Analysis and Machine Intelligence},
  year={2026},
  publisher={IEEE}
}

@inproceedings{RD-MLDG,
  title     = {Reasoning-Driven Multimodal LLM for Domain Generalization},
  author    = {Xu, Zhipeng and Wang, Zilong and Jiang, Xinyang and Li, Dongsheng and Cheng, De and Wang, Nannan},
  booktitle = {The Fourteenth International Conference on Learning Representations},
  year      = {2026},
  url       = {https://openreview.net/forum?id=psJiUopUt7}
}

@inproceedings{StPR,
  title     = {StPR: Spatiotemporal Preservation and Routing for Exemplar-Free Video Class-Incremental Learning},
  author    = {Wang, Huaijie and Cheng, De and Li, Guozhang and Xu, Zhipeng and He, Lingfeng and Li, Jie and Wang, Nannan and Gao, Xinbo},
  booktitle = {Proceedings of the Fourteenth International Conference on Learning Representations},
  year      = {2026},
  url       = {https://openreview.net/forum?id=VAn2YVMuZC}
}

@inproceedings{IKI,
  title     = {Interference-Isolated Elastic Weight Consolidation and Knowledge Calibration for Incremental Object Detection},
  author    = {Cheng, De and Zeng, Mingyue and Xu, Zhipeng and Xu, Di and Wang, Nannan and Gao, Xinbo},
  booktitle = {Proceedings of the Fourteenth International Conference on Learning Representations},
  year      = {2026},
  url       = {https://openreview.net/forum?id=VrXdmCjni4}
}

@inproceedings{DPR,
  title={Disentangled prompt representation for domain generalization},
  author={Cheng, De and Xu, Zhipeng and Jiang, Xinyang and Wang, Nannan and Li, Dongsheng and Gao, Xinbo},
  booktitle={Proceedings of the IEEE/CVF Conference on Computer Vision and Pattern Recognition},
  pages={23595--23604},
  year={2024}
}

@inproceedings{ronneberger2015u,
  title={U-net: Convolutional networks for biomedical image segmentation},
  author={Ronneberger, Olaf and Fischer, Philipp and Brox, Thomas},
  booktitle={Medical image computing and computer-assisted intervention--MICCAI 2015: 18th international conference, Munich, Germany, October 5-9, 2015, proceedings, part III 18},
  pages={234--241},
  year={2015},
  organization={Springer}
}

@inproceedings{radford2021learning,
  title={Learning transferable visual models from natural language supervision},
  author={Radford, Alec and Kim, Jong Wook and Hallacy, Chris and Ramesh, Aditya and Goh, Gabriel and Agarwal, Sandhini and Sastry, Girish and Askell, Amanda and Mishkin, Pamela and Clark, Jack and others},
  booktitle={International conference on machine learning},
  pages={8748--8763},
  year={2021},
  organization={PMLR}
}

@article{kingma2013auto,
  title={Auto-encoding variational bayes},
  author={Kingma, Diederik P and Welling, Max},
  journal={arXiv preprint arXiv:1312.6114},
  year={2013}
}

@article{liu2024stydesty,
  title={Stydesty: Min-max stylization and destylization for single domain generalization},
  author={Liu, Songhua and Jin, Xin and Yang, Xingyi and Ye, Jingwen and Wang, Xinchao},
  journal={arXiv preprint arXiv:2406.00275},
  year={2024}
}

@inproceedings{yang2024practical,
  title={Practical single domain generalization via training-time and test-time learning},
  author={Yang, Shuai and Zhang, Zhen and Gu, Lichuan},
  booktitle={Proceedings of the 30th ACM SIGKDD Conference on Knowledge Discovery and Data Mining},
  pages={3794--3805},
  year={2024}
}

@article{yang2026mutual,
  title={Mutual Information-Guided Style Augmentation for Single Domain Generalization},
  author={Yang, Shuai and Zhang, Zhen and Yu, Kui and Gu, Lichuan and Wu, Xindong},
  journal={ACM Transactions on Intelligent Systems and Technology},
  volume={17},
  number={3},
  pages={1--32},
  year={2026},
  publisher={ACM New York, NY}
}

@inproceedings{li2018learning,
  title={Learning to generalize: Meta-learning for domain generalization},
  author={Li, Da and Yang, Yongxin and Song, Yi-Zhe and Hospedales, Timothy},
  booktitle={Proceedings of the AAAI conference on artificial intelligence},
  volume={32},
  number={1},
  year={2018}
}

@article{zhuang2022surrogate,
  title={Surrogate gap minimization improves sharpness-aware training},
  author={Zhuang, Juntang and Gong, Boqing and Yuan, Liangzhe and Cui, Yin and Adam, Hartwig and Dvornek, Nicha and Tatikonda, Sekhar and Duncan, James and Liu, Ting},
  journal={arXiv preprint arXiv:2203.08065},
  year={2022}
}

@article{liu2024cross,
  title={Cross-Domain Feature Augmentation for Domain Generalization},
  author={Liu, Yingnan and Zou, Yingtian and Qiao, Rui and Liu, Fusheng and Lee, Mong Li and Hsu, Wynne},
  journal={arXiv preprint arXiv:2405.08586},
  year={2024}
}

@inproceedings{huangrepresentation,
  title     = {Representation Enhancement-Stabilization: Reducing Bias-Variance of Domain Generalization},
  author    = {Huang, Wei and Shi, Yilei and Xiong, Zhitong and Zhu, Xiao Xiang},
  booktitle = {European Conference on Computer Vision},
  pages     = {108--125},
  year      = {2024},
  organization = {Springer}
}

@article{zhou2022learning,
  title={Learning to prompt for vision-language models},
  author={Zhou, Kaiyang and Yang, Jingkang and Loy, Chen Change and Liu, Ziwei},
  journal={International Journal of Computer Vision},
  volume={130},
  number={9},
  pages={2337--2348},
  year={2022},
  publisher={Springer}
}

@inproceedings{li2025seeking,
  author    = {Li, Aodi and Zhuang, Liansheng and Long, Xiao and Yao, Minghong and Wang, Shafei},
  title     = {Seeking Consistent Flat Minima for Better Domain Generalization via Refining Loss Landscapes},
  booktitle = {Proceedings of the IEEE/CVF Conference on Computer Vision and Pattern Recognition},
  pages     = {15349--15359},
  year      = {2025}
}

@inproceedings{ballas2025gradient,
  author    = {Ballas, Aristotelis and Diou, Christos},
  title     = {Gradient-Guided Annealing for Domain Generalization},
  booktitle = {Proceedings of the IEEE/CVF Conference on Computer Vision and Pattern Recognition},
  pages     = {20558--20568},
  year      = {2025}
}

@inproceedings{wei2025indirect,
  author    = {Wei, Wei and Li, Zixiong and Yan, Jing and Shao, Mingwen and Li, Lin},
  title     = {Indirect Alignment and Relationship Preservation for Domain Generalization},
  booktitle = {Proceedings of the Thirty-Fourth International Joint Conference on Artificial Intelligence},
  pages     = {2054--2062},
  year      = {2025},
  doi       = {10.24963/ijcai.2025/229}
}

@inproceedings{wang2025rethinking,
  author    = {Wang, Zhenbin and Zhang, Lei and Huang, Wei and Zhang, Zhao and Wang, Zizhou},
  title     = {Rethinking Out-of-Distribution Detection and Generalization with Collective Behavior Dynamics},
  booktitle = {Advances in Neural Information Processing Systems},
  volume    = {38},
  pages     = {88401--88450},
  year      = {2025}
}

@inproceedings{thomas2025latent,
  author    = {Thomas, Xavier and Ghadiyaram, Deepti},
  title     = {What's in a Latent? Leveraging Diffusion Latent Space for Domain Generalization},
  booktitle = {Proceedings of the IEEE/CVF International Conference on Computer Vision},
  pages     = {2183--2194},
  year      = {2025}
}

@inproceedings{huang2024domainfusion,
  title     = {DomainFusion: Generalizing to Unseen Domains with Latent Diffusion Models},
  author    = {Huang, Yuyang and Chen, Yabo and Liu, Yuchen and Zhang, Xiaopeng and Dai, Wenrui and Xiong, Hongkai and Tian, Qi},
  booktitle = {European Conference on Computer Vision},
  pages     = {480--498},
  year      = {2024},
  publisher = {Springer},
  doi       = {10.1007/978-3-031-72940-9_27}
}

@inproceedings{zhuang2024time,
  title     = {Time-Varying LoRA: Towards Effective Cross-Domain Fine-Tuning of Diffusion Models},
  author    = {Zhuang, Zhan and Zhang, Yulong and Wang, Xuehao and Lu, Jiangang and Wei, Ying and Zhang, Yu},
  booktitle = {Advances in Neural Information Processing Systems},
  volume    = {37},
  pages     = {73920--73951},
  year      = {2024},
  doi       = {10.52202/079017-2351}
}

@inproceedings{noori2025fds,
  title     = {{FDS}: Feedback-Guided Domain Synthesis with Multi-Source Conditional Diffusion Models for Domain Generalization},
  author    = {Noori, Mehrdad and Cheraghalikhani, Milad and Bahri, Ali and Vargas Hakim, Gustavo Adolfo and Osowiechi, David and Yazdanpanah, Moslem and Ben Ayed, Ismail and Desrosiers, Christian},
  booktitle = {Proceedings of the IEEE/CVF Winter Conference on Applications of Computer Vision},
  pages     = {8504--8514},
  year      = {2025},
  doi       = {10.1109/WACV61041.2025.00824}
}

@article{choi2025trident,
  title   = {{TRIDENT}: Text-Free Data Augmentation Using Image Embedding Decomposition for Domain Generalization},
  author  = {Choi, Yoonyoung and Yu, Geunhyeok and Hwang, Hyoseok},
  journal = {IEEE Access},
  volume  = {13},
  pages   = {139816--139830},
  year    = {2025},
  doi     = {10.1109/ACCESS.2025.3596371}
}

@inproceedings{wang2025promea,
  title     = {{ProMEA}: Prompt-driven Expansion and Alignment for Single Domain Generalization},
  author    = {Wang, Yunyun and Guo, Yi and Liu, Xiaodong and Chen, Songcan},
  booktitle = {Proceedings of the Thirty-Fourth International Joint Conference on Artificial Intelligence},
  pages     = {2018--2026},
  year      = {2025}
}

@inproceedings{foret2021sharpness,
  title     = {Sharpness-Aware Minimization for Efficiently Improving Generalization},
  author    = {Foret, Pierre and Kleiner, Ariel and Mobahi, Hossein and Neyshabur, Behnam},
  booktitle = {International Conference on Learning Representations},
  year      = {2021}
}

@inproceedings{hu2022lora,
  title     = {{LoRA}: Low-Rank Adaptation of Large Language Models},
  author    = {Hu, Edward J. and Shen, Yelong and Wallis, Phillip and Allen-Zhu, Zeyuan and Li, Yuanzhi and Wang, Shean and Wang, Lu and Chen, Weizhu},
  booktitle = {International Conference on Learning Representations},
  year      = {2022}
}

@article{wu2023hpsv2,
  title   = {Human Preference Score v2: A Solid Benchmark for Evaluating Human Preferences of Text-to-Image Synthesis},
  author  = {Wu, Xiaoshi and Hao, Yiming and Sun, Keqiang and Chen, Yixiong and Zhu, Feng and Zhao, Rui and Li, Hongsheng},
  journal = {arXiv preprint arXiv:2306.09341},
  year    = {2023}
}

@article{oquab2024dinov2,
  title   = {{DINOv2}: Learning Robust Visual Features without Supervision},
  author  = {Oquab, Maxime and Darcet, Timoth{\'e}e and Moutakanni, Th{\'e}o and Vo, Huy and Szafraniec, Marc and Khalidov, Vasil and others},
  journal = {Transactions on Machine Learning Research},
  year    = {2024}
}

@inproceedings{zhang2018perceptual,
  title     = {The Unreasonable Effectiveness of Deep Features as a Perceptual Metric},
  author    = {Zhang, Richard and Isola, Phillip and Efros, Alexei A. and Shechtman, Eli and Wang, Oliver},
  booktitle = {Proceedings of the IEEE Conference on Computer Vision and Pattern Recognition},
  pages     = {586--595},
  year      = {2018}
}

@inproceedings{song2021ddim,
  title     = {Denoising Diffusion Implicit Models},
  author    = {Song, Jiaming and Meng, Chenlin and Ermon, Stefano},
  booktitle = {International Conference on Learning Representations},
  year      = {2021}
}

@inproceedings{he2016deep,
  title     = {Deep Residual Learning for Image Recognition},
  author    = {He, Kaiming and Zhang, Xiangyu and Ren, Shaoqing and Sun, Jian},
  booktitle = {Proceedings of the IEEE Conference on Computer Vision and Pattern Recognition},
  pages     = {770--778},
  year      = {2016}
}

@article{EKPC,
  title={EKPC: Elastic Knowledge Preservation and Compensation for Class-Incremental Learning: H. Wang et al.},
  author={Wang, Huaijie and Cheng, De and He, Lingfeng and Li, Yan and Li, Jie and Wang, Nannan and Gao, Xinbo},
  journal={International Journal of Computer Vision},
  volume={134},
  number={5},
  pages={238},
  year={2026},
  publisher={Springer}
}

@inproceedings{he2026harnessing,
  title={Harnessing textual semantic priors for knowledge transfer and refinement in clip-driven continual learning},
  author={He, Lingfeng and Cheng, De and Xu, Di and Wang, Huaijie and Wang, Nannan},
  booktitle={Proceedings of the AAAI Conference on Artificial Intelligence},
  volume={40},
  number={26},
  pages={21645--21653},
  year={2026}
}

@inproceedings{DoRA,
title={Task-Driven Subspace Decomposition for Knowledge Sharing and Isolation in LoRA-based Continual Learning},
author={Lingfeng He and De Cheng and Huaijie Wang and Xiaofeng Zhu and Xi Yang and Nannan Wang and Xinbo Gao},
booktitle={Forty-third International Conference on Machine Learning},
year={2026}
}

@inproceedings{li2026few,
  title={Few-Shot Hybrid Incremental Learning: Continually Learning under Data Scarcity and Task Uncertainty},
  author={Li, Yan and Shi, Yuzhu and Zhou, Kan and Zhang, Shu and He, Diqi and Zhang, Dingwen and Han, Junwei},
  booktitle={Proceedings of the IEEE/CVF Conference on Computer Vision and Pattern Recognition},
  pages={32334--32344},
  year={2026}
}

@article{CBCM,
  title={Dual-Branch Cross-Projection Debiasing through Diffusion-based Disentanglement},
  author={Zhao, Xiangqian and Jiang, Xinyang and Xu, Zhipeng and He, Lingfeng and Wang, Zilong and Li, Dongsheng and Cheng, De and Wang, Nannan},
  journal={arXiv preprint arXiv:2606.24161},
  year={2026}
}

@inproceedings{SIKD,
  title     = {Symbiosis-Inspired Knowledge Distillation for Incremental Object Detection},
  author    = {Zeng, Mingyue and Cheng, De and Xu, Zhipeng and Wang, Huaijie and Wang, Nannan and Gao, Xinbo},
  booktitle = {Proceedings of the 43rd International Conference on Machine Learning},
  year      = {2026}
}

@article{CKAA,
  title={Ckaa: Cross-subspace knowledge alignment and aggregation for robust continual learning},
  author={He, Lingfeng and Cheng, De and Ma, Zhiheng and Wang, Huaijie and Zhang, Dingwen and Wang, Nannan and Gao, Xinbo},
  journal={arXiv preprint arXiv:2507.09471},
  year={2025}
}

\begin{IEEEbiography}[{\includegraphics[width=1in,height=1.25in,clip,keepaspectratio]{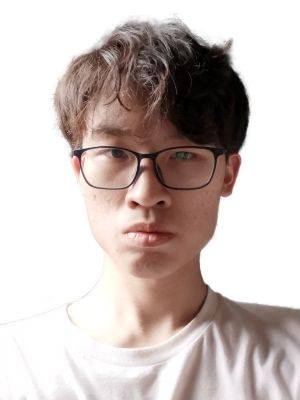}}]{Zhipeng Xu}
received the B.Eng. and M.Eng. degrees in Information and Communication Engineering from Xidian University, Xi’an, China, in 2023 and 2026, respectively. He is currently pursuing the Ph.D. degree at The Hong Kong University of Science and Technology. His research interests include domain generalization, parameter-efficient adaptation of foundation models, multimodal large language models, and AI agents.
\end{IEEEbiography}
\vspace{-5.5mm}

\begin{IEEEbiography}[{\includegraphics[width=1in,height=1.25in,clip,keepaspectratio]{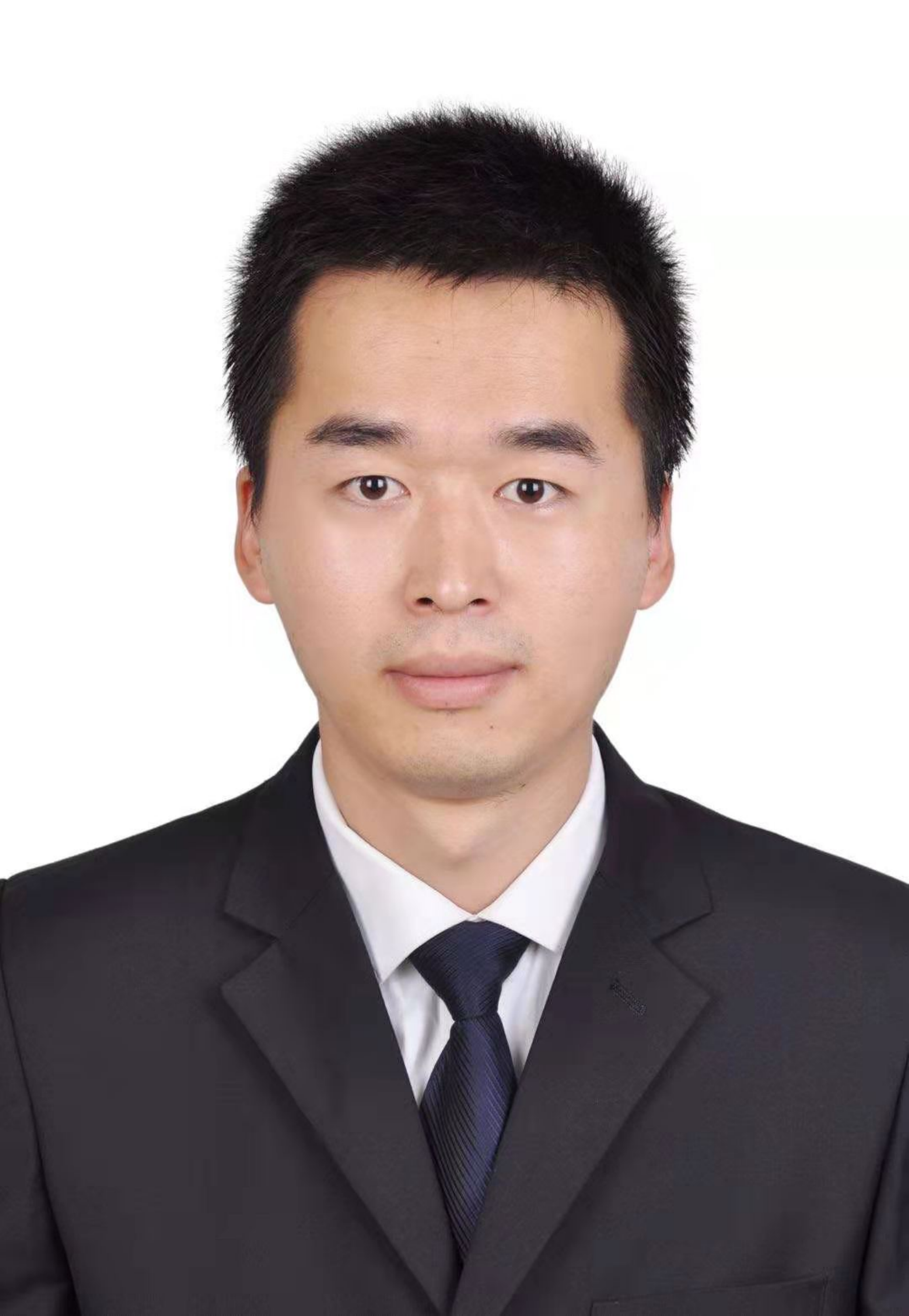}}]{De Cheng} is an associate professor with the School of Telecommunications Engineering, Xidian University, China. He received the B.S. and Ph.D. degrees from Xi'an Jiaotong University, Xi'an, China, in 2011 and 2017, respectively. From 2015 to 2017, he was a visiting scholar at Carnegie Mellon University, Pittsburgh, USA. His research interests include pattern recognition, machine learning, and multimedia analysis.
\end{IEEEbiography}
\vspace{-5.5mm}

\begin{IEEEbiography}[{\includegraphics[width=1in,height=1.25in,clip,keepaspectratio]{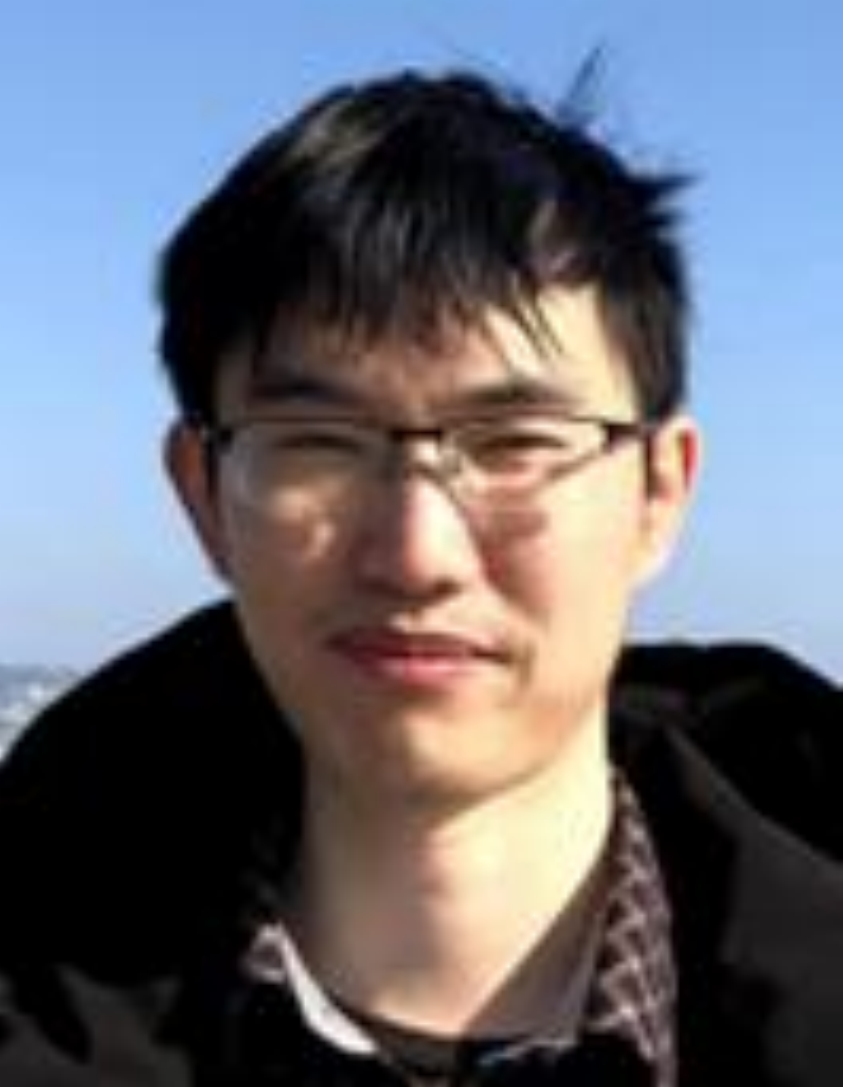}}]{Xinyang Jiang} received B.E. from Zhejiang University in 2012 and Ph.D. from Zhejiang University in 2017. He is currently a researcher at Microsoft Research Asia. Before joining MSRA, he was a researcher from Tencent Youtu Lab. His main research field is computer vision, including person Re-identification, vector graphics recognition and medical image understanding.
\end{IEEEbiography}
\vspace{-5.5mm}

\begin{IEEEbiography}[{\includegraphics[width=1in,height=1.23in,clip,keepaspectratio]{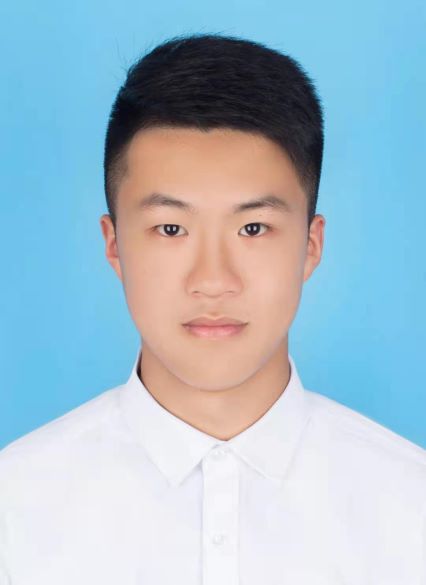}}]{Lingfeng He}
received the B.Sc. and M.Eng degree from Xidian
University, Xi'an, China, in 2023 and 2026, respectively. He is currently
pursuing his Ph.D. degree in Electronic and Computer Engineering in the Hong Kong University of Science and Technology. His research interests include continual learning, parameter-efficient fine-tuning and person ReID.
\end{IEEEbiography}
\vspace{-5.5mm}

\begin{IEEEbiography}[{\includegraphics[width=1in,height=1.25in,clip,keepaspectratio]{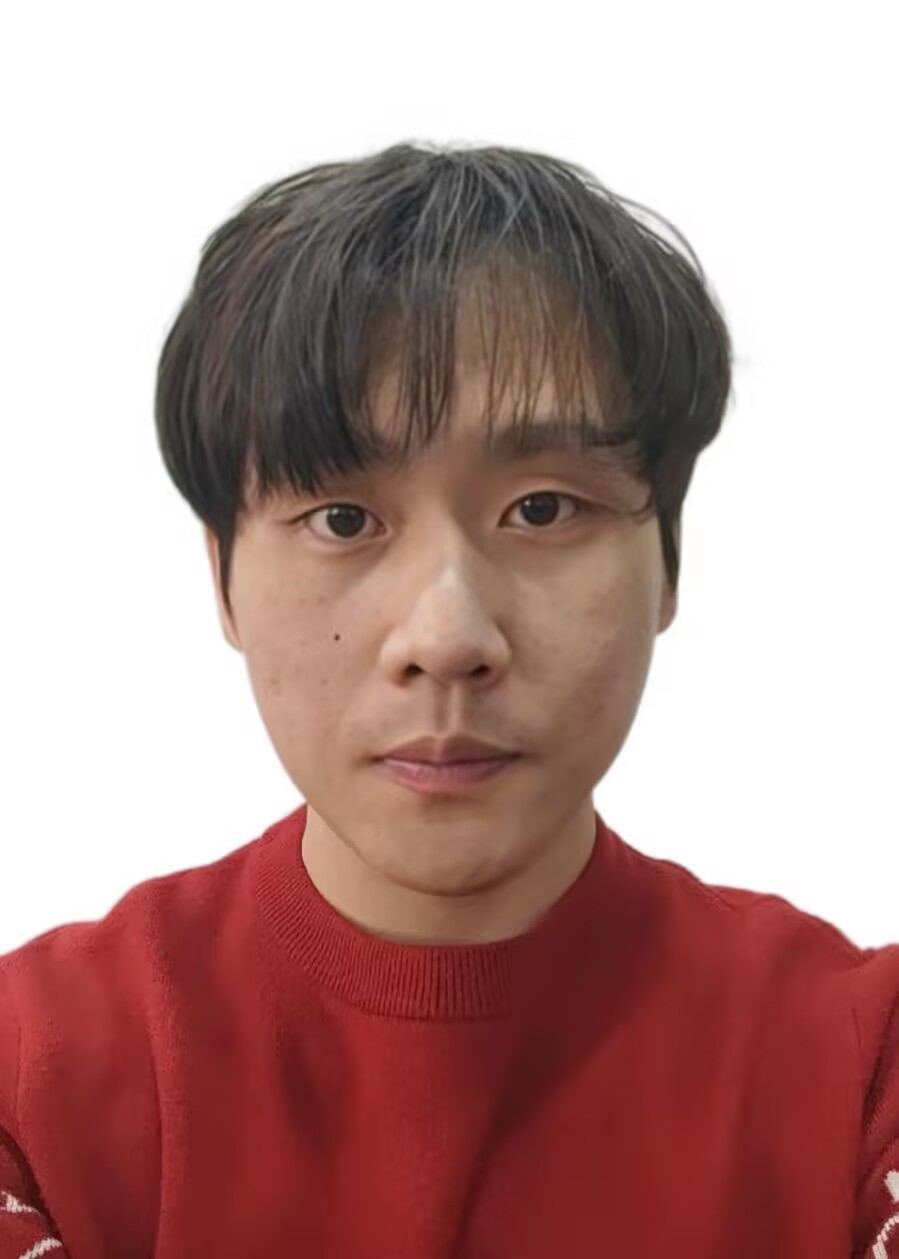}}]{Huaijie Wang}
received the B.Sc. degree from Xidian
University, Xi'an, China, in 2024. He is currently
pursuing his Ph.D. degree in School of Electronic Engineering in Xidian University. His research interest is continual learning.
\end{IEEEbiography}
\vspace{-5.5mm}

\begin{IEEEbiography}[{\includegraphics[width=1in,height=1.25in,clip,keepaspectratio]{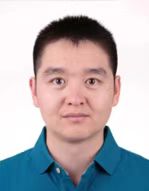}}]{Dongsheng Li} received B.E. from University of Science and Technology of China in 2007 and Ph.D. from Fudan University in 2012. He is now a principal research manager with Microsoft Research Asia (MSRA) since February 2020.  Before joining MSRA, he was a research staff member with IBM Research – China  since April 2015. He is also an adjunct professor with  School of Computer Science, Fudan University, Shanghai, China. His research interests include recommender systems and machine learning applications. His work on cognitive recommendation engine won the 2018 IBM Corporate Award.
\end{IEEEbiography}
\vspace{-5.5mm}

\begin{IEEEbiography}[{\includegraphics[width=1in,height=1.25in,clip,keepaspectratio]{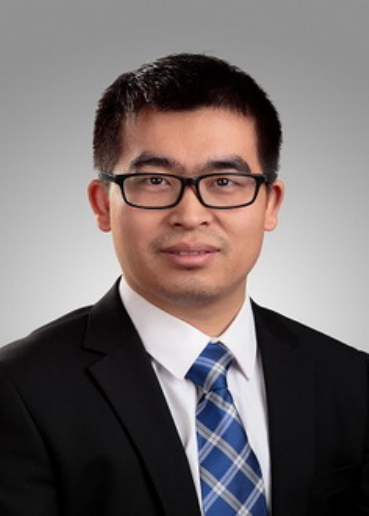}}]{Nannan Wang}
(Senior Member, IEEE) received the B.Sc. degree in information and computation science from the Xi'an University of Posts and Telecommunications in 2009 and the Ph.D. degree in information and telecommunications engineering from Xidian University in 2015. From September 2011 to September 2013, he was a Visiting Ph.D. Student with the University of Technology, Sydney, NSW, Australia. He is currently a Professor with the State Key Laboratory of Integrated Services Networks, Xidian University. He has published over 100 articles in refereed journals and proceedings, including IEEE T-PAMI, IJCV, CVPR, ICCV, etc. His current research interests include computer vision and machine learning.
\end{IEEEbiography}
\vspace{-5.5mm}

\begin{IEEEbiography}[{\includegraphics[width=1in,height=1.25in,clip,keepaspectratio]{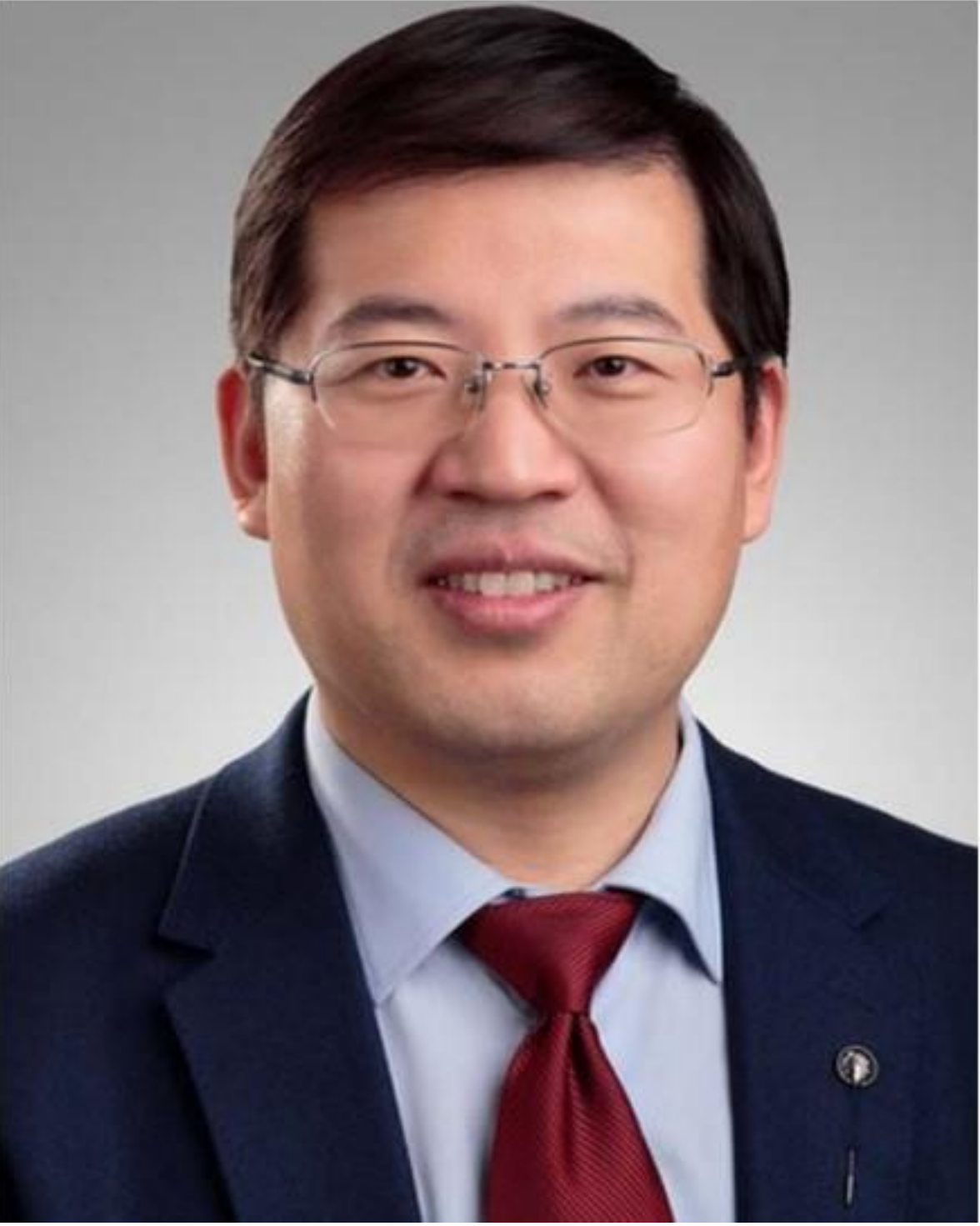}}]{Xinbo Gao}
(M'02-SM'07-F'24) received the B.Eng., M.Sc. and Ph.D. degrees in electronic engineering, signal and information processing from Xidian University, Xi’an, China, in 1994, 1997, and 1999, respectively. From 1997 to 1998, he was a research fellow at the Department of Computer Science, Shizuoka University, Shizuoka, Japan. From 2000 to 2001, he was a post-doctoral research fellow at the Department of Information Engineering, the Chinese University of Hong Kong, Hong Kong.
Since 1999, he has been at the School of Electronic Engineering, Xidian University and now he is a Professor of Pattern Recognition and Intelligent System of Xidian University. Since 2020, he has been also a Professor of Computer Science and Technology of Chongqing University of Posts and Telecommunications. His current research interests include computer vision, machine learning and pattern recognition. He has published seven books and around 300 technical articles in refereed journals and proceedings. Prof. Gao is on the Editorial Boards of several journals, including Signal Processing (Elsevier) and Neurocomputing (Elsevier). He served as the General Chair/Co-Chair, Program Committee Chair/Co-Chair, or PC Member for around 30 major international conferences. He is a Fellow of the IEEE, IET, AAIA, CIE, CCF, and CAAI.
\end{IEEEbiography}

\appendices
\section{Theoretical Analysis}\label{appx:theory}
\subsection{Notation and Regularity Conditions}
\label{app:notation_and_assumptions}

\noindent Before presenting the proofs, we collect the distributional
objects and regularity conditions used in the theoretical analysis. Unless
otherwise specified, all notation follows
Secs.~\ref{sec:preliminaries}, and
\ref{sec:classifier_guided_generation}.

\vspace{2.0mm}

\noindent\textbf{SDG setting and diffusion notation.}
PAPT++ is trained on labeled samples from the source domain,
$\mathcal{D}^{\mathcal{S}}=
\{(\mathbf{x}_i^s,y_i^s)\}_{i=1}^{N_s}.
$ We evaluate the learned model on an unseen target domain \(\mathcal{D}^{\mathcal{T}}\). Consistent with the single-domain
generalization setting, no data from \(\mathcal{D}^{\mathcal{T}}\) are
used during training.
Let \(\mathcal{Z}\subseteq\mathbb{R}^{d_z}\) denote the latent space of the pretrained VAE. As defined in Sec.~\ref{sec:preliminaries}, \(\mathcal{F}_E\) and \(\mathcal{F}_D\) denote the frozen VAE encoder
and decoder, respectively. For an input image \(\mathbf{x}\),
\(\mathbf{z}_0=\mathcal{F}_E(\mathbf{x})\) denotes its latent
representation before noise is added in the forward diffusion process.
For class \(k\), the class-level text prompt and its corresponding text
embedding are defined as:
\begin{equation}
    \mathbf{t}_{k,\mathrm{ref}}^c
    =
    \text{``a photo of a [class]''},
    \qquad
    \mathbf{T}_{k,\mathrm{ref}}
    =
    \boldsymbol{\tau}_{\boldsymbol{\psi}}
    (\mathbf{t}_{k,\mathrm{ref}}^c),
    \label{eq:app_text_condition}
\end{equation}
consistent with Eq.~\ref{eq:prompt}. 
After CSRL, the reference-learning LoRA parameters
\(\boldsymbol{\phi}_{\mathrm{ref}}\) are merged into the
pretrained U-Net backbone to obtain the fixed
reference-adapted backbone
\(\bar{\boldsymbol{\omega}}_{\mathrm{ref}}\).
During CADS,
\(\bar{\boldsymbol{\omega}}_{\mathrm{ref}}\) and the
text-encoder parameters \(\boldsymbol{\psi}\) remain frozen,
and only the self-attention LoRA parameters
\(\boldsymbol{\phi}\) are optimized. The LoRA parameters used
in synthesis round \(r\) are denoted by
\(\boldsymbol{\phi}^{(r)}\). For simplicity, we write:
\begin{equation}
\boldsymbol{\epsilon}_{\boldsymbol{\phi}}
=
\boldsymbol{\epsilon}_{
\bar{\boldsymbol{\omega}}_{\mathrm{ref}},
\boldsymbol{\phi}}
\label{eq:app_noise_predictor_shorthand}
\end{equation}
throughout the following analysis.

Let \(T\) denote the total number of diffusion steps. Given the
variance schedule \(\{\beta_t\}_{t=1}^{T}\), we define
\(\alpha_t=1-\beta_t\) and
\(\bar{\alpha}_t=\prod_{s=1}^{t}\alpha_s\). The forward diffusion
process progressively adds Gaussian noise to the clean latent
\(\mathbf{z}_0\) according to:
\begin{equation}
    q(\mathbf{z}_t\mid\mathbf{z}_{t-1})
    =
    \mathcal{N}\!\left(
        \mathbf{z}_t;
        \sqrt{\alpha_t}\,\mathbf{z}_{t-1},
        \beta_t\mathbf{I}
    \right).
\end{equation}
As shown in Eq.~\ref{eq:forward_diffusion_noising}, the noisy latent
at any diffusion step \(t\) can be sampled directly from
\(\mathbf{z}_0\) as:
\begin{equation}
    \mathbf{z}_t
    =
    \sqrt{\bar{\alpha}_t}\,\mathbf{z}_0
    +
    \sqrt{1-\bar{\alpha}_t}\,\boldsymbol{\epsilon},
    \qquad
    \boldsymbol{\epsilon}
    \sim\mathcal{N}(\mathbf{0},\mathbf{I}).
    \label{eq:app_forward_noising}
\end{equation}
Equivalently,
\begin{equation}
    q(\mathbf{z}_t\mid\mathbf{z}_0)
    =
    \mathcal{N}\!\left(
        \mathbf{z}_t;
        \sqrt{\bar{\alpha}_t}\,\mathbf{z}_0,
        (1-\bar{\alpha}_t)\mathbf{I}
    \right).
\end{equation}

\vspace{2.0mm}

\noindent\textbf{Reference and generated latent distributions.}
For class \(k\), let
\(\mathcal{A}_k=\{\mathbf{a}_{k,m}\}_{m=1}^{M_a}\)
denote the semantic reference set constructed in
Eq.~\ref{eq:reference_set}. Following
Eq.~\ref{eq:reference_latent}, we encode each reference image as:
\begin{equation}
    \mathbf{z}_{k,m}^{a}
    =
    \mathcal{F}_E(\mathbf{a}_{k,m}),
    \label{eq:app_reference_latent}
\end{equation}
where \(\mathbf{z}_{k,m}^{a}\) is the latent representation of the \(m\)-th reference image before forward noising. The encoded reference set defines the empirical class-conditional measure:
\begin{equation}
    \widehat{P}_A^k
    =
    \frac{1}{M_a}
    \sum_{m=1}^{M_a}
    \delta_{\mathbf{z}_{k,m}^{a}},
    \label{eq:app_empirical_reference_distribution}
\end{equation}
where \(\delta_{\mathbf z}\) denotes the Dirac probability measure at
\(\mathbf z\).

For a given LoRA parameter
\(\boldsymbol{\phi}\in\Phi\), the text-conditioned reverse diffusion process defines the trajectory density:
\begin{equation}
    p_{\boldsymbol{\phi}}^k(\mathbf{z}_{0:T})
    =
    p_{\mathrm{prior}}(\mathbf{z}_T)
    \prod_{t=1}^{T}
    p_{\boldsymbol{\phi}}
    \left(
        \mathbf{z}_{t-1}
        \mid
        \mathbf{z}_t,
        \mathbf{T}_{k,\mathrm{ref}}
    \right),
    \label{eq:app_reverse_trajectory}
\end{equation}
where
\begin{equation}
    p_{\mathrm{prior}}(\mathbf{z}_T)
    =
    \mathcal{N}
    \left(
        \mathbf{z}_T;
        \mathbf{0},
        \mathbf{I}
    \right).
\end{equation}
Under the fixed variance reverse process parameterization considered in the analysis, each reverse transition is given by:
\begin{equation}
\begin{aligned}
    &p_{\boldsymbol{\phi}}
    \left(
        \mathbf{z}_{t-1}
        \mid
        \mathbf{z}_t,
        \mathbf{T}_{k,\mathrm{ref}}
    \right)
    \\
    &\quad =
    \mathcal{N}\!\left(
        \mathbf{z}_{t-1};
        \boldsymbol{\mu}_{\boldsymbol{\phi}}
        \left(
            \mathbf{z}_t,
            t,
            \mathbf{T}_{k,\mathrm{ref}}
        \right),
        \sigma_t^2\mathbf{I}
    \right),
    \qquad
    \sigma_t^2>0,
\end{aligned}
\label{eq:app_reverse_transition}
\end{equation}
where the reverse mean
\(\boldsymbol{\mu}_{\boldsymbol{\phi}}\) is parameterized through the
noise-prediction network
\(\boldsymbol{\epsilon}_{\boldsymbol{\phi}}\).
We define \(P_{\boldsymbol{\phi}}^k\) as the
\(\mathbf{z}_0\)-marginal distribution of
Eq.~\ref{eq:app_reverse_trajectory}. Its density is:
\begin{equation}
    p_{\boldsymbol{\phi}}^k(\mathbf{z}_0)
    =
    \int
    p_{\boldsymbol{\phi}}^k(\mathbf{z}_{0:T})
    \,\mathrm{d}\mathbf{z}_{1:T}.
    \label{eq:app_generated_latent_marginal}
\end{equation}
Thus, \(P_{\boldsymbol{\phi}}^k\) is the class-conditional latent
distribution generated by the LoRA-adapted diffusion model under the
text condition \(\mathbf{T}_{k,\mathrm{ref}}\). The distributions
attainable by varying the LoRA parameters form the generator-induced class-conditional family:
\begin{equation}
    \mathcal{P}_{\Phi}^k
    =
    \left\{
        P_{\boldsymbol{\phi}}^k:
        \boldsymbol{\phi}\in\Phi
    \right\}.
    \label{eq:app_generator_family}
\end{equation}

Because the reverse transitions in
Eq.~\ref{eq:app_reverse_transition} have nondegenerate Gaussian
covariances, \(P_{\boldsymbol{\phi}}^k\) is absolutely continuous with
respect to the Lebesgue measure. In particular, it assigns zero
probability to every individual latent point. By contrast,
\(\widehat{P}_A^k\) assigns positive probability to each encoded
reference latent. Therefore,
\(\widehat{P}_A^k\) is not absolutely continuous with respect to
\(P_{\boldsymbol{\phi}}^k\), and
\(
    D_{\mathrm{KL}}
    \left(
        \widehat{P}_A^k
        \,\middle\|\,
        P_{\boldsymbol{\phi}}^k
    \right)
    =
    +\infty.
    \label{eq:app_empirical_generated_kl}
\)

To remove this discrete--continuous support mismatch from the distributional comparison used in the analysis, we introduce the Gaussian-smoothed reference distribution:
\begin{equation}
\begin{aligned}
    P_A^k
    &=
    \widehat{P}_A^k
    *
    \mathcal{N}(\mathbf{0},\sigma_A^2\mathbf{I}) \\
    &=
    \frac{1}{M_a}
    \sum_{m=1}^{M_a}
    \mathcal{N}\!\left(
        \mathbf{z}_{k,m}^{a},
        \sigma_A^2\mathbf{I}
    \right),
    \qquad
    \sigma_A>0,
    \label{eq:app_smoothed_reference_distribution}
\end{aligned}
\end{equation}
where \(*\) denotes convolution and \(\sigma_A>0\) is a fixed smoothing bandwidth that controls the spread of each Gaussian component around the corresponding reference latent. As \(\sigma_A\downarrow0\), \(P_A^k\) converges weakly to the empirical reference measure \(\widehat{P}_A^k\).
Equivalently, a latent \(\mathbf{z}_0^a\sim P_A^k\) can be generated by sampling:
\begin{equation}
    M
    \sim
    \operatorname{Unif}\{1,\ldots,M_a\},
    \qquad
    \boldsymbol{\xi}
    \sim
    \mathcal{N}(\mathbf{0},\sigma_A^2\mathbf{I}),
\end{equation}
independently, and setting:
\begin{equation}
    \mathbf{z}_0^a
    =
    \mathbf{z}_{k,M}^{a}
    +
    \boldsymbol{\xi}.
\end{equation}

Based on the smoothed reference distribution \(P_A^k\), we define the following class-wise denoising objective as:
\begin{equation}
\begin{aligned}
    \mathcal{L}_{\mathrm{den}}^k
    (\boldsymbol{\phi};\mathcal{A}_k)
    =
    \mathbb{E}_{\substack{
        \mathbf{z}_0^a\sim P_A^k,\;
        t\sim\nu,\;
        \boldsymbol{\epsilon}\sim
        \mathcal{N}(\mathbf{0},\mathbf{I})
    }}
    \left[
        \left\|
        \boldsymbol{\epsilon}
        -
        \boldsymbol{\epsilon}_{\boldsymbol{\phi}}
        \left(
            \mathbf{z}_t^a,
            t,
            \mathbf{T}_{k,\mathrm{ref}}
        \right)
        \right\|_2^2
    \right],
\end{aligned}
\label{eq:app_classwise_denoising_loss}
\end{equation}
where \(\nu\) is the timestep-sampling distribution used during
training and:
\begin{equation}
    \mathbf{z}_t^a
    =
    \sqrt{\bar{\alpha}_t}\,\mathbf{z}_0^a
    +
    \sqrt{1-\bar{\alpha}_t}\,
    \boldsymbol{\epsilon}.
    \label{eq:app_reference_forward_noising}
\end{equation}

The implemented round-wise loss
\(\mathcal{L}_{\mathrm{den}}^{(r)}\) in
Eq.~\ref{eq:reference_denoising_loss} samples encoded
reference latents from the empirical measure
\(\widehat{P}_A^k\). For the distributional analysis, we
instead use its Gaussian-smoothed counterpart \(P_A^k\) in
Eq.~\ref{eq:app_classwise_denoising_loss}. This smoothing
removes the discrete--continuous support mismatch and makes
the KL comparison well-defined. Thus,
Eq.~\ref{eq:app_classwise_denoising_loss} is a smoothed
analytical counterpart of the implemented reference-denoising
objective, while the actual optimization continues to use
Eq.~\ref{eq:reference_denoising_loss}.
Starting from \(\mathbf{z}_0\sim P_A^k\), the forward diffusion process defines the reference trajectory density:
\begin{equation}
    q_A^k(\mathbf{z}_{0:T})
    =
    p_A^k(\mathbf{z}_0)
    \prod_{t=1}^{T}
    q(\mathbf{z}_t\mid\mathbf{z}_{t-1}).
    \label{eq:app_reference_trajectory}
\end{equation}
For the proof of
Theorem~\ref{thm:semantic_consistency}, we consider the trajectory-level KL divergence:
\begin{equation}
    \mathcal{K}_{\mathrm{traj}}^k(\boldsymbol{\phi})
    =
    D_{\mathrm{KL}}
    \left(
        q_A^k(\mathbf{z}_{0:T})
        \,\middle\|\,
        p_{\boldsymbol{\phi}}^k(\mathbf{z}_{0:T})
    \right).
    \label{eq:app_trajectory_kl}
\end{equation}
By construction, \(P_A^k\) and \(P_{\boldsymbol{\phi}}^k\) are the \(\mathbf{z}_0\)-marginals of
\(q_A^k(\mathbf{z}_{0:T})\) and
\(p_{\boldsymbol{\phi}}^k(\mathbf{z}_{0:T})\).
Since marginalizing out the intermediate states \(\mathbf{z}_{1:T}\) cannot increase the KL divergence, we obtain:
\begin{equation}
    D_{\mathrm{KL}}
    \left(
        P_A^k
        \,\middle\|\,
        P_{\boldsymbol{\phi}}^k
    \right)
    \leq
    \mathcal{K}_{\mathrm{traj}}^k(\boldsymbol{\phi}).
    \label{eq:app_marginal_kl}
\end{equation}
This is an application of the data-processing inequality to the
projection
\(\pi_0(\mathbf{z}_{0:T})=\mathbf{z}_0\).

\vspace{2.0mm}

\noindent\textbf{Generator-induced semantic ambiguity set.}
For a semantic tolerance \(\delta_{\mathrm{sem}}>0\), we first define the
class-wise feasible LoRA parameter set as:
\begin{equation}
    \Phi_{\mathrm{sem}}^k(\delta_{\mathrm{sem}})
    =
    \left\{
        \boldsymbol{\phi}\in\Phi
        \;\middle|\;
        \mathcal{L}_{\mathrm{den}}^k
        (\boldsymbol{\phi};\mathcal{A}_k)
        \leq
        \delta_{\mathrm{sem}}
    \right\}.
    \label{eq:app_semantic_parameter_set}
\end{equation}
The corresponding generator-induced semantic ambiguity set can be defined as:
\begin{equation}
    \mathcal{B}_{\mathrm{sem}}^k(\delta_{\mathrm{sem}})
    =
    \left\{
        P_{\boldsymbol{\phi}}^k
        :
        \boldsymbol{\phi}
        \in
        \Phi_{\mathrm{sem}}^k(\delta_{\mathrm{sem}})
    \right\}
    \subseteq
    \mathcal{P}_{\Phi}^k.
    \label{eq:app_semantic_ambiguity_set}
\end{equation}
Because CADS uses a single shared LoRA parameter across all classes, we
assume that the joint feasible parameter set is nonempty:
\begin{equation}
    \bigcap_{k=1}^{C}
    \Phi_{\mathrm{sem}}^k(\delta_{\mathrm{sem}})
    \neq
    \varnothing.
    \label{eq:app_joint_semantic_feasibility}
\end{equation}

Eq.~\ref{eq:app_semantic_ambiguity_set} defines a distribution-valued
set obtained by mapping the semantically feasible LoRA parameters to their
induced class-conditional clean-latent distributions. Thus, although the
elements of the set are distributions, semantic feasibility is imposed at
the generator-parameter level through the reference-based denoising
objective. From this perspective, CADS can be interpreted as a model-restricted,
DRO-inspired search for high-risk class-conditional distributions induced
by semantically feasible LoRA parameters. The semantic regularization
term in Eq.~\ref{eq:generation_objective} provides an empirical penalized
counterpart of the hard feasibility constraint in Eq.~\ref{eq:app_semantic_parameter_set}.

\vspace{2.0mm}

\noindent\textbf{Target-domain risk and ambiguity-set coverage.}
Let \(P_T\) denote the joint distribution of images and labels in the unseen target domain \(\mathcal{D}^{\mathcal{T}}\), and let:
\begin{equation}
    \pi_T^k
    =
    P_T(Y=k)
    \label{eq:app_target_class_prior}
\end{equation}
denote the target-domain prior of class \(k\). We use \(P_T^k\) to
represent the corresponding class-conditional latent distribution. More
specifically, for every measurable set \(E\subseteq\mathcal{Z}\),
\begin{equation}
    P_T^k(E)
    =
    P_T
    \left(
        \mathcal{F}_E(\mathbf{X})\in E
        \mid
        Y=k
    \right).
    \label{eq:app_target_latent_distribution}
\end{equation}

For convenience, define the class-wise loss function:
\begin{equation}
    g_{\boldsymbol{\theta},k}(\mathbf{z})
    =
    \ell_{\mathrm{CE}}
    \left(
        \mathbf{f}_{\boldsymbol{\theta}}
        (\mathcal{F}_D(\mathbf{z})),
        k
    \right).
    \label{eq:app_classwise_loss_function}
\end{equation}
For any class-\(k\) latent distribution \(P\), the corresponding classification risk is:
\begin{equation}
    \mathcal{R}_k
    (\mathbf{f}_{\boldsymbol{\theta}};P)
    =
    \mathbb{E}_{\mathbf{z}\sim P}
    \left[
        g_{\boldsymbol{\theta},k}(\mathbf{z})
    \right].
    \label{eq:app_classwise_risk}
\end{equation}
The overall risk on the unseen target domain is then given by
\begin{equation}
    \mathcal{R}_T
    (\mathbf{f}_{\boldsymbol{\theta}})
    =
    \sum_{k=1}^{C}
    \pi_T^k
    \mathcal{R}_k
    (\mathbf{f}_{\boldsymbol{\theta}};P_T^k).
    \label{eq:app_target_risk}
\end{equation}
This is the target-domain risk considered in
Theorem~\ref{thm:target_risk_bound}.

To quantify the discrepancy between each unseen class-conditional distribution and the semantic ambiguity set, we use the total variation distance:
\begin{equation}
    d_{\mathrm{TV}}(P,Q)
    =
    \sup_{E\in\mathfrak{B}(\mathcal{Z})}
    |P(E)-Q(E)|,
    \label{eq:app_total_variation}
\end{equation}
where \(P\) and \(Q\) are distributions on \(\mathcal{Z}\), and
\(\mathfrak{B}(\mathcal{Z})\) denotes the Borel \(\sigma\)-algebra.
For each class \(k\), we define the residual coverage gap as:
\begin{equation}
    \rho_k
    =
    \inf_{P\in
    \mathcal{B}_{\mathrm{sem}}^k(\delta_{\mathrm{sem}})}
    d_{\mathrm{TV}}(P_T^k,P).
    \label{eq:app_coverage_gap}
\end{equation}
Thus, \(\rho_k\) can be seen as the smallest remaining mismatch between the unseen class-conditional distribution \(P_T^k\) and the distributions covered by the semantic ambiguity set. A smaller \(\rho_k\) indicates that the ambiguity set provides better coverage of the corresponding target-domain shift.

\vspace{2.0mm}

\noindent\textbf{Regularity conditions.}
The subsequent proofs use the following conditions.

\vspace{2.0mm}

\noindent\textbf{A1. Fixed-variance noise-prediction parameterization.}
The forward schedule satisfies
\(\beta_t\in(0,1)\) for all \(t=1,\ldots,T\). For \(t=2,\ldots,T\), the forward posterior and the reverse transition
are Gaussian with the same fixed covariance:
\begin{align}
    &q(\mathbf{z}_{t-1}\mid\mathbf{z}_t,\mathbf{z}_0)
    =
    \mathcal{N}
    \left(
        \widetilde{\boldsymbol{\mu}}_t,
        \widetilde{\beta}_t\mathbf{I}
    \right),
    \\
    &p_{\boldsymbol{\phi}}
    (\mathbf{z}_{t-1}\mid
     \mathbf{z}_t,\mathbf{T}_{k,\mathrm{ref}})
    =
    \mathcal{N}
    \left(
        \boldsymbol{\mu}_{\boldsymbol{\phi},t},
        \widetilde{\beta}_t\mathbf{I}
    \right),
\end{align}
where
\begin{equation}
    \widetilde{\beta}_t
    =
    \frac{1-\bar{\alpha}_{t-1}}
         {1-\bar{\alpha}_t}
    \beta_t
    >0.
\end{equation}
The reverse mean follows the standard noise-prediction parameterization through \(\boldsymbol{\epsilon}_{\boldsymbol{\phi}}\).

The reconstruction distribution
\(p_{\boldsymbol{\phi}}
(\mathbf{z}_0\mid\mathbf{z}_1,\mathbf{T}_{k,\mathrm{ref}})\)
has covariance \(\sigma_1^2\mathbf{I}\), with
\(0<\sigma_1^2<\infty\), and its mean is parameterized through the same
noise predictor. Under this parameterization, the reconstruction term at
\(t=1\) and each transition KL term for \(t\geq2\) can be expressed as a
finite, strictly positive, timestep-dependent coefficient multiplying the
corresponding squared noise-prediction error, plus a term independent of
\(\boldsymbol{\phi}\).

The timestep distribution \(\nu\) has full support:
\(\nu(t)>0\) for every \(t\in\{1,\ldots,T\}\). Since \(T\) is finite, if
\(w_t\) denotes the coefficient of the squared noise-prediction error at
timestep \(t\), it follows that:
\begin{equation}
    C_{\nu}
    =
    \max_{1\leq t\leq T}
    \frac{w_t}{\nu(t)}
    <
    \infty.
    \label{eq:app_timestep_weight_constant}
\end{equation}
Thus, the weighted noise-prediction terms arising from the variational
decomposition can be controlled by the denoising objective in
Eq.~\ref{eq:app_classwise_denoising_loss}.
Since \(P_A^k\) is a finite Gaussian mixture and
\(p_{\mathrm{prior}}=\mathcal{N}(\mathbf{0},\mathbf{I})\), the terminal
term satisfies:
\begin{equation}
    \mathbb{E}_{\mathbf{z}_0\sim P_A^k}
    \left[
        D_{\mathrm{KL}}
        \left(
            q(\mathbf{z}_T\mid\mathbf{z}_0)
            \,\middle\|\,
            p_{\mathrm{prior}}
        \right)
    \right]
    <\infty.
    \label{eq:app_terminal_integrability}
\end{equation}
We additionally assume that the squared noise-prediction errors in
Eq.~\ref{eq:app_classwise_denoising_loss} are integrable for every
\(\boldsymbol{\phi}\in\Phi\) considered in the analysis.

\vspace{2.0mm}

\noindent\textbf{A2. Uniformly bounded classification loss.}
Let \(\mathcal{H}\) denote the hypothesis class of  classifiers.
We assume that there exists a constant \(L_{\max}<\infty\), independent of the classifier and the class index, such that, for every \(\mathbf{f}_{\boldsymbol{\theta}}\in\mathcal{H}\) and every class \(k\),
\begin{equation}
    0
    \leq
    \ell_{\mathrm{CE}}
    \left(
        \mathbf{f}_{\boldsymbol{\theta}}
        (\mathcal{F}_D(\mathbf{z})),
        k
    \right)
    \leq
    L_{\max}
    \label{eq:app_bounded_loss}
\end{equation}
holds \(P_T^k\)-almost surely and \(P\)-almost surely for every
\(P\in
\mathcal{B}_{\mathrm{sem}}^k(\delta_{\mathrm{sem}})\).
For the standard cross-entropy loss used in the classifier objective, this condition amounts to assuming that the predicted probability assigned to the labeled class is uniformly bounded below by \(\exp(-L_{\max})\) over the distributions above.
Under this assumption, for any such distribution \(P\),
\begin{equation}
\left|
    \mathcal{R}_k
    (\mathbf{f}_{\boldsymbol{\theta}};P_T^k)
    -
    \mathcal{R}_k
    (\mathbf{f}_{\boldsymbol{\theta}};P)
\right|
\leq
L_{\max}
d_{\mathrm{TV}}(P_T^k,P).
\label{eq:app_tv_risk_difference}
\end{equation}

\vspace{2.0mm}

\noindent\textbf{A3. Round-wise semantic feasibility.}
At synthesis round \(r\), the generator parameters are optimized while the
classifier
\(\mathbf{f}_{\boldsymbol{\theta}^{(r-1)}}\) is held fixed, as in
Eq.~\ref{eq:generation_objective}. As part of
Assumption~\ref{assump:approx_adv_search}, we assume that the resulting
shared LoRA parameter is semantically feasible for every class:
\begin{equation}
    \boldsymbol{\phi}^{(r)}
    \in
    \bigcap_{k=1}^{C}
    \Phi_{\mathrm{sem}}^k(\delta_{\mathrm{sem}}).
    \label{eq:app_roundwise_parameter_feasibility}
\end{equation}
By the definition of the generator-induced semantic ambiguity set, this implies:
\begin{equation}
    P_{\boldsymbol{\phi}^{(r)}}^k
    \in
    \mathcal{B}_{\mathrm{sem}}^k(\delta_{\mathrm{sem}})
    \qquad
    \text{for every }k\in\{1,\ldots,C\}.
    \label{eq:app_roundwise_feasibility}
\end{equation}

Given this round-wise feasibility, we define the class-wise search gap as:
\begin{equation}
\begin{aligned}
    \varepsilon_{{\rm adv},k}^{(r)}
    ={}&
    \sup_{P\in
    \mathcal{B}_{\mathrm{sem}}^k(\delta_{\mathrm{sem}})}
    \mathcal{R}_k
    \left(
        \mathbf{f}_{\boldsymbol{\theta}^{(r-1)}};P
    \right)
    \\
    &-
    \mathcal{R}_k
    \left(
        \mathbf{f}_{\boldsymbol{\theta}^{(r-1)}};
        P_{\boldsymbol{\phi}^{(r)}}^k
    \right).
\end{aligned}
\label{eq:app_classwise_search_gap}
\end{equation}
Since \(P_{\boldsymbol{\phi}^{(r)}}^k\) belongs to the ambiguity set,
\(\varepsilon_{{\rm adv},k}^{(r)}\geq0\).

The only additional assumption here is the round-wise feasibility condition in Eq.~\ref{eq:app_roundwise_parameter_feasibility}, which requires the
shared LoRA parameter obtained at round \(r\) to satisfy the semantic
constraint for every class. Under this condition,
Eq.~\ref{eq:app_classwise_search_gap} simply defines the approximation error
of the CADS search. Specifically,
\(\varepsilon_{{\rm adv},k}^{(r)}\) is the difference between the largest
class-\(k\) risk allowed by the semantic ambiguity set and the risk achieved
by the distribution generated at round \(r\). Therefore,
\(\varepsilon_{{\rm adv},k}^{(r)}=0\) means that CADS reaches the class-wise
worst case, whereas a positive value indicates that it does not fully reach
this worst case. Such a gap can occur because the same LoRA parameter is
optimized for all classes and may not maximize every class-wise risk
simultaneously. It may also result from the finite-step and non-convex
optimization used in CADS.

\vspace{2.0mm}

\noindent\textbf{A4. Conditional class-balanced sampling.}
At synthesis round \(r\), the optimized generator parameter \(\boldsymbol{\phi}^{(r)}\) is held fixed while constructing the generated dataset \(\mathcal{D}_g^{(r)}\). The finite-sample analysis below therefore focuses on the randomness introduced by sampling from the resulting class-conditional distributions.

For each class \(k\), we independently draw \(M_g\) latent samples from the
generator-induced distribution \(P_{\boldsymbol{\phi}^{(r)}}^k\):
\begin{equation}
    \mathbf{z}_{k,j}^{(r)}
    \mid
    \boldsymbol{\phi}^{(r)}
    \overset{\mathrm{ind}}{\sim}
    P_{\boldsymbol{\phi}^{(r)}}^k,
    \qquad
    k=1,\ldots,C,\quad
    j=1,\ldots,M_g.
    \label{eq:app_conditional_sampling}
\end{equation}
The resulting indexed class-balanced latent-label sample is
\begin{equation}
    \mathcal{S}_r
    =
    \left\{
        \left(
            \mathbf{z}_{k,j}^{(r)},k
        \right)
        :
        k=1,\ldots,C,\;
        j=1,\ldots,M_g
    \right\}.
    \label{eq:app_generated_latent_sample}
\end{equation}
The corresponding image-space samples are obtained by decoding the latent
samples:
\[
    \tilde{\mathbf{x}}_{k,j}^{(r)}
    =
    \mathcal{F}_D
    \left(
        \mathbf{z}_{k,j}^{(r)}
    \right),
    \qquad
    k=1,\ldots,C,\quad
    j=1,\ldots,M_g,
\]
as in Eq.~\ref{eq:generated_dataset}. Drawing \(M_g\) samples for every class yields the class-balanced sampling scheme used by CADS.
The class-balanced risk induced by the generator is:
\begin{equation}
    \mathcal{R}_g^{(r)}
    (\mathbf{f}_{\boldsymbol{\theta}})
    =
    \frac{1}{C}
    \sum_{k=1}^{C}
    \mathcal{R}_k
    \left(
        \mathbf{f}_{\boldsymbol{\theta}};
        P_{\boldsymbol{\phi}^{(r)}}^k
    \right),
    \label{eq:app_generated_population_risk}
\end{equation}
and its empirical estimate is given by:
\begin{equation}
    \widehat{\mathcal{R}}_g^{(r)}
    (\mathbf{f}_{\boldsymbol{\theta}})
    =
    \frac{1}{CM_g}
    \sum_{k=1}^{C}
    \sum_{j=1}^{M_g}
    \ell_{\mathrm{CE}}
    \left(
        \mathbf{f}_{\boldsymbol{\theta}}
        (\tilde{\mathbf{x}}_{k,j}^{(r)}),
        k
    \right).
    \label{eq:app_generated_empirical_risk}
\end{equation}
For every fixed
\(\mathbf{f}_{\boldsymbol{\theta}}\in\mathcal{H}\),
the empirical risk is conditionally unbiased:
\begin{equation}
    \mathbb{E}
    \left[
        \widehat{\mathcal{R}}_g^{(r)}
        (\mathbf{f}_{\boldsymbol{\theta}})
        \,\middle|\,
        \boldsymbol{\phi}^{(r)}
    \right]
    =
    \mathcal{R}_g^{(r)}
    (\mathbf{f}_{\boldsymbol{\theta}}).
    \label{eq:app_generated_risk_unbiased}
\end{equation}
Equation~\ref{eq:app_generated_empirical_risk} is precisely the generated-data
term in Eq.~\ref{eq:classifier_objective} before multiplication by \(\mu\).

Finally, define the induced loss class:
\begin{equation}
    \ell_{\mathrm{CE}}\circ\mathcal{H}
    =
    \left\{
        (\mathbf{z},k)
        \mapsto
        \ell_{\mathrm{CE}}
        \left(
            \mathbf{f}_{\boldsymbol{\theta}}
            (\mathcal{F}_D(\mathbf{z})),
            k
        \right)
        :
        \mathbf{f}_{\boldsymbol{\theta}}\in\mathcal{H}
    \right\}.
    \label{eq:app_loss_class}
\end{equation}
Under the conditionally independent class-balanced sampling scheme in
Eq.~\ref{eq:app_conditional_sampling}, the expected Rademacher complexity
of the induced loss class is defined as:
\begin{equation}
\begin{aligned}
    &\mathfrak{R}_{CM_g}^{(r)} 
    (\ell_{\mathrm{CE}}\circ\mathcal{H}) \\
    &=
    \mathbb{E}_{\mathcal{S}_r,\boldsymbol{\sigma}}
    \left[
        \sup_{\mathbf{f}_{\boldsymbol{\theta}}\in\mathcal{H}}
        \left|
        \begin{aligned}
            &\frac{1}{CM_g}
            \sum_{k=1}^{C}
            \sum_{j=1}^{M_g}
            \sigma_{k,j}
            \\
            &\qquad {}\times
            \ell_{\mathrm{CE}}
            \left(
                \mathbf{f}_{\boldsymbol{\theta}}
                (\mathcal{F}_D(\mathbf{z}_{k,j}^{(r)})),
                k
            \right)
        \end{aligned}
        \right|
    \right].
\end{aligned}
\label{eq:app_rademacher_complexity}
\end{equation}
where the expectation over \(\mathcal{S}_r\) is taken according to the
conditionally independent sampling scheme in
Eq.~\ref{eq:app_conditional_sampling}, and the
\(\sigma_{k,j}\) are independent Rademacher random variables that are
independent of \(\mathcal{S}_r\).
\subsection{Proof of Theorem~\ref{thm:semantic_consistency}}

We first state a technical lemma that upper-bounds the trajectory-level KL divergence in terms of the semantic-reference denoising loss.

\begin{lemma}[Control of trajectory-level KL divergence by the denoising loss]
\label{lem:trajectory_kl_control}
Under Condition A1, for each class \(k\), there exist constants
\(\gamma_k>0\) and \(C_k\), both independent of \(\boldsymbol{\phi}\), such that:
\begin{equation}
\mathcal{K}_{\mathrm{traj}}^k(\boldsymbol{\phi})
\le
\gamma_k
\mathcal{L}_{\mathrm{den}}^k
\left(
\boldsymbol{\phi};\mathcal{A}_k
\right)
+
C_k.
\label{eq:app_trajectory_kl_control}
\end{equation}
\end{lemma}

\begin{IEEEproof}[\textbf{Proof of Lemma~\ref{lem:trajectory_kl_control}}]
For a fixed class \(k\), let \(q_A^k(\mathbf z_{0:T})\) and \(p_{\boldsymbol{\phi}}^k(\mathbf z_{0:T})\) denote the reference forward and parameterized reverse trajectory distributions, respectively.

\vspace{3.0mm}

Using the reverse-time factorization of the conditional forward process, we have:
\begin{equation}
\begin{aligned}
\mathcal K_{\mathrm{traj}}^k(\boldsymbol{\phi})
&={}
-h(P_A^k)
\\
&+
\mathbb E_{\mathbf z_0^a\sim P_A^k}
\Bigg[
D_{\mathrm{KL}}
\Bigg(
q(\mathbf z_T\mid\mathbf z_0^a)
\,\big\|\,
p_{\mathrm{prior}}
\Bigg)
\Bigg]
\\
&+
\mathbb E_{q_A^k(\mathbf z_0,\mathbf z_1)}
\Big[
-\log
p_\phi(
\mathbf z_0
\mid
\mathbf z_1,\mathbf T_{k,\mathrm{ref}})
\Big]
\\
&+
\sum_{t=2}^{T}
\mathbb E_{q_A^k(\mathbf z_0,\mathbf z_t)}
\Bigg[
D_{\mathrm{KL}}
\Bigg(
q(\mathbf z_{t-1}\mid\mathbf z_t,\mathbf z_0)
\\
&\qquad\qquad\qquad\quad
\mathrel{\big\|}
p_\phi(
\mathbf z_{t-1}
\mid
\mathbf z_t,\mathbf T_{k,\mathrm{ref}})
\Bigg)
\Bigg].
\label{eq:traj_kl_decomp}
\end{aligned}
\end{equation}
Here, \(h(P_A^k)\) denotes the differential entropy of the absolutely continuous reference latent distribution \(P_A^k\). Under Condition A1, the entropy term and the expected terminal conditional KL term are finite and independent of the trainable LoRA parameters \(\boldsymbol{\phi}\). The third term is the reconstruction negative log-likelihood at \(t=1\), and the summation contains the transition KL divergences for \(t=2,\ldots,T\).

\vspace{3.0mm}

We first consider the transition terms for
\(t=2,\ldots,T\).
For the forward Gaussian diffusion process,
\(q(\mathbf z_{t-1}\mid\mathbf z_t,\mathbf z_0)\)
is a Gaussian conditional distribution.
Under the fixed reverse variance schedule, the reverse transition
\(p_\phi(\mathbf z_{t-1}\mid
\mathbf z_t,\mathbf T_{k,\mathrm{ref}})\)
uses the same covariance as the corresponding forward diffusion
posterior. 
Thus, for each \(t=2,\ldots,T\),
\begin{align}
q(\mathbf z_{t-1}\mid\mathbf z_t,\mathbf z_0)
&=
\mathcal N
\left(
\widetilde{\boldsymbol{\mu}}_t,
\widetilde{\beta}_t\mathbf I
\right),
\\
p_{\boldsymbol{\phi}}
(\mathbf z_{t-1}\mid
\mathbf z_t,\mathbf T_{k,\mathrm{ref}})
&=
\mathcal N
\left(
\boldsymbol{\mu}_{\boldsymbol{\phi},t},
\widetilde{\beta}_t\mathbf I
\right).
\end{align}
Here,
\(\widetilde{\boldsymbol{\mu}}_t\) and
\(\boldsymbol{\mu}_{\boldsymbol{\phi},t}\)
are the corresponding mean vectors.
Because the two Gaussian distributions use the same covariance, their KL divergence is:
\begin{equation}
\begin{aligned}
D_{\mathrm{KL}}
\left(
q(\mathbf z_{t-1}\mid \mathbf z_t,\mathbf z_0)
\,\middle\|\,
p_\phi(\mathbf z_{t-1}\mid \mathbf z_t,\mathbf T_{k,\mathrm{ref}})
\right) \\
=
\frac{1}{2\widetilde{\beta}_t}
\left\|
\widetilde{\boldsymbol{\mu}}_t
-
\boldsymbol{\mu}_{\boldsymbol{\phi},t}
\right\|_2^2 .
\label{eq:gaussian_kl_same_cov}
\end{aligned}
\end{equation}

Let $\alpha_t$ denote the one-step noise schedule coefficient, so that
$\bar\alpha_t=\prod_{s=1}^{t}\alpha_s$, and let $\beta_t=1-\alpha_t$.
Using the noise-prediction parameterization in Eq.~\ref{eq:forward_diffusion_noising} and~\ref{eq:noise_prediction}, together with the reverse transition in Eq.~\ref{eq:app_reverse_trajectory}, the reverse mean can be expressed as:
\begin{equation}
\boldsymbol\mu_\phi
=
\frac{1}{\sqrt{\alpha_t}}
\left(
\mathbf z_t
-
\frac{\beta_t}{\sqrt{1-\bar\alpha_t}}
\epsilon_\phi(\mathbf z_t,t,\mathbf T_{k,\mathrm{ref}})
\right).
\end{equation}
Using the forward noising relation:
\begin{equation}
\mathbf z_t
=
\sqrt{\bar\alpha_t}\mathbf z_0
+
\sqrt{1-\bar\alpha_t}\boldsymbol\epsilon,
\end{equation}
where
\(\boldsymbol\epsilon\sim\mathcal N(\mathbf 0,\mathbf I)\),
the mean of the forward diffusion posterior can be written as:
\begin{equation}
\boldsymbol\mu_q
=
\frac{1}{\sqrt{\alpha_t}}
\left(
\mathbf z_t
-
\frac{\beta_t}{\sqrt{1-\bar\alpha_t}}
\boldsymbol\epsilon
\right).
\end{equation}
Therefore,
\begin{equation}
\boldsymbol\mu_q-\boldsymbol\mu_\phi
=
\frac{\beta_t}{\sqrt{\alpha_t}\sqrt{1-\bar\alpha_t}}
\left(
\epsilon_\phi(\mathbf z_t,t,\mathbf T_{k,\mathrm{ref}})
-
\boldsymbol\epsilon
\right).
\end{equation}
Substituting this relation into Eq.~\ref{eq:gaussian_kl_same_cov}, we obtain:
\begin{align}
&D_{\mathrm{KL}}
\left(
q(\mathbf z_{t-1}\mid \mathbf z_t,\mathbf z_0)
\,\middle\|\,
p_\phi(\mathbf z_{t-1}\mid \mathbf z_t,\mathbf T_{k,\mathrm{ref}})
\right)
\nonumber\\
&\qquad=
\frac{\beta_t^2}
{2\widetilde{\beta}_t\alpha_t(1-\bar\alpha_t)}
\left\|
\boldsymbol\epsilon
-
\epsilon_\phi(\mathbf z_t,t,\mathbf T_{k,\mathrm{ref}})
\right\|_2^2 .
\end{align}
Thus, by defining:
\begin{equation}
a_t
=
\frac{\beta_t^2}
{2\widetilde{\beta}_t
\alpha_t(1-\bar\alpha_t)},
\qquad
t=2,\ldots,T.
\end{equation}
we have:
\begin{align}
&D_{\mathrm{KL}}
\left(
q(\mathbf z_{t-1}\mid \mathbf z_t,\mathbf z_0)
\,\middle\|\,
p_\phi(\mathbf z_{t-1}\mid \mathbf z_t,\mathbf T_{k,\mathrm{ref}})
\right)
\nonumber\\
&\qquad=
a_t
\left\|
\boldsymbol\epsilon
-
\epsilon_\phi(\mathbf z_t,t,\mathbf T_{k,\mathrm{ref}})
\right\|_2^2 .
\label{eq:transition_kl_mse}
\end{align}
Here, \(a_t>0\) depends only on the forward diffusion schedule and the fixed reverse variance schedule, and is independent of \(\phi\).

\vspace{3.0mm}

For \(t=1\), the conditional distribution
\(q(\mathbf z_0\mid\mathbf z_1,\mathbf z_0)\)
is degenerate; hence, the transition-KL identity in
Eq.~\ref{eq:transition_kl_mse} does not apply.
The corresponding term in Eq.~\ref{eq:traj_kl_decomp}
is instead the reconstruction negative log-likelihood.

\vspace{3.0mm}

Under the fixed-covariance Gaussian reconstruction kernel and the same noise-prediction mean parameterization, we have:
\begin{equation}
\begin{aligned}
&\mathbf z_0^a
-
\boldsymbol\mu_\phi(
\mathbf z_1^a,1,\mathbf T_{k,\mathrm{ref}})
\\
&\qquad=
\frac{\beta_1}
{\sqrt{\alpha_1}\sqrt{1-\bar\alpha_1}}
\left(
\epsilon_\phi(
\mathbf z_1^a,1,\mathbf T_{k,\mathrm{ref}})
-
\boldsymbol\epsilon
\right).
\end{aligned}
\end{equation}
Consequently, the reconstruction negative log-likelihood satisfies:
\begin{equation}
\begin{aligned}
&\mathbb E_{q_A^k(\mathbf z_0,\mathbf z_1)}
\Big[
-\log
p_\phi(
\mathbf z_0
\mid
\mathbf z_1,\mathbf T_{k,\mathrm{ref}})
\Big]
\\
&\qquad=
b_1
+
a_1
\mathbb E_{\mathbf z_0^a\sim P_A^k,\boldsymbol\epsilon}
\Bigg[
\left\|
\boldsymbol\epsilon
-
\epsilon_\phi(
\mathbf z_1^a,1,\mathbf T_{k,\mathrm{ref}})
\right\|_2^2
\Bigg],
\label{eq:reconstruction_nll_mse}
\end{aligned}
\end{equation}
where
\begin{equation}
a_1
=
\frac{\beta_1^2}
{2\sigma_1^2\alpha_1(1-\bar\alpha_1)}
\end{equation}
and \(b_1\) is the normalization constant of the fixed
Gaussian reconstruction kernel. Both \(a_1\) and \(b_1\)
are independent of \(\phi\).

\vspace{3.0mm}

Combining Eq.~\ref{eq:traj_kl_decomp},
Eq.~\ref{eq:transition_kl_mse}, and
Eq.~\ref{eq:reconstruction_nll_mse}, we obtain:
\begin{equation}
\begin{aligned}
\mathcal K_{\mathrm{traj}}^k(\phi)
={}&
C_k
\\
&+
\sum_{t=1}^{T}
a_t
\mathbb E_{\mathbf z_0^a\sim P_A^k,\boldsymbol\epsilon}
\Bigg[
\left\|
\boldsymbol\epsilon
-
\epsilon_\phi(
\mathbf z_t^a,t,\mathbf T_{k,\mathrm{ref}})
\right\|_2^2
\Bigg],
\label{eq:traj_kl_after_transition_bound}
\end{aligned}
\end{equation}
where
\begin{equation}
\begin{aligned}
C_k
={}&
-h(P_A^k)
+
b_1
\\
&+
\mathbb E_{\mathbf z_0^a\sim P_A^k}
\Bigg[
D_{\mathrm{KL}}
\Bigg(
q(\mathbf z_T\mid\mathbf z_0^a)
\,\big\|\,
p_{\mathrm{prior}}
\Bigg)
\Bigg].
\end{aligned}
\end{equation}
By the conditions of Lemma~\ref{lem:trajectory_kl_control},
\(C_k\) is finite and independent of \(\phi\).

\vspace{3.0mm}

Since the timestep distribution \(\nu\) satisfies
\(\nu(t)>0\) for every \(t=1,\ldots,T\), we have
\begin{equation}
\begin{aligned}
&\sum_{t=1}^{T}
a_t
\mathbb E_{\mathbf z_0^a\sim P_A^k,\boldsymbol\epsilon}
\Bigg[
\left\|
\boldsymbol\epsilon
-
\epsilon_\phi(
\mathbf z_t^a,t,\mathbf T_{k,\mathrm{ref}})
\right\|_2^2
\Bigg]
\\
&=
\sum_{t=1}^{T}
\frac{a_t}{\nu(t)}
\nu(t)
\mathbb E_{\mathbf z_0^a\sim P_A^k,\boldsymbol\epsilon}
\Bigg[
\left\|
\boldsymbol\epsilon
-
\epsilon_\phi(
\mathbf z_t^a,t,\mathbf T_{k,\mathrm{ref}})
\right\|_2^2
\Bigg]
\\
&\le
\gamma_k
\sum_{t=1}^{T}
\nu(t)
\mathbb E_{\mathbf z_0^a\sim P_A^k,\boldsymbol\epsilon}
\Bigg[
\left\|
\boldsymbol\epsilon
-
\epsilon_\phi(
\mathbf z_t^a,t,\mathbf T_{k,\mathrm{ref}})
\right\|_2^2
\Bigg],
\end{aligned}
\end{equation}
where
\begin{equation}
\gamma_k
=
\max_{1\le t\le T}
\frac{a_t}{\nu(t)}
>0.
\end{equation}
By the definition of the class-wise denoising loss in
Eq.~\ref{eq:app_classwise_denoising_loss}, it follows that
\begin{equation}
\mathcal K_{\mathrm{traj}}^k(\phi)
\le
\gamma_k
\mathcal L_{\mathrm{den}}^k(\phi;\mathcal A_k)
+
C_k.
\label{eq:trajectory_kl_denoising_bound}
\end{equation}
This completes the proof.
\end{IEEEproof}

\begin{IEEEproof}[\textbf{Proof of Theorem~\ref{thm:semantic_consistency}}]
Fix a class \(k\) and a parameter
\(\boldsymbol{\phi}\).
Since \(P_A^k\) and
\(P_{\boldsymbol{\phi}}^k\) are the
\(\mathbf z_0\)-marginals of
\(q_A^k(\mathbf z_{0:T})\) and
\(p_{\boldsymbol{\phi}}^k(\mathbf z_{0:T})\), respectively,
the data-processing inequality for KL divergence gives:
\begin{equation}
D_{\mathrm{KL}}
\left(
P_A^k
\middle\|
P_{\boldsymbol{\phi}}^k
\right)
\le
\mathcal K_{\mathrm{traj}}^k
(\boldsymbol{\phi}).
\end{equation}
Applying Lemma~\ref{lem:trajectory_kl_control}, we obtain:
\begin{equation}
D_{\mathrm{KL}}
\left(
P_A^k
\middle\|
P_{\boldsymbol{\phi}}^k
\right)
\le
\gamma_k
\mathcal L_{\mathrm{den}}^k
\left(
\boldsymbol{\phi};\mathcal A_k
\right)
+
C_k.
\end{equation}
This proves Eq.~\ref{eq:first_statement}.

\vspace{3.0mm}

For Eq.~\ref{eq:second_statement}, take any:
\[
P\in
\mathcal B_{\mathrm{sem}}^k
(\delta_{\mathrm{sem}}).
\]
By Eq.~\ref{eq:app_semantic_parameter_set}
and~\ref{eq:app_semantic_ambiguity_set}, there exists:
\[
\widetilde{\boldsymbol{\phi}}
\in
\Phi_{\mathrm{sem}}^k
(\delta_{\mathrm{sem}})
\]
such that:
\begin{equation}
P
=
P_{\widetilde{\boldsymbol{\phi}}}^k,
\qquad
\mathcal L_{\mathrm{den}}^k
\left(
\widetilde{\boldsymbol{\phi}};
\mathcal A_k
\right)
\le
\delta_{\mathrm{sem}}.
\end{equation}
Applying the preceding bound to
\(\widetilde{\boldsymbol{\phi}}\) gives:
\begin{equation}
\begin{aligned}
D_{\mathrm{KL}}
\left(
P_A^k
\middle\|
P
\right)
&=
D_{\mathrm{KL}}
\left(
P_A^k
\middle\|
P_{\widetilde{\boldsymbol{\phi}}}^k
\right)
\\
&\le
\gamma_k
\mathcal L_{\mathrm{den}}^k
\left(
\widetilde{\boldsymbol{\phi}};
\mathcal A_k
\right)
+
C_k
\\
&\le
\gamma_k\delta_{\mathrm{sem}}+C_k.
\end{aligned}
\end{equation}
Thus, Eq.~\ref{eq:second_statement} follows.
\end{IEEEproof}

\subsection{Proof of Theorem~\ref{thm:target_risk_bound}}

\begin{IEEEproof}[\textbf{Proof of Theorem~\ref{thm:target_risk_bound}}]
Fix a classifier
\(\mathbf f_{\boldsymbol{\theta}}\in\mathcal H\)
and a class \(k\).
By the definition of \(\rho_k\) in
Eq.~\ref{eq:app_coverage_gap}, for every \(\eta>0\), there exists:
\[
P_{k,\eta}
\in
\mathcal B_{\mathrm{sem}}^k(\delta_{\mathrm{sem}})
\]
such that:
\begin{equation}
d_{\mathrm{TV}}(P_T^k,P_{k,\eta})
\le
\rho_k+\eta.
\end{equation}
Applying Eq.~\ref{eq:app_tv_risk_difference} gives:
\begin{equation}
\begin{aligned}
\mathcal R_k
(\mathbf f_{\boldsymbol{\theta}};P_T^k)
&\le
\mathcal R_k
(\mathbf f_{\boldsymbol{\theta}};P_{k,\eta})
+
L_{\max}
d_{\mathrm{TV}}(P_T^k,P_{k,\eta})
\\
&\le
\sup_{P\in
\mathcal B_{\mathrm{sem}}^k(\delta_{\mathrm{sem}})}
\mathcal R_k
(\mathbf f_{\boldsymbol{\theta}};P)
+
L_{\max}(\rho_k+\eta).
\end{aligned}
\end{equation}
Letting \(\eta\downarrow0\), we obtain:
\begin{equation}
\mathcal R_k
(\mathbf f_{\boldsymbol{\theta}};P_T^k)
\le
\sup_{P\in
\mathcal B_{\mathrm{sem}}^k(\delta_{\mathrm{sem}})}
\mathcal R_k
(\mathbf f_{\boldsymbol{\theta}};P)
+
L_{\max}\rho_k.
\label{eq:class_risk_final_bound}
\end{equation}
Finally, applying Eq.~\ref{eq:class_risk_final_bound} to each class and using Eq.~\ref{eq:app_target_risk}, we have:
\begin{equation}
\begin{aligned}
\mathcal R_T(\mathbf f_{\boldsymbol{\theta}})
&=
\sum_{k=1}^{C}
\pi_T^k
\mathcal R_k
(\mathbf f_{\boldsymbol{\theta}};P_T^k)
\\
&\le
\sum_{k=1}^{C}
\pi_T^k
\sup_{P\in
\mathcal B_{\mathrm{sem}}^k(\delta_{\mathrm{sem}})}
\mathcal R_k
(\mathbf f_{\boldsymbol{\theta}};P)
+
L_{\max}
\sum_{k=1}^{C}
\pi_T^k\rho_k.
\end{aligned}
\end{equation}
\end{IEEEproof}

\subsection{Proof of Proposition~\ref{prop:semantic_worst_to_generated}}

\begin{IEEEproof}[\textbf{Proof of Proposition~\ref{prop:semantic_worst_to_generated}}]
Fix a synthesis round \(r\).
By the definition of the class-wise search gap in
Eq.~\ref{eq:app_classwise_search_gap}, for each class \(k\),
\begin{equation}
\begin{aligned}
&\sup_{P\in
\mathcal B_{\mathrm{sem}}^k(\delta_{\mathrm{sem}})}
\mathcal R_k
\left(
\mathbf f_{\boldsymbol{\theta}^{(r-1)}};P
\right)
\\
&\quad=
\mathcal R_k
\left(
\mathbf f_{\boldsymbol{\theta}^{(r-1)}};
P_{\boldsymbol{\phi}^{(r)}}^k
\right)
+
\varepsilon_{\mathrm{adv},k}^{(r)}.
\end{aligned}
\label{eq:app_classwise_adv_gap}
\end{equation}
Let \(\{\pi^k\}_{k=1}^{C}\) be nonnegative class weights
satisfying \(\sum_{k=1}^{C}\pi^k=1\).
Multiplying Eq.~\ref{eq:app_classwise_adv_gap} by
\(\pi^k\) and summing over \(k\) gives:
\begin{equation}
\begin{aligned}
&\sum_{k=1}^{C}\pi^k
\sup_{P\in
\mathcal B_{\mathrm{sem}}^k(\delta_{\mathrm{sem}})}
\mathcal R_k
\left(
\mathbf f_{\boldsymbol{\theta}^{(r-1)}};P
\right)
\\
&\quad=
\sum_{k=1}^{C}\pi^k
\mathcal R_k
\left(
\mathbf f_{\boldsymbol{\theta}^{(r-1)}};
P_{\boldsymbol{\phi}^{(r)}}^k
\right)
+
\sum_{k=1}^{C}\pi^k
\varepsilon_{\mathrm{adv},k}^{(r)}.
\end{aligned}
\label{eq:app_weighted_adv_gap}
\end{equation}
This proves Eq.~\ref{eq:weighted_worst_to_generated}.

\vspace{3.0mm}

Taking \(\pi^k=\pi_T^k\) in
Eq.~\ref{eq:app_weighted_adv_gap} and applying
Theorem~\ref{thm:target_risk_bound}, we obtain:
\begin{equation}
\begin{aligned}
\mathcal R_T
\left(
\mathbf f_{\boldsymbol{\theta}^{(r-1)}}
\right)
&\le
\sum_{k=1}^{C}\pi_T^k
\mathcal R_k
\left(
\mathbf f_{\boldsymbol{\theta}^{(r-1)}};
P_{\boldsymbol{\phi}^{(r)}}^k
\right)
\\
&\quad+
\sum_{k=1}^{C}\pi_T^k
\varepsilon_{\mathrm{adv},k}^{(r)}
+
L_{\max}
\sum_{k=1}^{C}\pi_T^k\rho_k.
\end{aligned}
\end{equation}
Thus, Eq.~\ref{eq:target_bound_generated_distribution} follows.

\vspace{3.0mm}

Finally, setting \(\pi^k=1/C\) in
Eq.~\ref{eq:app_weighted_adv_gap} and using the definitions of
\(\mathcal R_g^{(r)}\) and
\(\varepsilon_{\mathrm{adv}}^{(r)}\), we obtain:
\begin{equation}
\frac{1}{C}
\sum_{k=1}^{C}
\sup_{P\in
\mathcal B_{\mathrm{sem}}^k(\delta_{\mathrm{sem}})}
\mathcal R_k
\left(
\mathbf f_{\boldsymbol{\theta}^{(r-1)}};P
\right)
=
\mathcal R_g^{(r)}
\left(
\mathbf f_{\boldsymbol{\theta}^{(r-1)}}
\right)
+
\varepsilon_{\mathrm{adv}}^{(r)}.
\end{equation}
This proves Eq.~\ref{eq:balanced_worst_to_generated}.
\end{IEEEproof}

\subsection{Proof of Proposition~\ref{prop:generated_risk_gap}}

\begin{IEEEproof}[ \textbf{Proof of Proposition~\ref{prop:generated_risk_gap}}]
Fix a synthesis round \(r\).
By the round-wise feasibility condition in
Eq.~\ref{eq:app_roundwise_feasibility}, for every class \(k\),
\[
P_{\boldsymbol{\phi}^{(r)}}^k
\in
\mathcal B_{\mathrm{sem}}^k(\delta_{\mathrm{sem}}).
\]
Therefore, for any
\(\mathbf f_{\boldsymbol{\theta}}\in\mathcal H\),
\begin{equation}
\begin{aligned}
\mathcal R_g^{(r)}
(\mathbf f_{\boldsymbol{\theta}})
&=
\frac{1}{C}
\sum_{k=1}^{C}
\mathcal R_k
\left(
\mathbf f_{\boldsymbol{\theta}};
P_{\boldsymbol{\phi}^{(r)}}^k
\right)
\\
&\le
\frac{1}{C}
\sum_{k=1}^{C}
\sup_{P\in
\mathcal B_{\mathrm{sem}}^k(\delta_{\mathrm{sem}})}
\mathcal R_k
\left(
\mathbf f_{\boldsymbol{\theta}};P
\right).
\end{aligned}
\end{equation}
Thus, Eq.~\ref{eq:generated_risk_semantic_bound} follows.

\vspace{3.0mm}

For the finite-sample claim, condition on
\(\boldsymbol{\phi}^{(r)}\), and let
\(\mathcal S_r\) be the class-balanced latent-label sample defined in Eq.~\ref{eq:app_generated_latent_sample}. 
By Eq.~\ref{eq:app_conditional_sampling}, the elements of
\(\mathcal S_r\) are independent conditional on
\(\boldsymbol{\phi}^{(r)}\). All subsequent probabilities and expectations are taken with respect to this conditional sampling distribution.

\vspace{3.0mm}

Let
\[
\mathcal G
=
\ell_{\mathrm{CE}}\circ\mathcal H
\]
be the loss class defined in Eq.~\ref{eq:app_loss_class}, and define
\begin{equation}
\Phi(\mathcal S_r)
=
\sup_{\mathbf f_{\boldsymbol{\theta}}\in\mathcal H}
\left|
\mathcal R_g^{(r)}
(\mathbf f_{\boldsymbol{\theta}})
-
\widehat{\mathcal R}_g^{(r)}
(\mathbf f_{\boldsymbol{\theta}})
\right|.
\end{equation}
Equivalently, by
Eq.~\ref{eq:app_generated_population_risk}
and~\ref{eq:app_generated_empirical_risk},
\begin{equation}
\begin{aligned}
\Phi(\mathcal S_r)
=
\sup_{h\in\mathcal G}
\Bigg|
&
\frac{1}{C}
\sum_{k=1}^{C}
\mathbb E_{\mathbf z\sim
P_{\boldsymbol{\phi}^{(r)}}^k}
\left[
h(\mathbf z,k)
\right]
\\
&-
\frac{1}{CM_g}
\sum_{k=1}^{C}
\sum_{j=1}^{M_g}
h(\mathbf z_{k,j}^{(r)},k)
\Bigg|.
\end{aligned}
\end{equation}
By Condition A2 and
Eq.~\ref{eq:app_roundwise_feasibility}, every
\(h\in\mathcal G\) takes values in
\([0,L_{\max}]\) almost surely under each generated
class-conditional distribution.
Consequently, replacing one element of \(\mathcal S_r\) can change \(\Phi(\mathcal S_r)\) by at most \(L_{\max}/(CM_g)\).
McDiarmid's inequality therefore implies that, with probability at least \(1-\delta\) conditional on \(\boldsymbol{\phi}^{(r)}\),
\begin{equation}
\Phi(\mathcal S_r)
\le
\mathbb E_{\mathcal S_r}
\left[
\Phi(\mathcal S_r)
\right]
+
L_{\max}
\sqrt{
\frac{\log(2/\delta)}{2CM_g}
}.
\label{eq:mcdiarmid_generated_gap}
\end{equation}

Let \(\mathcal S_r'\) be a conditionally independent copy of
\(\mathcal S_r\), and let
\(\{\sigma_{k,j}\}\) be independent Rademacher variables.
The standard symmetrization argument yields:
\begin{equation}
\begin{aligned}
\mathbb E_{\mathcal S_r}
\left[
\Phi(\mathcal S_r)
\right]
&\le
\mathbb E_{\mathcal S_r,\mathcal S_r'}
\Bigg[
\sup_{h\in\mathcal G}
\Bigg|
\frac{1}{CM_g}
\sum_{k=1}^{C}
\sum_{j=1}^{M_g}
\Big[
h(\mathbf z_{k,j}^{\prime(r)},k) \\
&-
h(\mathbf z_{k,j}^{(r)},k)
\Big]
\Bigg|
\Bigg]
\\
&=
\mathbb E_{\mathcal S_r,\mathcal S_r',
\boldsymbol{\sigma}}
\Bigg[
\sup_{h\in\mathcal G}
\Bigg|
\frac{1}{CM_g}
\sum_{k=1}^{C}
\sum_{j=1}^{M_g}
\sigma_{k,j}
\Big[
h(\mathbf z_{k,j}^{\prime(r)},k) \\
&-
h(\mathbf z_{k,j}^{(r)},k)
\Big]
\Bigg|
\Bigg]
\\
&\le
2\mathfrak R_{CM_g}^{(r)}
(\ell_{\mathrm{CE}}\circ\mathcal H).
\end{aligned}
\label{eq:rademacher_generated_bound}
\end{equation}
Here, the equality follows from the coordinate-wise exchangeability
of each original--ghost sample pair, and the final inequality follows
from the triangle inequality and
Eq.~\ref{eq:app_rademacher_complexity}.

\vspace{3.0mm}

Combining Eq.~\ref{eq:mcdiarmid_generated_gap}
and~\ref{eq:rademacher_generated_bound}, we obtain, with conditional
probability at least \(1-\delta\),
\begin{equation}
\begin{aligned}
\sup_{\mathbf f_{\boldsymbol{\theta}}\in\mathcal H}
\left|
\mathcal R_g^{(r)}
(\mathbf f_{\boldsymbol{\theta}})
-
\widehat{\mathcal R}_g^{(r)}
(\mathbf f_{\boldsymbol{\theta}})
\right|
&\le
2\mathfrak R_{CM_g}^{(r)}
(\ell_{\mathrm{CE}}\circ\mathcal H)
\\
&\quad+
L_{\max}
\sqrt{
\frac{\log(2/\delta)}{2CM_g}
}.
\end{aligned}
\end{equation}
Consequently, on the same event, for every
\(\mathbf f_{\boldsymbol{\theta}}\in\mathcal H\),
\begin{equation}
\begin{aligned}
\mathcal R_g^{(r)}
(\mathbf f_{\boldsymbol{\theta}})
-
\widehat{\mathcal R}_g^{(r)}
(\mathbf f_{\boldsymbol{\theta}})
&\le
2\mathfrak R_{CM_g}^{(r)}
(\ell_{\mathrm{CE}}\circ\mathcal H)
\\
&\quad+
L_{\max}
\sqrt{
\frac{\log(2/\delta)}{2CM_g}
}.
\end{aligned}
\end{equation}
Thus, Eq.~\ref{eq:generated_risk_gap_bound} follows.
\end{IEEEproof}

\section{Further Analysis}

\begin{figure*}[!t]
\centering
\includegraphics[width=0.98\textwidth]{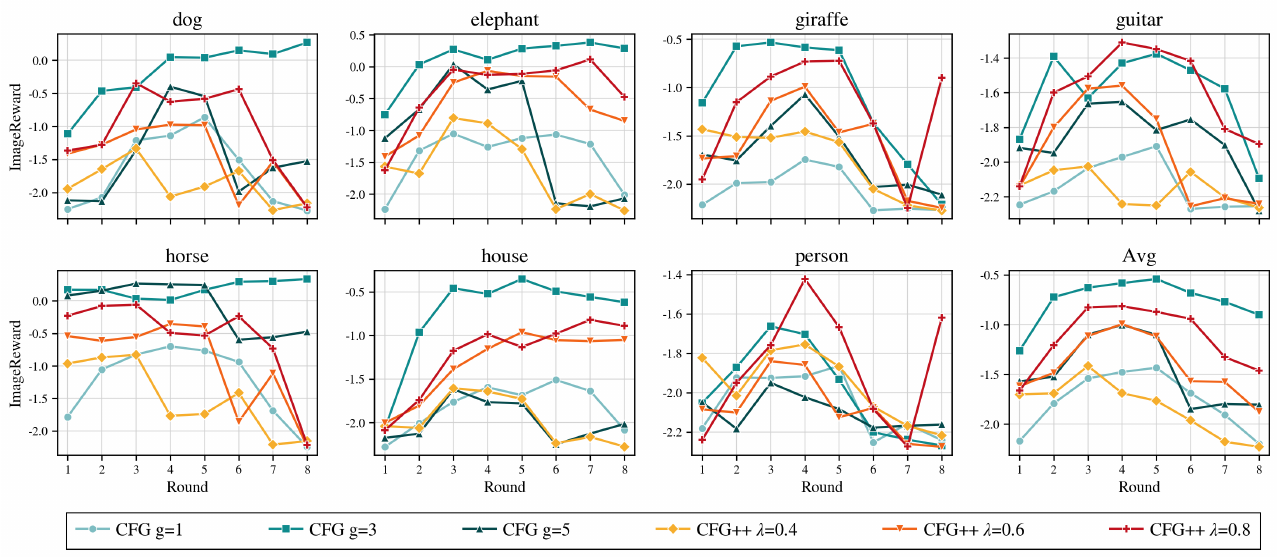}
\caption{
Class-wise and average ImageReward scores over CADS synthesis rounds on PACS under different CFG scales \(g\) and CFG++ coefficients \(\lambda\).
Each class panel shows the score trajectory for one class, whereas ``Avg'' shows the mean across the seven classes. Higher scores indicate stronger agreement with human preferences for prompt-conditioned image generation.
}
\label{fig:cads_round_imagereward_pacs}
\end{figure*}

\begin{figure*}[!t]
\centering
\includegraphics[width=0.98\textwidth]{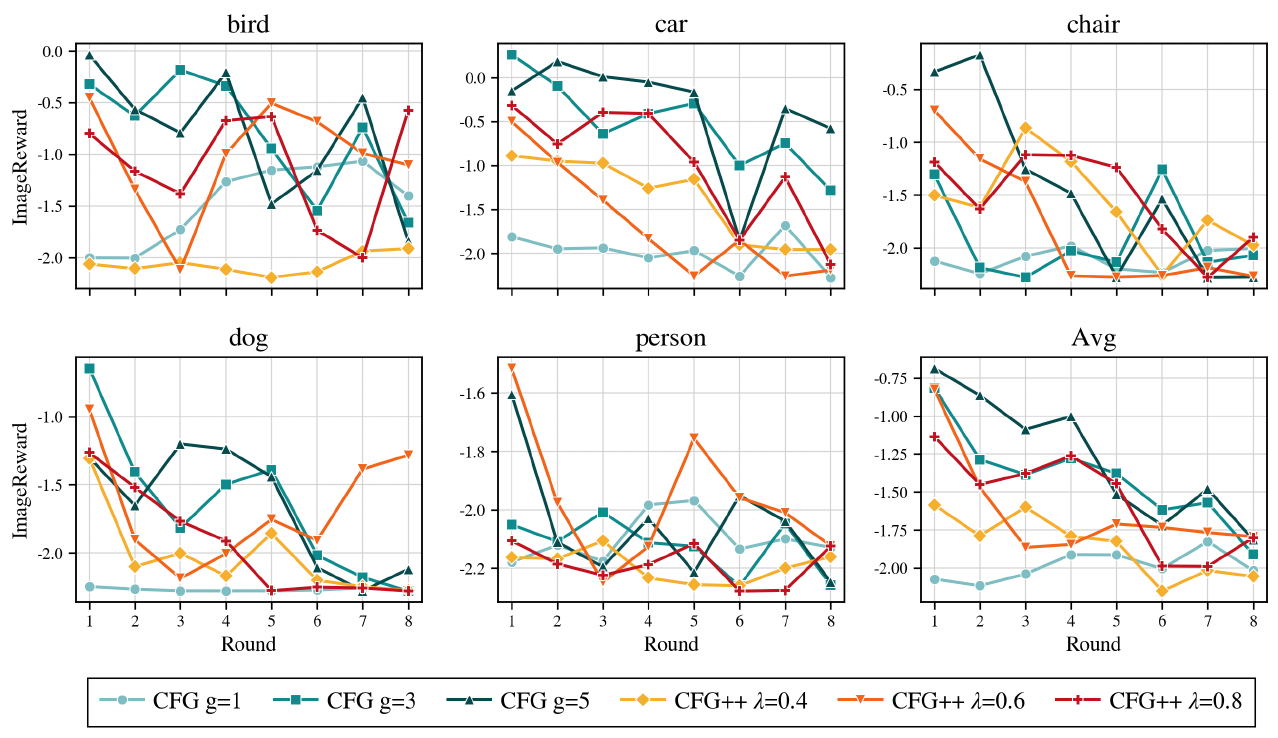}
\caption{
Class-wise and average ImageReward scores over CADS synthesis rounds
on VLCS under different CFG scales \(g\) and CFG++ coefficients
\(\lambda\).
Each class panel shows the score trajectory for one class, whereas
``Avg'' shows the mean across the five classes.
Higher scores indicate stronger agreement with human preferences for
prompt-conditioned image generation.
}
\label{fig:cads_round_imagereward}
\end{figure*}

\begin{table}[!t]
\centering
\caption{Impact of CFG/CFG++ on PACS (\%) with ResNet-18 under a high risk-guidance weight ($\lambda_{\mathrm{risk}}=0.5$). The best results are in \textcolor{red}{red}, and the second-highest results are in \textcolor{blue}{blue}.}
\label{tab:pacs_cfg_ablation}
\setlength{\tabcolsep}{5.8pt}
\renewcommand{\arraystretch}{1.0}
\begin{tabular}{c|l|cccc|c}
\toprule
Idx & Method & A & C & P & S & Avg. \\
\midrule
1 & CFG ($g=1$)
& 67.12
& 77.77
& 60.00
& 78.04
& 70.73 \\
2 & CFG ($g=3$)
& 68.70
& 77.42
& 60.19
& \textcolor{red}{79.30}
& 71.40 \\
3 & CFG ($g=5$)
& 71.43
& 77.95
& \textcolor{blue}{65.55}
& 74.68
& 72.40 \\
\midrule
4 & CFG++ ($\lambda=0.4$)
& 70.64
& \textcolor{red}{82.08}
& 62.56
& 77.32
& 73.15 \\
5 & CFG++ ($\lambda=0.6$)
& \textcolor{blue}{73.02}
& \textcolor{blue}{81.54}
& \textcolor{red}{67.79}
& \textcolor{blue}{78.88}
& \textcolor{red}{75.31} \\
6 & CFG++ ($\lambda=0.8$)
& \textcolor{red}{73.61}
& 78.02
& 64.56
& 77.15
& \textcolor{blue}{73.34} \\
\bottomrule
\end{tabular}
\end{table}

\noindent Fig.~\ref{fig:cads_round_imagereward_pacs}
and Fig.~\ref{fig:cads_round_imagereward} show the class-wise and
average ImageReward scores over CADS synthesis rounds on PACS and
VLCS, respectively. For this analysis, we use a high risk-guidance
weight of $\lambda_{\mathrm{risk}}=0.5$ to strengthen classifier
feedback and create a more challenging synthesis setting. The
ImageReward trajectories vary across classes and are often
non-monotonic because the classifier is updated after each round,
changing the variations that remain difficult for the current model.
A higher ImageReward does not necessarily lead to better domain
generalization performance. Strong CFG guidance often improves
ImageReward, especially in the early rounds, by encouraging the
generated images to closely follow their class prompts. However, it
may also keep the generated samples close to typical and relatively
easy class appearances, thereby limiting the effect of
classification-risk guidance. Conversely, insufficient semantic
guidance may allow risk maximization to produce visually degraded or
non-semantic images that are difficult for the classifier but provide
limited training value. CFG++ provides a better balance between these
two cases. In particular, $\lambda=0.6$ maintains adequate prompt
alignment and visual quality while allowing CADS to discover
challenging within-class variations without relying on non-semantic
artifacts to increase classification loss. As shown in
Tab.~\ref{tab:vlcs_cfg_ablation} and
Tab.~\ref{tab:pacs_cfg_ablation}, this setting achieves the highest
average accuracy among all evaluated CFG and CFG++ configurations on
both VLCS and PACS, outperforming the best CFG configuration by
0.82 and 2.91 pp, respectively. These results indicate that effective
guidance should preserve sufficient semantic quality without
suppressing the challenging variations targeted by CADS.

\section{Notation Summary}

\begin{table*}[!t]
\centering
\caption{Principal notation used in the theoretical analysis.}
\label{tab:theory_notation}
\footnotesize
\setlength{\tabcolsep}{4pt}
\renewcommand{\arraystretch}{1.15}

\begin{tabular}{p{0.25\textwidth} p{0.69\textwidth}}
\hline
\textbf{Symbol} & \textbf{Description} \\
\hline

\multicolumn{2}{l}{\textbf{Basic setting and diffusion notation}} \\

$C,\;k$
&
Number of classes and class index, respectively, with
$k=1,\ldots,C$.
\\

$T,\;t,\;r$
&
Total number of diffusion steps, diffusion timestep, and synthesis-round
index, respectively. The subscript $T$ in target-domain quantities such
as $P_T^k$ and $\pi_T^k$ denotes the target domain rather than a
diffusion timestep.
\\

$M_a,\;M_g$
&
Numbers of semantic reference images and generated high-risk samples
per class, respectively.
\\

$\mathcal Z,\;\mathcal F_E,\;\mathcal F_D$
&
Latent space and the frozen VAE encoder and decoder, respectively.
\\

$\mathbf z_0,\;\mathbf z_t,\;\boldsymbol{\epsilon}$
&
Clean latent, noisy latent at diffusion timestep $t$, and Gaussian noise
sampled from $\mathcal N(\mathbf 0,\mathbf I)$, respectively.
\\

$\mathbf T_{k,\mathrm{ref}}$
&
Text embedding of the class-level reference prompt for class $k$.
\\

$\boldsymbol{\phi},\;\Phi,\;
 \boldsymbol{\phi}^{(r)}$
&
Trainable self-attention LoRA parameters, their admissible parameter
space, and the optimized LoRA parameters at synthesis round $r$,
respectively.
\\

$\boldsymbol{\epsilon}_{\boldsymbol{\phi}}
(\mathbf z_t,t,\mathbf T_{k,\mathrm{ref}})$
&
Noise predictor of the LoRA-adapted diffusion model under the
class-level text condition.
\\

\hline
\multicolumn{2}{l}{\textbf{Reference distributions and semantic ambiguity sets}} \\

$\mathcal A_k$
&
Semantic reference image set for class $k$, containing $M_a$ reference
images.
\\

$\widehat P_A^k,\;P_A^k,\;\sigma_A$
&
Empirical distribution of the encoded reference latents, its
Gaussian-smoothed counterpart, and the corresponding smoothing
bandwidth, respectively.
\\

$q_A^k(\mathbf z_{0:T}),\;
 p_{\boldsymbol{\phi}}^k(\mathbf z_{0:T})$
&
Reference forward trajectory distribution initialized from $P_A^k$ and
class-conditional reverse generative trajectory distribution,
respectively.
\\

$P_{\boldsymbol{\phi}}^k,\;
 P_{\boldsymbol{\phi}^{(r)}}^k$
&
Clean-latent marginals generated by the reverse diffusion process under
parameters $\boldsymbol{\phi}$ and
$\boldsymbol{\phi}^{(r)}$, respectively.
\\

$\mathcal P_{\Phi}^k$
&
Generator-induced family of all class-$k$ latent distributions
attainable by varying $\boldsymbol{\phi}\in\Phi$.
\\

$\mathcal L_{\mathrm{den}}^k
(\boldsymbol{\phi};\mathcal A_k),\;\nu(t)$
&
Smoothed class-wise semantic-reference denoising loss and its
timestep-sampling distribution, respectively.
\\

$\mathcal K_{\mathrm{traj}}^k
(\boldsymbol{\phi})$
&
Trajectory-level KL divergence from
$q_A^k(\mathbf z_{0:T})$ to
$p_{\boldsymbol{\phi}}^k(\mathbf z_{0:T})$.
\\

$\gamma_k,\;C_k$
&
Constants in the trajectory-KL upper bound
$\mathcal K_{\mathrm{traj}}^k
\leq
\gamma_k\mathcal L_{\mathrm{den}}^k+C_k$;
both are independent of $\boldsymbol{\phi}$.
\\

$\Phi_{\mathrm{sem}}^k
(\delta_{\mathrm{sem}})$
&
Set of class-$k$ LoRA parameters whose semantic-reference denoising
loss does not exceed $\delta_{\mathrm{sem}}$.
\\

$\mathcal B_{\mathrm{sem}}^k
(\delta_{\mathrm{sem}}),\;
\delta_{\mathrm{sem}}$
&
Generator-induced semantic ambiguity set for class $k$ and the semantic
tolerance controlling its feasible parameter set, respectively.
\\

\hline
\multicolumn{2}{l}{\textbf{Target-domain risk and adversarial search}} \\

$P_T,\;P_T^k$
&
Joint image-label distribution in the unseen target domain and its
class-$k$ conditional latent distribution, respectively.
\\

$\pi^k,\;\pi_T^k$
&
Generic class weight and target-domain prior probability of class $k$,
respectively.
\\

$\mathbf f_{\boldsymbol{\theta}},\;
\mathcal H$
&
Classifier parameterized by $\boldsymbol{\theta}$ and its hypothesis
class, respectively.
\\

$g_{\boldsymbol{\theta},k}(\mathbf z),\;
L_{\max}$
&
Classification loss
$\ell_{\mathrm{CE}}
(\mathbf f_{\boldsymbol{\theta}}(\mathcal F_D(\mathbf z)),k)$
and its uniform upper bound used in the analysis.
\\

$\mathcal R_k
(\mathbf f_{\boldsymbol{\theta}};P)$
&
Class-wise risk of classifier
$\mathbf f_{\boldsymbol{\theta}}$ under a class-$k$ latent distribution
$P$.
\\

$\mathcal R_T
(\mathbf f_{\boldsymbol{\theta}})$
&
Overall target-domain risk of classifier
$\mathbf f_{\boldsymbol{\theta}}$.
\\

$d_{\mathrm{TV}}(P,Q),\;\rho_k$
&
Total variation distance and the residual coverage gap between $P_T^k$
and $\mathcal B_{\mathrm{sem}}^k(\delta_{\mathrm{sem}})$, respectively.
\\

$\varepsilon_{\mathrm{adv},k}^{(r)},\;
\varepsilon_{\mathrm{adv}}^{(r)}$
&
Class-wise CADS search gap at round $r$ and its class-balanced average,
where
$\varepsilon_{\mathrm{adv}}^{(r)}
=C^{-1}\sum_{k=1}^{C}
\varepsilon_{\mathrm{adv},k}^{(r)}$.
\\

\hline
\multicolumn{2}{l}{\textbf{Generated data and finite-sample analysis}} \\

$\mathbf z_{k,j}^{(r)},\;
\tilde{\mathbf x}_{k,j}^{(r)}$
&
The $j$-th generated latent sample of class $k$ at round $r$ and its
decoded image, respectively.
\\

$\mathcal S_r,\;\mathcal D_g^{(r)}$
&
Class-balanced latent-label sample and the corresponding decoded
generated dataset at synthesis round $r$, respectively.
\\

$\mathcal R_g^{(r)}
(\mathbf f_{\boldsymbol{\theta}}),\;
\widehat{\mathcal R}_g^{(r)}
(\mathbf f_{\boldsymbol{\theta}})$
&
Population risk over the generated class-conditional distributions and
its empirical estimate on $\mathcal D_g^{(r)}$, respectively.
\\

$\mathcal G
\equiv
\ell_{\mathrm{CE}}\circ\mathcal H$
&
Loss class induced by the classifier hypothesis class and the
cross-entropy loss.
\\

$\mathfrak R_{CM_g}^{(r)}
(\ell_{\mathrm{CE}}\circ\mathcal H)$
&
Expected Rademacher complexity of the generated loss class under the
conditional class-balanced sampling scheme at round $r$.
\\

$\Delta_g^{(r)}(\delta)$
&
Finite-sample deviation term
$2\mathfrak R_{CM_g}^{(r)}
(\ell_{\mathrm{CE}}\circ\mathcal H)
+
L_{\max}
\sqrt{\log(2/\delta)/(2CM_g)}$,
where $\delta$ is the confidence parameter.
\\

\hline
\end{tabular}
\end{table*}

\noindent For ease of reference, Tab.~\ref{tab:theory_notation}
summarizes the principal notation used in the theoretical development
and the accompanying proofs.
For clarity, the symbols are grouped according to their roles in
semantic ambiguity set construction, target-domain risk analysis,
adversarial search, and finite-sample generalization.
Throughout, ($k$), ($t$), and ($r$) index the class, diffusion timestep,
and synthesis round, respectively.
Note that ($T$) denotes the total number of diffusion steps, whereas
the subscript ($T$) in target-domain quantities such as ($P_T^k$) and
($\pi_T^k$) identifies the target domain.

\begin{figure*}[!t]
\centering
\includegraphics[width=0.98\textwidth]{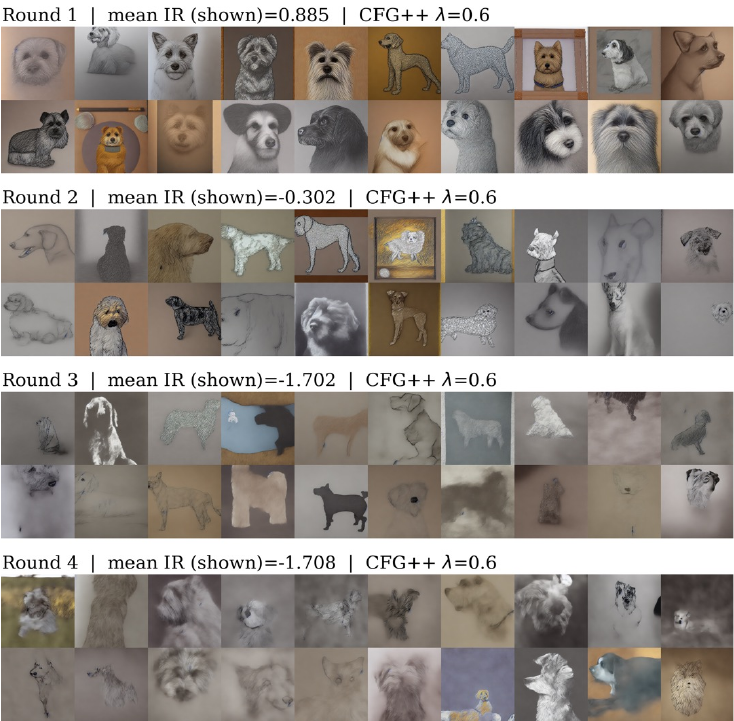}
\caption{Qualitative visualization of challenging samples conditioned
on the dog class during PAPT++ generation--training rounds 1--4 using
CFG++ with $\lambda=0.6$. Each row corresponds to one round, and the
mean ImageReward (IR) of the displayed samples is reported above the row.}
\label{fig:visualization_dog_round1_4}
\end{figure*}

\begin{figure*}[!t]
\centering
\includegraphics[width=0.98\textwidth]{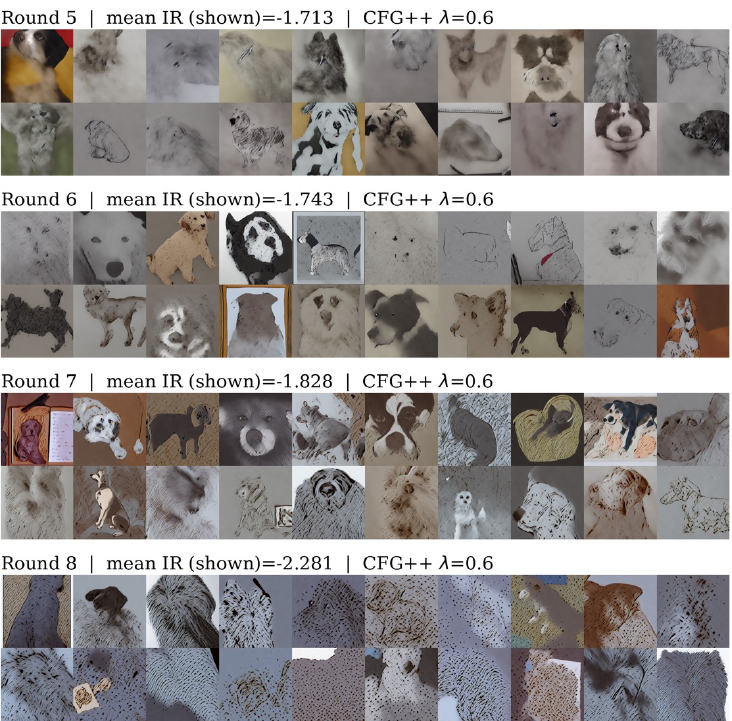}
\caption{Qualitative visualization of challenging samples conditioned
on the dog class during PAPT++ generation--training rounds 5--8 using
CFG++ with $\lambda=0.6$. Each row corresponds to one round, and the
mean ImageReward (IR) of the displayed samples is reported above the row.}
\label{fig:visualization_dog_round5_8}
\end{figure*}

\subsection{Qualitative Analysis Across PAPT++ Generation--Training Rounds}

\noindent To further examine the evolution of the generated samples, Fig.~\ref{fig:visualization_dog_round1_4} and Fig.~\ref{fig:visualization_dog_round5_8} visualize dog-class examples produced using CFG++ with $\lambda=0.6$ over eight CADS synthesis rounds. The early rounds mainly contain clean and canonical dog appearances. As the classifier is updated and provides new feedback, the later rounds explore less typical within-class variations in style, texture, shape, and composition, accompanied by a general decrease in the mean ImageReward of the displayed samples. Importantly, a lower ImageReward is not itself the objective, since samples that lose class semantics would provide limited training value. Instead, many of the visualized samples retain recognizable dog-related cues while moving beyond typical and relatively easy appearances. Together with the quantitative results in Tab.~\ref{tab:pacs_cfg_ablation} and Tab.~\ref{tab:vlcs_cfg_ablation} and Fig.~\ref{fig:cads_round_imagereward_pacs} and Fig.~\ref{fig:cads_round_imagereward}, these examples further demonstrate that CFG++ with $\lambda=0.6$ helps CADS balance semantic fidelity and sample difficulty, thereby generating challenging yet class-consistent samples for improving generalization to unseen domains.

\end{document}